\documentclass[11pt,fceqn]{article}

\usepackage{fullpage}
\usepackage{algorithm}
\usepackage{algpseudocode}
\usepackage{amsmath,amsthm,amsfonts}
\usepackage{amsmath}
\usepackage[subnum]{cases}

\usepackage{amssymb}
\usepackage{color}
\usepackage{mathrsfs}
\usepackage{enumitem}
\usepackage{bm}
\usepackage{multirow}
\usepackage{booktabs}
\usepackage{makecell}
\usepackage{graphicx}
\usepackage{subcaption}
\usepackage{comment}
\usepackage{cases}
\usepackage{appendix}
\usepackage{tikz}
\usetikzlibrary{arrows,shapes}

\usepackage[colorlinks, linkcolor=black, citecolor=blue, urlcolor=black]{hyperref}
\usepackage{cleveref}
\usepackage{marginnote}

\numberwithin{equation}{section}

\theoremstyle{definition}
\newtheorem{mainthm}{Theorem}
\newtheorem{thm}{Theorem}[section]
\newtheorem{lem}[thm]{Lemma}
\newtheorem{coro}[thm]{Corollary}

\newtheorem{prop}[thm]{Proposition}
\newtheorem{defn}[thm]{Definition}
\newtheorem{assumption}[thm]{Assumption}

\newcommand\dd{\mathrm{d}}
\newcommand\ee{\mathrm{e}}

\newcommand{\goto}{\rightarrow}

\newcommand{\e}{\mathrm{e}}

\newcommand{\R}{\mathbb{R}}

\usepackage{tikz}
\usepackage{pgfplots}
\usetikzlibrary{calc,intersections}

\usetikzlibrary{shapes.geometric, arrows, positioning}

\tikzset{
   mybox/.style  = {draw, rectangle, minimum width=4cm, minimum height=0.8cm, text centered, text width=4.4cm,   
  font=\normalsize},
  box/.style  = {draw, rectangle, minimum width=4.0cm, minimum height=0.8cm, text centered, text width=4.5cm,   
  font=\normalsize},
   myarrow/.style = {line width=0.2pt, draw=black, -triangle 60, postaction={draw, line width=0.2pt, shorten >=10pt,-}}
}

\tikzstyle{arrow} = [->, >=stealth, -triangle 60]

\allowdisplaybreaks

\DeclareMathOperator{\E}{\mathbb{E}}

\begin{document}

\title{On Learning Rates and Schr\"odinger Operators}

\author{Bin Shi\thanks{University of California, Berkeley. Email: \url{binshi@berkeley.edu}.} \and  Weijie J.~Su\thanks{The Wharton School, University of Pennsylvania. Email: \url{suw@wharton.upenn.edu}.} \and Michael I.~Jordan\thanks{University of California, Berkeley. Email: 
\url{jordan@cs.berkeley.edu}.}}
\date\today

\maketitle

\begin{abstract}

The learning rate is perhaps the single most important parameter in the training of neural networks and, more broadly, in stochastic (nonconvex) optimization. Accordingly, there are numerous effective, but poorly understood, techniques for tuning the learning rate, including learning rate decay, which starts with a large initial learning rate that is gradually decreased. In this paper, we present a general theoretical analysis of the effect of the learning rate in stochastic gradient descent (SGD).  Our analysis is based on the use of a \emph{learning-rate-dependent stochastic differential equation} (lr-dependent SDE) that serves as a surrogate for SGD. For a broad class of objective functions, we establish a linear rate of convergence for this continuous-time formulation of SGD, highlighting the fundamental importance of the learning rate in SGD, and contrasting to gradient descent and stochastic gradient Langevin dynamics. Moreover, we obtain an explicit expression for the optimal linear rate by analyzing the spectrum of the Witten-Laplacian, a special case of the Schr\"odinger operator associated with the lr-dependent SDE. Strikingly, this expression clearly reveals the dependence of the linear convergence rate on the learning rate---the linear rate  decreases rapidly to zero as the learning rate tends to zero for a broad class of nonconvex functions, whereas it stays constant for strongly convex functions. Based on this sharp distinction between nonconvex and convex problems, we provide a mathematical interpretation of the benefits of using learning rate decay for nonconvex optimization.

\end{abstract}


\section{Introduction}
\label{sec: intro}

Gradient-based optimization has been the workhorse algorithm powering recent developments in statistical machine learning. Many of these developments involve solving nonconvex optimization problems, which raises new challenges for theoreticians, given that classical theory has often been restricted to the convex setting.

A particular focus in machine learning is the class of gradient-based methods referred to as \emph{stochastic gradient descent} (SGD), given its desirable runtime properties, and its desirable statistical performance in a wide range of nonconvex problems.  Consider the minimization of a (nonconvex) function $f$ defined in terms of an expectation:
\[
f(x) = \E_{\zeta} f(x; \zeta),
\]
where the expectation is over the randomness embodied in $\zeta$. A simple example of this is empirical risk minimization, where the loss function,
\[
f(x) = \frac{1}{n} \sum_{i = 1}^{n} f_{i}(x),
\]
is averaged over $n$ data points, where the datapoint-specific losses, $f_i(x)$, are indexed by $i$ and where $x$ denotes a parameter. When $n$ is large, it is computationally prohibitive to compute the full gradient of the objective function, and SGD provides a compelling alternative.  SGD is a gradient-based update based on a (noisy) gradient evaluated from a single data point or a mini-batch:
\[
\widetilde{\nabla} f(x) := \frac1{B} \sum_{i \in \mathcal{B}} \nabla f_i(x) = \nabla f(x) + \xi,
\]
where the set $\mathcal{B}$ of size $B$ is sampled uniformly from the $n$ data points and therefore the noise term $\xi$ has mean zero. Starting from an initial point $x_0$, SGD updates the iterates according to
\begin{align}\label{eqn: SGD_1}
x_{k + 1} = x_{k} - s \widetilde{\nabla} f(x_k) = x_{k} - s\nabla f(x_{k}) - s \xi_k,
\end{align}
where $\xi_k$ denotes the noise term at the $k$th iteration. Note that the step size $s > 0$, also known as the \textit{learning rate}, can either be constant or vary with the iteration~\cite{bottou2010large}.


\begin{figure}[h!]
\centering
\includegraphics[width=3.6in, height=2.5in]{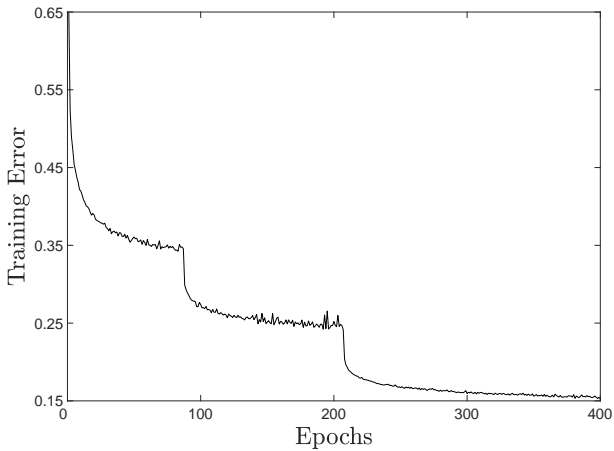}
\caption{\small Training error using SGD with mini-batch size 32 to train an 8-layer convolutional neural network on CIFAR-10~\cite{krizhevsky2009learning}. The first 90 epochs use a learning rate of $s = 0.006$, the next 120 epochs use $s = 0.003$, and the final 190 epochs use $s = 0.0005$. Note that the training error decreases as the learning rate $s$ decreases and a smaller $s$ leads to a larger number of epochs for SGD to reach a plateau. See~\cite{he2016deep} for further investigation of this phenomenon.} 
\label{fig: sgd-deep}
\end{figure}


The learning rate plays an essential role in determining the performance of SGD and many of the practical variants of SGD~\cite{bengio2012practical}.\footnote{Note that the mini-batch size as another parameter can be, to some extent, incorporated into the learning rate. See discussion later in this section.} The overall effect of the learning rate can be complex.  In convex optimization problems, theoretical analysis can explain many aspects of this complexity, but in the nonconvex setting the effect of the learning rate is yet more complex and theory is lacking~\cite{zeiler2012adadelta,kingma2014adam}. As a numerical illustration of this complexity, Figure~\ref{fig: sgd-deep} plots the error of SGD with a piecewise constant learning rate in the training of a neural network on the CIFAR-10 dataset. With a constant learning rate, SGD quickly reaches a plateau in terms of training error, and whenever the learning rate decreases, the plateau decreases as well, thereby yielding better optimization performance. This illustration exemplifies the idea of learning rate decay, a technique that is used in training deep neural networks (see, e.g., \cite{he2016deep,bottou2018optimization,sordello2019robust}). Despite its popularity and the empirical evidence of its success, however, the literature stops short of providing a \textit{general} and \textit{quantitative} approach to understanding how the learning rate impacts the performance of SGD and its variants in the nonconvex setting~\cite{you2019learning,li2019towards}.  Accordingly, strategies for setting learning rate decay schedules are generally adhoc and empirical.

In the current paper we provide theoretical insight into the dependence of SGD on the learning rate in nonconvex optimization. Our approach builds on a recent line of work in which optimization algorithms are studied via the analysis of their behavior in continuous-time limits~\cite{su2016differential,jordan2018dynamical,shi2018understanding}. Specifically, in the case of SGD, we study stochastic differential equations (SDEs) as surrogates for discrete stochastic optimization methods (see, e.g., \cite{kushner2003stochastic,li2017stochastic,krichene2017acceleration,chaudhari2018deep,diakonikolas2019generalized}). The construction is roughly as follows. Taking a small but nonzero learning rate $s$, let $t_{k} = ks$ denote a time step and define $x_{k} = X_s(t_{k})$ for some sufficiently smooth curve $X_s(t)$. Applying a Taylor expansion in powers of $s$, we obtain:
\[
x_{k + 1} = X_s(t_{k+1}) = X_s(t_{k}) + \dot{X_s}(t_{k}) s + O(s^{2}).  
\]
Let $W$ be a standard Brownian motion and, for the time being, assume that the noise term $\xi_k$ is approximately normally distributed with unit variance. Informally, this leads to\footnote{Although a Brownian motion is not differentiable, the formal notation $\dd W(t)/ \dd t$ can be given a rigorous interpretation~\cite{evans2012introduction,villani2006hypocoercive}.}
\[
-\sqrt{s} \xi_k = W(t_{k + 1}) - W(t_{k}) = s \,\frac{\dd W(t_k)}{\dd t} + O(s^{2}).
\]
Plugging the last two displays into \eqref{eqn: SGD_1}, we get
\begin{align}\nonumber
\dot{X_s}(t_{k}) +  O(s) = - \nabla f(X_s(t_{k})) + \sqrt{s} \frac{\dd W(t_{k})}{\dd t} + O\left( s^{\frac{3}{2}}\right).
\end{align}
Retaining both $O(1)$ and $O(\sqrt{s})$ terms but ignoring smaller terms, we obtain a \emph{learning-rate-dependent stochastic differential equation} (lr-dependent SDE) that approximates the discrete-time SGD algorithm:
\begin{align}\label{eqn: sgd_high_resolution_formally}
\dd X_s= - \nabla f(X_s)\dd t +  \sqrt{s}\dd W,
\end{align}
where the initial condition is the same value $x_0$ as its discrete counterpart. This SDE has been shown to be a valid approximating surrogate for SGD in earlier work~\cite{kushner2003stochastic,chaudhari2018stochastic}. As an indication of the generality of this formulation, we note that it can seamlessly take account of the mini-batch size $B$; in particular, the effective learning rate scales as $O(s/B)$ in the mini-batch setting (see more discussion in \cite{smith2017don}). Throughout this paper we focus on \eqref{eqn: sgd_high_resolution_formally} and regard $s$ alone as the effective learning rate.\footnote{Recognizing that the variance of $\xi_k$ is inversely proportional to the mini-batch size $B$, we assume that the noise term $\xi_k$ has variance $\sigma^2/B$. Under this assumption the resulting SDE reads $\dd X_s= - \nabla f(X_s)\dd t +  \sigma \sqrt{s/B}\dd W$. In light of this, the effective learning rate through incorporating the mini-batch size is $O(\sigma^2 s/B)$.}


\begin{figure}[h!]
\centering
\begin{minipage}[t]{0.45\linewidth}
\centering
\includegraphics[scale=0.16]{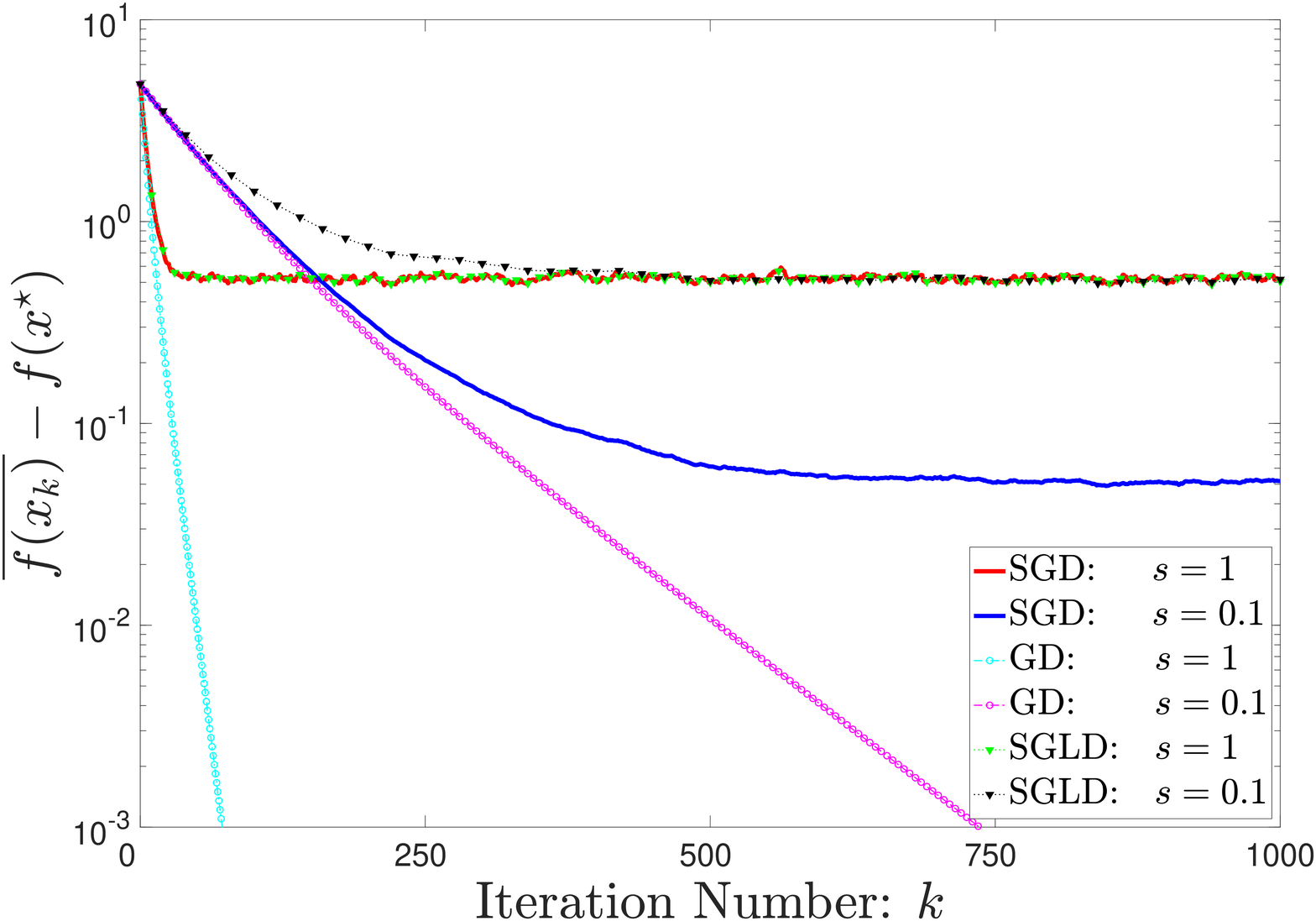}
\end{minipage}
\begin{minipage}[t]{0.45\linewidth}
\centering
\includegraphics[scale=0.16]{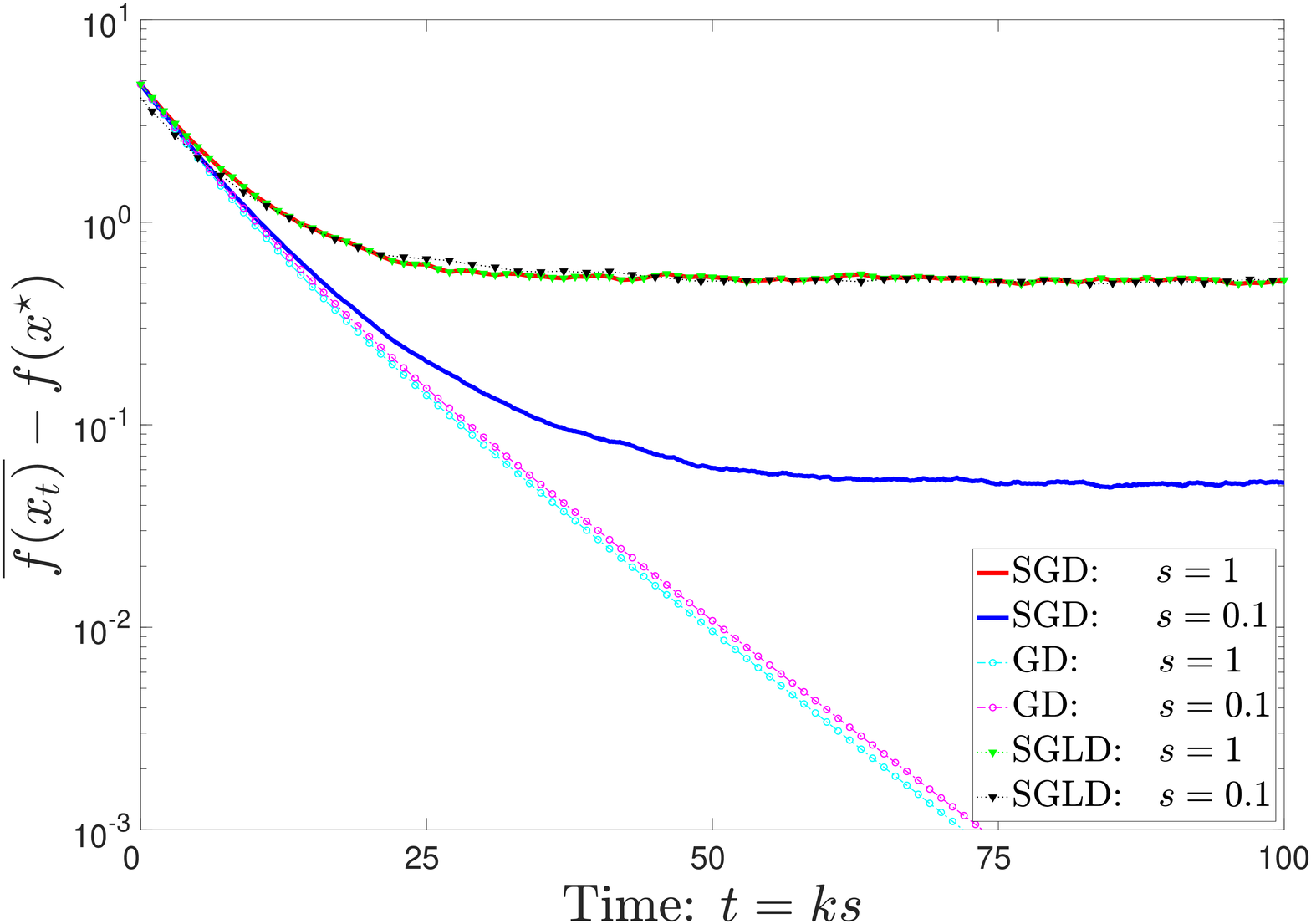}
\end{minipage}
\caption{\small Illustrative examples showing distinct behaviors of GD, SGD, and SGLD. The $y$-axis displays the optimization error $\overline{f(x_k)} - f(x^\star)$, where $f(x^\star)$ denotes the minimum value of the objective and in the case of SGD and SGLD $\overline{f(x_k)}$ denotes an average over 1000 replications. The objective function is $f(x_1,x_2) = 5 \times 10^{-2}x_1^{2} + 2.5 \times 10^{-2}x_2^{2}$, with an initial point $(8,8)$, and the noise $\xi_k$ in the gradient follows a standard normal distribution. Note that SGD with $s=1$ is identical to SGLD with $s = 1$. As shown in the right panel, taking time $t = ks$ as the $x$-axis, the learning rate has little to no impact on GD and SGLD in terms of optimization error.}
\label{fig:traj_compare}
\end{figure}


Intuitively, a larger learning rate $s$ gives rise to more stochasticity in the lr-dependent SDE~\eqref{eqn: sgd_high_resolution_formally}, and vice versa. Accordingly, the learning rate must have a substantial impact on the dynamics of SGD in its continuous-time formulation. In stark contrast, this parameter plays a fundamentally different role in gradient descent (GD) and stochastic gradient Langevin dynamics (SGLD) when one considers their limiting differential equations. In particular, consider GD:
\begin{align}\nonumber
x_{k + 1} = x_{k} - s \nabla f(x_{k}),
\end{align}
which can be modeled by the following ordinary differential equation (ODE):
\begin{equation}\nonumber
\dot X = - \nabla f(X),
\end{equation}
and the SGLD algorithm, which adds Gaussian noise $\xi_k$ to the GD iterates:
\begin{align}\nonumber
x_{k + 1} = x_{k} - s \nabla f(x_{k}) + \sqrt{s} \xi_k,
\end{align}
and its SDE model:
\begin{equation}\nonumber
\dd X = -\nabla f(X) \dd t + \dd W.
\end{equation}
These differential equations are derived in the same way as \eqref{eqn: sgd_high_resolution_formally}, namely by the Taylor expansion and retaining $O(1)$ and $O(\sqrt{s})$ terms.\footnote{The coefficients of the $O(\sqrt{s})$ terms turn out to be zero in both differential equations. See more discussion in Appendix~\ref{subsec: supplement-DE-algorithms} and particularly \Cref{fig:chart} therein.} While the SDE for modeling SGD sets the square root of the learning rate to be its diffusion coefficient, both the GD and SGLD counterparts are completely free of this parameter. This distinction between SGD and the other two methods is reflected in their different numerical performance as revealed in \Cref{fig:traj_compare}. The right plot of this figure shows that the behaviors of both GD and SGLD in the time $t = ks$  scale are almost invariant in terms of optimization error with respect to the learning rate. In striking contrast, the stationary optimization error of SGD decreases significantly as the learning rate decays. As a consequence of this distinction, GD and SGLD do not exhibit the phenomenon that is shown in~\Cref{fig: sgd-deep}.

\subsection{Overview of contributions}
\label{subsec: contribution}


The discussion thus far suggests that one may examine the effect of the learning rate in SGD using the lr-dependent SDE~\eqref{eqn: sgd_high_resolution_formally}.  In particular, this SDE distinguishes SGD from GD and SGLD.  Accordingly, in the current paper we study the lr-dependent SDE, and make the following contributions.


\begin{enumerate}
\item 
\noindent{\bf Linear convergence to stationarity.} We show that, for a large class of (nonconvex) objectives, the continuous-time formulation of SGD converges to its stationary distribution at a \textit{linear rate}.\footnote{Roughly speaking, stationarity refers to the distribution of $X_s(t)$ in the limit $t \rightarrow \infty$. See a more precise definition in \Cref{sec:main-results}.} In particular, we prove that the solution $X_s(t)$ to the lr-dependent SDE obeys
\begin{equation}\label{eq:intro_conv}
\E  f(X_s(t)) - f^\star  \le \epsilon(s) + C(s) \e^{-\lambda_s  t},
\end{equation}
where $f^\star$ denotes the global minimum of the objective function $f$, $\epsilon(s)$ denotes the risk at stationarity, and $C(s)$ depends on both the learning rate and the distribution of the initial $x_0$. Notably, we can show that $\epsilon(s)$ decreases monotonically to zero as $s \goto 0$. This bound can be carried over to the discrete case by a uniform approximation between SGD and the lr-dependent SDE~\eqref{eqn: sgd_high_resolution_formally}. Specifically, the term $C(s) \e^{-\lambda_s  t}$ becomes $C(s) \e^{-\lambda_s k s}$, showing that the convergence is linear as well in the discrete regime. This is consistent with the numerical evidence from \Cref{fig: sgd-deep} and \Cref{fig:traj_compare}. 

This convergence result sheds light on why SGD performs so well in many practical nonconvex problems. In particular, note that while GD can be trapped in a saddle point or a local minimum, SGD can efficiently escape saddle points, provided that the linear rate $\lambda_s$ is not too small (this is the case if $s$ is sufficiently large; see the second contribution). This superiority of SGD in the nonconvex setting must be attributed to the noise in the gradient and this implication is consistent with earlier work showing that stochasticity in gradients significantly accelerates the escape of saddle points for gradient-based methods~\cite{jin2017escape,lee2016gradient}.


\item \noindent{\bf Distinctions between convexity and nonconvexity.} The first contribution stops short of saying anything about how $\lambda_s$ depends on the learning rate $s$ and the \textit{geometry} of the objective $f$. Such an analysis is fundamental to an explanation of the differing effects of the learning rate in deep learning (nonconvex optimization) and convex optimization. In the current paper we show that if the objective $f$ is a nonconvex function and satisfies certain regularity conditions, we have:\footnote{We write $a_m \asymp b_m$ if there exist positive constants $c$ and $c'$ such that $c b_m \le a_m \le c' b_m$ for all $m$.}
\begin{equation}\label{eq:lambda_s_exp_intro}
\lambda_s \asymp \e^{-\frac{2 H_f}{s}},
\end{equation}
for a certain value $H_f > 0$ that only depends on $f$. This expression for $\lambda_s$ enables a concrete interpretation of the effect of learning rate in \Cref{fig: sgd-deep}. In brief, in the nonconvex setting, $\lambda_s$ decreases to zero quickly as the learning rate $s$ tends to zero. As a consequence, with a large learning rate $s$ at the beginning, SGD converges rapidly to stationarity and the rate becomes smaller as the learning rate decreases. 

For comparison, $\lambda_s$ is equal to $\mu$ if $f$ is $\mu$-strongly convex for $\mu > 0$, regardless of the learning rate $s$. As such, the convergence behaviors of SGD are necessarily different between convex and nonconvex objectives. To appreciate this implication, we refer to \Cref{fig: sgd_nonconvex_traj-gen}. Note that all four plots show that a larger learning rate gives rise to a larger stationary risk, as predicted by the monotonically increasing nature of $\epsilon$ with respect to $s$ in \eqref{eq:intro_conv}. The most salient part of this figure is, however, shown in the right panel. Specifically, the right panel, which uses time $t$ as the $x$-axis, shows that in the (strongly) convex setting the linear rate of the convergence is roughly the same between the two choices of learning rate, which is consistent with the result that $\lambda_s$ is constant in the case of a strongly convex objective. In the nonconvex case (bottom
right), however, the rate of convergence is more rapid with the larger learning rate $s = 0.1$, which is implied by the fact that $\lambda_{0.1} > \lambda_{0.05}$. In stark contrast, the two plots in the left panel, which use the number $k$ of iterations for the $x$-axis, are observed to have a larger rate of linear convergence with a larger learning rate. This is because in the $k$ scale the rate $\lambda_s s$ of linear convergence always increases as $s$ increases no matter if the objective is convex or nonconvex.

\end{enumerate}

The mathematical tools that we bring to bear in analyzing the lr-dependent SDE~\eqref{eqn: sgd_high_resolution_formally}
are as follows.  We establish the linear convergence via a Poincar\'e-type inequality that is due to Villani~\cite{villani2009hypocoercivity}. 
The asymptotic expression for the rate $\lambda_s$ is proved by making use of the spectral theory of the Schr\"odinger operator or, more concretely, the Witten-Laplacian associated with the Fokker--Planck--Smoluchowski equation that governs the lr-dependent SDE. We believe that these tools will prove to be useful in theoretical analyses of other stochastic approximation methods.

\begin{figure}[htb!]
\centering
\begin{minipage}[t]{0.45\linewidth}
\centering
\includegraphics[scale=0.15]{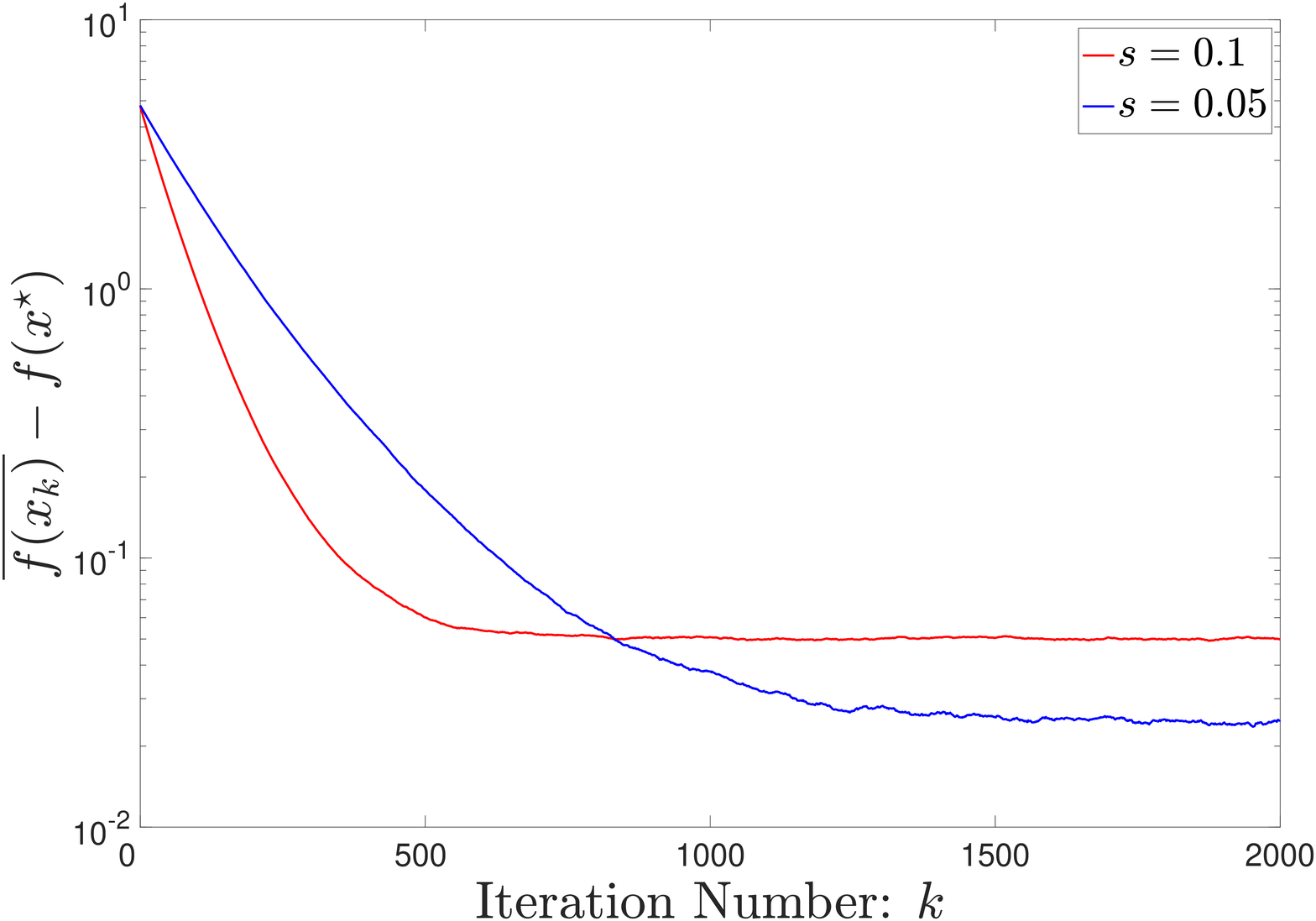}
\end{minipage}
\begin{minipage}[t]{0.45\linewidth}
\centering
\includegraphics[scale=0.15]{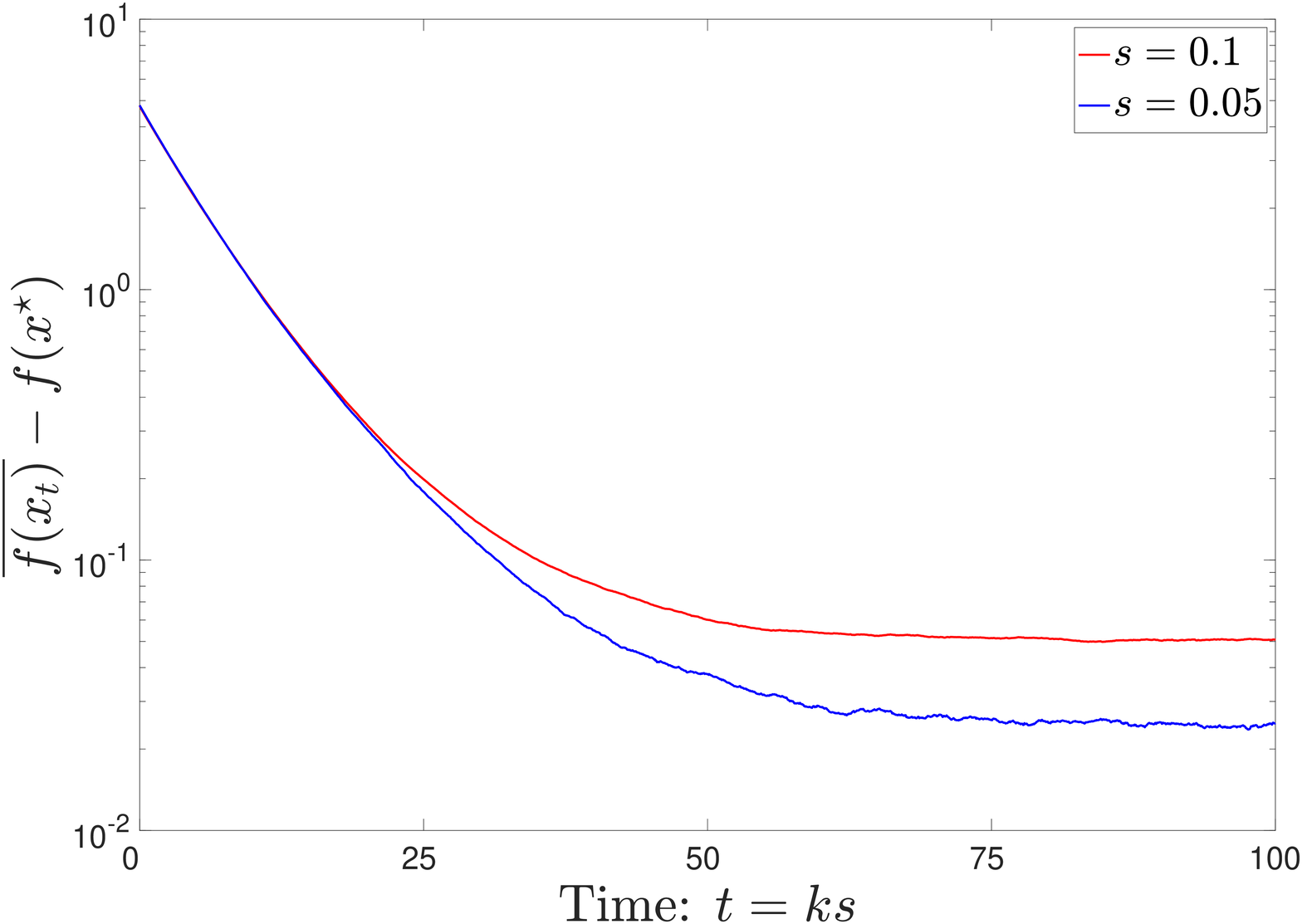}
\end{minipage}\\
\begin{minipage}[t]{0.45\linewidth}
\centering
\includegraphics[scale=0.15]{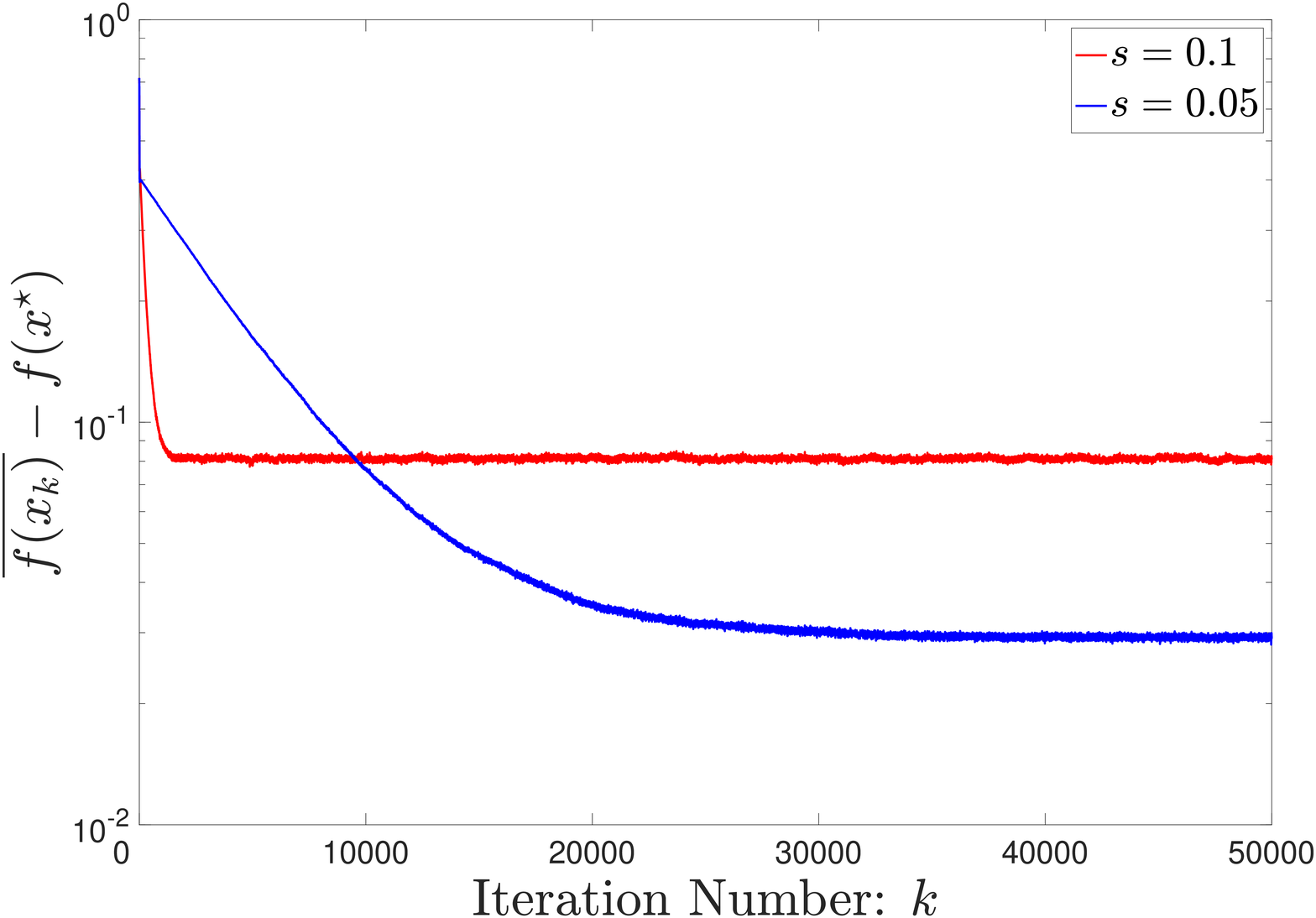}
\end{minipage}
\begin{minipage}[t]{0.45\linewidth}
\centering
\includegraphics[scale=0.15]{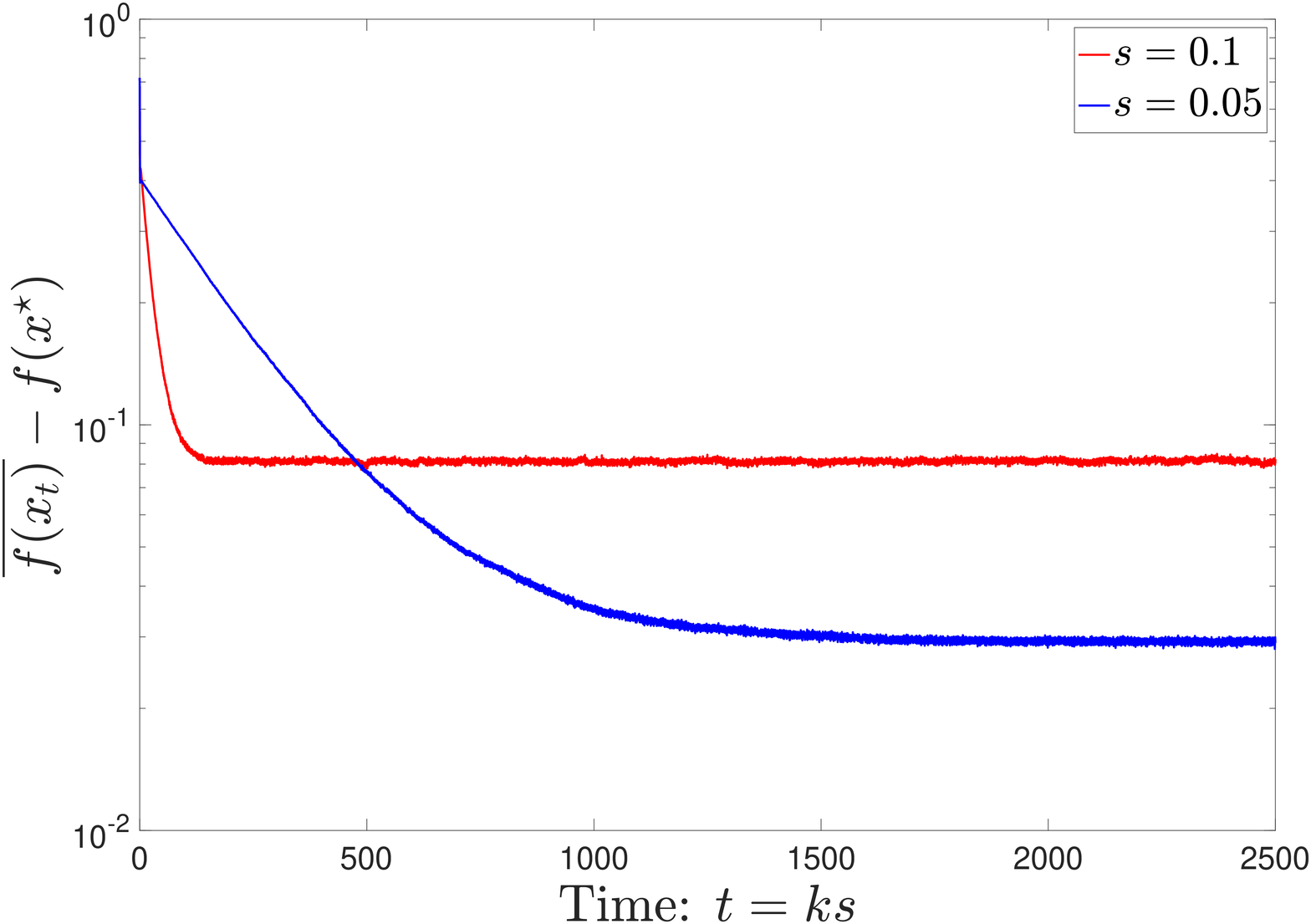}
\end{minipage}
\caption{\small The dependence of the optimization dynamics of SGD on the learning rate \textit{differs} between convex objectives and nonconvex objectives. The learning rate is set to either $s = 0.1$ or $s = 0.05$. The two top plots consider minimizing a convex function $f(x_1, x_2) = 5 \times 10^{-2}x_1^{2} + 2.5 \times 10^{-2}x_2^{2}$, with an initial point $(8,8)$, and the bottom plots consider minimizing a nonconvex function $f(x_1,x_2) = [(x_1+0.7)^{2} + 0.1](x_1 - 0.7)^{2} + (x_2 + 0.7)^{2} [(x_2 - 0.7)^{2} + 0.1 ]$, with an initial point $(-0.9, 0.9)$. The gradient noise is drawn from the standard normal distribution. All results are averaged over 10000 independent replications.} 
\label{fig: sgd_nonconvex_traj-gen}
\end{figure}

\subsection{Related work}
\label{subsec: related_works}


Recent years have witnessed a surge of research devoted to explanations of the effectiveness of deep neural networks, with a particular focus on understanding how the learning rate affects the behavior of stochastic optimization. In \cite{smith2017don,keskar2016large}, the authors uncovered various tradeoffs linking the learning rate and the mini-batch size. Moreover, \cite{jastrzkebski2017three,jastrzebski2018relation} related the learning rate to the generalization performance of neural networks in the early phase of training. This connection has been further strengthened by the demonstration that learning rate decay encourages SGD to learn features of increasing complexity~\cite{li2019towards,you2019learning}. From a topological perspective, \cite{davis2019stochastic} establish connections between the learning rate and the sharpness of local minima. Empirically, deep learning models work well with non-decaying schedules such as cyclical learning rates~\cite{loshchilov2016sgdr,smith2017cyclical} (see also the review~\cite{sun2019optimization}), with recent theoretical justification~\cite{li2019exponential}.


In a different direction, there has been a flurry of activity in using dynamical systems to analyze discrete optimization methods. For example, \cite{su2016differential,wibisono2016variational,shi2018understanding} derived ODEs for modeling Nesterov's accelerated gradient methods and used the ODEs to understand the acceleration phenomenon (see the review \cite{jordan2018dynamical}). In the stochastic setting, this approach has been recently pursued by various authors \cite{chaudhari2018deep,chaudhari2018stochastic,mandt2016variational,lee2016gradient,caluya2019gradient,li2017stochastic} to establish various properties of stochastic optimization. As a notable advantage, the continuous-time perspective allows us to work without assumptions on the boundedness of the domain and gradients, as opposed to older analyses of SGD (see, for example, \cite{hazan2008adaptive}).


Our work is motivated in part by the recent progress on Langevin dynamics, in particular in nonconvex settings~\cite{villani2009hypocoercivity, pavliotis2014stochastic, helffer2004quantitative,bovier2005metastability}. In relating to Langevin dynamics, $s$ in the lr-dependent SDE can be thought of as the temperature parameter and, under certain conditions, this SDE has a stationary distribution given by the Gibbs measure, which is proportional to $\exp(-2f/s)$. Of particular relevance to the present paper from this perspective is a line of work that has considered the optimization properties of SGLD and analyzed its convergence rates~\cite{hwang1980laplace,raginsky2017non,zhang2017hitting}. Compared to these results, however, the present paper is distinct in that our analysis provides a more concise and sharp delineation of the convergence rate based on geometric properties of the objective function.


\subsection{Organization}
\label{subsec: framework}

The remainder of the paper is structured as follows. In \Cref{sec: preliminaries} we introduce basic assumptions and techniques employed throughout this paper. Next, \Cref{sec:main-results} develops our main theorems. In \Cref{sec:expl-benef-init}, we use the results of \Cref{sec:main-results} to offer insights into the benefit of taking a larger initial learning rate followed by a sequence of decreasing learning rates in training neural networks. \Cref{sec: convergence-rate} formally proves the linear convergence \eqref{eq:intro_conv} and \Cref{sec: estimate-rate} further specifies the rate of convergence \eqref{eq:lambda_s_exp_intro}. Technical details of the proofs are deferred to the appendices. We conclude the paper in \Cref{sec: conclusion} with a few directions for future research.

\section{Preliminaries}
\label{sec: preliminaries}

Throughout this paper, we assume that the objective function $f$ is infinitely differentiable in $\R^d$; that is, $f \in C^{\infty}(\mathbb{R}^{d})$. We use $\|\cdot\|$ to denote the standard Euclidean norm.

\begin{defn}[Confining condition \cite{pavliotis2014stochastic,markowich1999trend}]\label{defn: confining}
A function $f$ is said to be \textit{confining} if it is infinitely differentiable and satisfies $\lim_{\|x\| \rightarrow +\infty} f(x) = +\infty$ and $\exp(-2f/s)$ is integrable for all $s > 0$:
\[
\int_{\R^d}  \e^{- \frac{2f(x)}{s}} \mathrm{d}x < +\infty.
\]
\end{defn}
This condition is quite mild and, indeed, it essentially requires that the function grows sufficiently rapidly when $x$ is far from the origin. This condition is met, for example, when an $\ell_2$ regularization term is added to the objective function $f$ or, equivalently, weight decay is employed in the SGD update.

Next, we need to show that the~lr-dependent SDE~\eqref{eqn: sgd_high_resolution_formally} with an arbitrary learning rate $s > 0$ admits a unique global solution under mild conditions on the objective $f$. We will show in \Cref{sec:discrete-perspective} that the solution to this SDE approximates the SGD iterates well. The  formal description is shown rigorously in Proposition~\ref{prop: approx}. Recall that the lr-dependent SDE~\eqref{eqn: sgd_high_resolution_formally} is
\[
\dd X_s= - \nabla f(X_s)\dd t +  \sqrt{s}\dd W,
\]
where the initial point $X_s(0)$ is distributed according to a probability density function $\rho$ in $\R^d$, independent of the standard Brownian motion $W$. It is well known that the probability density $\rho_s(t, \cdot)$ of $X_s(t)$ evolves according to the Fokker--Planck--Smoluchowski equation
\begin{equation}\label{eqn: Fokker-Planck}
\frac{\partial \rho_s}{\partial t}  =   \nabla \cdot \left(\rho_s \nabla f\right) + \frac{s}{2} \Delta \rho_s,
\end{equation}
with the boundary condition $\rho_s(0, \cdot) = \rho$. Here, $\Delta \equiv \nabla \cdot \nabla$ is the Laplacian. For completeness, in Appendix~\ref{subsec: fok-plk} we derive this Fokker--Planck--Smoluchowski equation from the lr-dependent SDE~\eqref{eqn: sgd_high_resolution_formally} by It\^o's formula. If the objective $f$ satisfies the confining condition, then this equation admits a unique invariant Gibbs distribution that takes the form
\begin{equation}\label{eqn: Gibbs}
\mu_{s} = \frac{1}{Z_s} \e^{- \frac{2 f}{s}}.
\end{equation}
The proof of uniqueness is shown in Appendix~\ref{subsec: proof-unique-steady}. The normalization factor is $Z_{s} = \int_{\mathbb{R}^{d}} \e^{- \frac{2f}{s}} \dd x$. Taking any initial probability density $\rho_s(0, \cdot) \equiv \rho$ in $L^2(\mu_s^{-1})$ (a measurable function $g$ is said to belong to $L^2(\mu_s^{-1})$ if $\|g\|_{\mu_s^{-1}}  := \left( \int_{\mathbb{R}^{d}} g^{2} \mu_s^{-1} \dd x \right)^{\frac{1}{2}} < + \infty$), we have the following guarantee:

\begin{lem}[Existence and uniqueness of the weak solution]\label{prop: unique-existence}
For any confining function $f$ and any initial probability density $\rho \in L^2(\mu_s^{-1})$, the lr-dependent SDE~\eqref{eqn: sgd_high_resolution_formally} admits a weak solution whose probability density in $C^{1}\left([0,+\infty), L^{2}(\mu_s^{-1}) \right)$ is the unique solution to the Fokker--Planck--Smoluchowski equation~\eqref{eqn: Fokker-Planck}. 
\end{lem}


The proof of \Cref{prop: unique-existence} is shown in Appendix~\ref{subsec: proof_wellposedness}. For more information, \Cref{thm: converge} in \Cref{sec: convergence-rate} shows that the probability density $\rho_s(t, \cdot)$ converges to the Gibbs distribution as $t \goto \infty$. 

Finally, we need a condition that is due to Villani for the development of our main results in the next section.
\begin{defn}[Villani condition~\cite{villani2009hypocoercivity}]\label{defn: villani-condition}
A confining function $f$ is said to satisfy the Villani condition if
$
\|\nabla f(x)\|^{2}/s - \Delta f(x) \rightarrow + \infty
$
as $\|x\| \rightarrow +\infty$ for all $s > 0$.
\end{defn}

This condition amounts to saying that the gradient has a sufficiently large squared norm compared with the Laplacian of the function. Strictly speaking, some loss functions used for training neural networks might not satisfy this condition. However, the Villani condition does not look as stringent as it appears since the SGD iterates in the training process are bounded and this condition is essentially concerned with the function at infinity.

\section{Main Results}
\label{sec:main-results}

In this section, we state our main results. In brief, in \Cref{sec:linear-convergence} we show linear convergence to stationarity for SGD in its continuous formulation, the lr-dependent SDE. In \Cref{sec:rate-line-conv}, we derive a quantitative expression of the rate of linear convergence and study the difference in the behavior of SGD in the convex and nonconvex settings. This distinction is further elaborated in \Cref{sec:discrete-perspective} by carrying over the continuous-time convergence guarantees to the discrete case. Finally, \Cref{sec:one-dimens-example} offers an exposition of the theoretical results in the univariate case. Proofs of the results presented in this section are deferred to Section~\ref{sec: convergence-rate} and Section~\ref{sec: estimate-rate}.


\subsection{Linear convergence}
\label{sec:linear-convergence}

In this subsection we are concerned with the expected excess risk, $\E  f(X_s(t)) - f^\star$. Recall that $f^\star = \inf_x f(x)$. 

\begin{mainthm}\label{thm: continuous-qualitative}
Let $f$ satisfy both the confining condition and the Villani condition. Then there exists $\lambda_s > 0$ for any learning rate $s > 0$ such that the expected excess risk satisfies
\begin{equation}\label{eqn: continuous-qualitative}
\E  f(X_s(t)) - f^\star  \le \epsilon(s) + D(s) \e^{-\lambda_s  t},
\end{equation}
for all $t \ge 0$. Here $\epsilon(s) = \epsilon(s; f) \ge 0$ is a strictly increasing function of $s$ depending only on the objective function $f$, and $D(s) = D(s; f, \rho) \ge 0$ depends only on $s, f$, and the initial distribution $\rho$.

\end{mainthm}

Briefly, the proof of this theorem is based on the following decomposition of the excess risk:
\[
\E  f(X_s(t)) - f^\star = \E  f(X_s(t)) - \E f(X_s(\infty))  + \E f(X_s(\infty))  - f^\star,
\]
where we informally use $\E f(X_s(\infty))$ to denote $\E_{X \sim \mu_s} f(X)$ in light of the fact that $X_s(t)$ converges weakly to $\mu_s$ as $t \goto +\infty$ (see \Cref{thm: converge}). The question is thus to quantify how fast $\E  f(X_s(t)) - \E f(X_s(\infty))$ vanishes to zero as $t \goto \infty$ and how the excess risk at stationarity $\E f(X_s(\infty))  - f^\star$ depends on the learning rate. The following two propositions address these two questions. Recall that $\rho \in L^2(\mu_s^{-1})$ is the probability density of the initial iterate in SGD. 

\begin{prop}\label{prop:convv}
Under the assumptions of \Cref{thm: continuous-qualitative}, there exists $\lambda_s > 0$ for any learning rate $s$ such that
\[
\left| \E  f(X_s(t)) - \E f(X_s(\infty))  \right| \le  C(s) \left\| \rho - \mu_{s} \right\|_{\mu_s^{-1}} \e^{-\lambda_s  t},
\]
for all $t \ge 0$, where the constant $C(s) > 0$ depends only on $s$ and $f$, and where 
\[
\left\| \rho - \mu_{s} \right\|_{\mu_s^{-1}} = \left(\int_{\mathbb{R}^{d}}  \left( \rho - \mu_s \right)^{2} \mu_s^{-1} \dd x \right)^{\frac{1}{2}}
\]
measures the gap between the initialization and the stationary distribution. 
\end{prop}

Loosely speaking, it takes $O(1/\lambda_s)$ time to converge to stationarity. In relating to \Cref{thm: continuous-qualitative}, $D(s)$ can be set to $C(s) \left\| \rho - \mu_{s} \right\|_{\mu_s^{-1}}$. Notably, the proof of \Cref{prop:convv} shall reveal that $C(s)$ increases as $s$ increases.

Turning to the analysis of the second term, $\E f(X_s(\infty))  - f^\star$, we write henceforth $\epsilon(s) := \E f(X_s(\infty))  - f^\star$.
\begin{prop}\label{thm: compare_step}
Under the assumptions of \Cref{thm: continuous-qualitative}, the excess risk at stationarity, $\epsilon(s)$, is a strictly increasing function of $s$. Moreover, for any $S > 0$, there exists a constant $A$ that depends only on $S$ and $f$ and satisfies
\[
\epsilon(s) \equiv \E f(X_s(\infty))  - f^\star \le A s,
\]
for any learning rate $0 < s \le S$.
\end{prop}


The two propositions are proved in \Cref{sec: convergence-rate}. The proof of \Cref{thm: continuous-qualitative} is a direct consequence of \Cref{prop:convv} and \Cref{thm: compare_step}. More precisely, the two propositions taken together give
\begin{equation}\label{eq:as_strong}
\E  f(X_s(t)) - f^\star  \le O(s) + C(s) \e^{-\lambda_s  t},
\end{equation}
for a bounded learning rate $s$.



Taken together, these results offer insights into the phenomena observed in \Cref{fig: sgd-deep}. In particular, \Cref{prop:convv} states that, from the continuous-time perspective, the risk of SGD with a constant learning rate applied to a (nonconvex) objective function converges to stationarity at a \textit{linear} rate. Moreover, \Cref{thm: compare_step} demonstrates that the excess risk at stationarity decreases as the learning rate $s$ tends to zero. This is in agreement with the numerical experiments illustrated in Figures~\ref{fig: sgd-deep}, \ref{fig:traj_compare}, and \ref{fig: sgd_nonconvex_traj-gen}. For comparison, this property is not observed in GD and SGLD. 



The following result gives the iteration complexity of SGD in its continuous-time formulation.
\begin{coro}\label{coro:iter_epsi}
Under the assumptions of \Cref{thm: compare_step}, for any $\epsilon > 0$, if the learning rate $s \le \min\{ \epsilon/(2 A), S\}$ and $t \ge \frac1{\lambda_s} \log \frac{2C(s) \left\| \rho - \mu_{s} \right\|_{\mu_s^{-1}} }{\epsilon}$, then
\[
\E  f(X_s(t)) - f^\star \le \epsilon.
\]
\end{coro}



\subsection{The rate of linear convergence}
\label{sec:rate-line-conv}


We now turn to the key issue of understanding how the linear rate $\lambda_s$ depends on the learning rate. In this subsection, we show that for certain objective functions, $\lambda_s$ admits a simple expression that allows us to interpret how the convergence rate depends on the learning rate. 

We begin by considering a strongly convex function. Recall the definition of strong convexity: for $\mu > 0$, a function $f$ is $\mu$-strongly convex if 
\[
f(y) \ge f(x) + \langle \nabla f(x), y - x \rangle + \frac{\mu}{2} \|y - x\|^2,
\] 
for all $x, y$. Equivalently, $f$ is $\mu$-strong convex if all eigenvalues of its Hessian $\nabla^2 f(x)$ are greater than or equal to $\mu$ for all $x$ (note that here $f$ is assumed to be infinitely differentiable). As is clear, a strongly convex function satisfies the confining condition. In \Cref{subsec: proof-mu-strong}, we prove the following proposition by making use of a Poincar\'e-type inequality, the Bakry--Emery theorem~\cite{bakry2013analysis}.\footnote{In fact, we can obtain a tighter log-Sobolev inequality for convergence of the probability densities in $L^1(\mathbb{R}^d)$, as is shown in \Cref{sbsec: l1-space-density}.}


\begin{prop}\label{prop:lambda_str_mu}
In addition to the assumptions of \Cref{thm: continuous-qualitative}, assume that the objective $f$ is a $\mu$-strongly convex function. Then, $\lambda_{s}$ in~\eqref{eqn: continuous-qualitative} satisfies $\lambda_s = \mu$.
\end{prop}

We turn to the more challenging setting where $f$ is \textit{nonconvex}. Let us refer to the objective $f$ as a \textit{Morse function} if its Hessian has full rank at any critical point $x$ (that is, $\nabla f(x) = 0$).\footnote{See \Cref{sec:basics-morse-theory} for a discussion of Morse functions. Note that (infinitely differentiable) strongly convex functions are Morse functions.}


\begin{mainthm}\label{thm: continuous-quantative}
In addition to the assumptions of Theorem~\ref{thm: continuous-qualitative}, assume that the objective $f$ is a Morse function and has at least two local minima.\footnote{We call $x$ a local minimum of $f$ if $\nabla f(x) = 0$ and the Hessian $\nabla^2 f(x)$ is positive definite. By convention, in this paper a global minimum is also considered a local minimum.} Then the constant $\lambda_{s}$ in~\eqref{eqn: continuous-qualitative} satisfies
\begin{equation}\label{eqn: main-lambda-nonconvex}
\lambda_s = (\alpha + o(s)) \e^{- \frac{2H_{f}}{s}},
\end{equation}
for $0 < s \le s_0$, where $s_0 >0, \alpha > 0$, and $H_{f} > 0$ are constants that all depend \textit{only} on $f$. 

\end{mainthm}

The proof of this result relies on tools in the spectral theory of Schr\"{o}dinger operators and is deferred to \Cref{sec: estimate-rate}. From now on, we call $\lambda_s$ in~\eqref{eqn: continuous-qualitative} the \textit{exponential decay constant}. To obviate any confusion, $o(s)$ in~\Cref{thm: continuous-quantative} stands for a quantity that tends to zero as $s \goto 0$, and the precise expression for $H_{f}$ shall be given in \Cref{sec: estimate-rate}, with a simple example provided in \Cref{sec:one-dimens-example}. To leverage \Cref{thm: continuous-quantative} for understanding the phenomena discussed in \Cref{sec: intro}, however, it suffices to recognize the fact that $H_{f}$ is completely determined by $f$. Moreover, we remark that while \Cref{thm: continuous-qualitative} shows that $\lambda_s$ exists for any learning rate, the present theorem assumes a bounded learning rate. 



The key implication of this result is that the rate of convergence is highly contingent upon the learning rate $s$: the exponential decay constant increases as the learning rate $s$ increases. Accordingly, the linear convergence to stationarity established in \Cref{sec:linear-convergence} is faster if $s$ is larger, and, by recognizing the exponential dependence of $\lambda_s$ on $s$, the convergence would be very slow if the learning rate $s$ is very small. For example, if $H_{f} = 0.05$, setting $s = 0.1$ and $s = 0.001$ gives
\begin{align*}
\frac{\lambda_{0.1}}{\lambda_{0.001}}  \approx \frac{\e^{-1}}{\e^{-100}}  = 9.889 \times 10^{42}.
\end{align*}



Moreover,as we will see clearly in \Cref{sec: estimate-rate}, $\lambda_s$ is completely determined by the \textit{geometry} of $f$. In particular, it does not depend on the probability distribution of the initial point or the dimension $d$ given that the constant $H_{f}$ has no direct dependence on the dimension $d$. For comparison, the linear rate in the nonconvex case is shown by \Cref{thm: continuous-quantative} to depend on the learning rate $s$, while the linear rate of convergence stays constant regardless of $s$ if the
objective is strongly convex. This fundamental distinction between the convex and nonconvex settings enables an interpretation of the observation brought up in \Cref{fig: sgd-deep}, in particular the right panel of \Cref{fig: sgd_nonconvex_traj-gen}. More precisely, with time $t$ being the $x$-axis, SGD with a larger learning rate leads to a faster convergence rate in the nonconvex setting, while for the (strongly) convex setting the convergence rate is independent of the learning rate. For further in-depth discussion of the implications of \Cref{thm: continuous-quantative} (see \Cref{sec:expl-benef-init}). 




\subsection{Discretization}
\label{sec:discrete-perspective}

In this subsection, we carry over the results developed from the continuous perspective to the discrete regime. In addition to assuming that the objective function $f$ satisfies the Villani condition, satisfies the confining condition, and is a Morse function, we also now assume $f$ to be $L$-smooth; that is, $f$ has $L$-Lipschitz continuous gradients in the sense that $\left\| \nabla f(x) - \nabla f(y) \right\| \leq L \|x - y\|$ for all $x, y$. Moreover, we restrict the learning rate $s$ to be no larger than $1/L$. The following proposition is the key theoretical tool that allows translation to the discrete regime.
\begin{prop}\label{prop: approx}
For any $L$-smooth objective $f$ and any initialization $X_{s}(0)$ drawn from a probability density $\rho \in L^{2}(\mu_s^{-1})$, the~lr-dependent SDE~\eqref{eqn: sgd_high_resolution_formally} has a unique global solution $X_s$  in expectation; that is, $\E X_s(t)$ as a function of $t$ in $C^{1}([0, + \infty); \mathbb{R}^{d})$ is unique. Moreover, there exists  $B(T) > 0$ such that the SGD iterates $x_{k}$ satisfy
\[
\max_{0 \leq k \leq T/s} \left| \E f(x_k) - \E f(X_s(ks)) \right| \leq B(T) s,
\]
for any fixed $T > 0$.
\end{prop}
We note that there exists a sharp bound on $B(T)$ in~\cite{bally1996law}. For completeness, we also remark that the convergence can be strengthened to the strong sense:
\[
 \max_{0 \leq k \leq T/s} \E   \left\| x_k - X_s(ks)  \right\| \leq B'(T)s.
\]
This result has appeared in \cite{mil1975approximate, talay1982analyse, pardoux1985approximation, talay1984efficient, kloeden1992approximation} and we provide a self-contained proof in Appendix~\ref{subsec: proof-approx}. 

We now state the main result of this subsection.
\begin{mainthm}\label{thm: main1}
In addition to the assumptions of \Cref{thm: continuous-qualitative}, assume that $f$ is $L$-smooth. Then, the following two conclusions hold:
\begin{enumerate}
\item[(a)]
For any $T > 0$, the iterates of SGD with learning rate $0 < s \le 1/L$ satisfy
\begin{equation}\label{eqn: final-estimate-sgd}
\E f(x_k) - f^{\star} \leq (A + B(T))s + C \left\| \rho - \mu_{s} \right\|_{\mu_s^{-1}} \e^{- s\lambda_{s} k},
\end{equation}
for all $k \le T/s$, where $\lambda_s$ is the exponential decay constant in~\eqref{eqn: continuous-qualitative}, $A$ as in \Cref{thm: compare_step} depends only on $1/L$ and $f$, $C = C_{1/L}$ is as in \Cref{prop:convv}, and $B(T)$ depends only on the time horizon $T$ and the Lipschitz constant $L$. 

\item[(b)] If $f$ is a Morse function with at least two local minima, with $\lambda_s$ appearing in \eqref{eqn: final-estimate-sgd} being given by \eqref{eqn: main-lambda-nonconvex}, and if $f$ is $\mu$-strongly convex then $\lambda_s = \mu$.
\end{enumerate}

\end{mainthm}

\Cref{thm: main1} follows as a direct consequence of \Cref{thm: continuous-qualitative} and \Cref{prop: approx}. Note that the second part of \Cref{thm: main1} is simply a restatement of \Cref{thm: continuous-quantative} and \Cref{prop:lambda_str_mu}. As earlier in the continuous-time formulation, we also mention that the dimension parameter $d$ is not an essential parameter for characterizing the rate of linear convergence. In relating to \Cref{fig: sgd_nonconvex_traj-gen}, note that its left panel with $k$ being the $x$-axis shows a faster linear convergence of SGD when using a larger learning rate, regardless of convexity or nonconvexity of the objective. This is because the linear rate $s \lambda_s$ in \eqref{eqn: final-estimate-sgd} is always an increasing function of $s$ even for the strongly convex case, where $\lambda_s$ itself is constant.


\subsection{A one-dimensional example}
\label{sec:one-dimens-example}

In this section we provide some intuition for the theoretical results presented in the preceding subsections. Our priority is to provide intuition rather than rigor. Consider the simple example of $f$ presented in \Cref{fig:one_d_barrier}, which has a global minimum $x^\star$, a local minimum $x^\bullet$, and a local maximum $x^\circ$.\footnote{We can also regard $x^\circ$ as a saddle point in the sense that the Hessian at this point has one negative eigenvalue. See~\Cref{sec:basics-morse-theory} for more discussion.}  We use this toy example to gain insight into the expression \eqref{eqn: main-lambda-nonconvex} for the exponential decay constant $\lambda_s$; deferring the rigorous derivation of this number in the general case to \Cref{sec: estimate-rate}.

From \eqref{eqn: continuous-qualitative} it suggests that the lr-dependent SDE~\eqref{eqn: sgd_high_resolution_formally} takes about $O(1/\lambda_s)$ time to achieve approximate stationarity. Intuitively, for the specific function in \Cref{fig:one_d_barrier}, the bottleneck in achieving stationarity is to pass through the local maximum $x^\circ$. Now, we show that it takes about $O(1/\lambda_s)$ time to pass $x^\circ$ from the local minimum $x^\bullet$. For simplicity, write
\[
f(x) = \frac{\theta}{2} (x - x^\bullet)^2 + g(x),
\]
where $g(x) = f(x^\bullet)$ stays constant if $x \le x^\circ - \nu$ for a very small positive $\nu$ and $\theta > 0$. Accordingly, the lr-dependent SDE~\eqref{eqn: sgd_high_resolution_formally} is reduced to the Ornstein--Uhlenbeck process,
\[
\dd X_s = -\theta (X_s - x^\bullet) \dd t + \sqrt{s} \dd W,
\]
before hitting $x^\circ$. Denote by $\tau_{x^\circ}$ the first time the Ornstein--Uhlenbeck process hits $x^\circ$. It is well known that the hitting time obeys
\begin{equation}\label{eq:hitting_t}
\begin{aligned}
\E \tau_{x^\circ} &\approx \frac{\sqrt{\pi s}}{(x^\circ - x^\bullet)\theta\sqrt{\theta}} \, \e^{\frac{2}{s} \cdot \frac12 \theta (x^\circ - x^\bullet)^2}\approx \frac{\sqrt{\pi s}}{(x^\circ - x^\bullet)\theta\sqrt{\theta}} \, \e^{\frac{2 H_{f}}{s}},
\end{aligned}
\end{equation}
where $H_{f} := f(x^\circ) - f(x^\bullet) \approx f(x^\circ) - g(x^\circ) =  \frac12 \theta (x^\circ - x^\bullet)^2$. This number, which we refer to as the \textit{Morse saddle barrier}, is the difference between the function values at the local maximum $x^\circ$ and the local minimum $x^\bullet$ in our case. As an implication of \eqref{eq:hitting_t}, the continuous-time formulation of SGD takes time (at least) of the order $\e^{(1 + o(1)) \frac{2 H_{f}}{s}}$ to achieve approximate stationarity. This is consistent with the exponential decay constant $\lambda_s$ given in~\eqref{eqn: main-lambda-nonconvex}.


\begin{figure}[htb!]
\centering
\begin{tikzpicture}[scale=0.7]
     \draw  plot[smooth, tension=.7] coordinates{(0, 6) (2.5, 2) (5, 4)  (7.5, 0.5) (10, 6)};
     \draw[dashed] (2.65,1.98) -- (6.30,1.98);
     \draw[dashed] (4.9,4.0) -- (4.9,2);
     \node at (2.65,1.8) {$x^{\bullet}$};
     \node at (5.0,4.2) {$x^\circ$};
     \node at (7.5, 0.3) {$x^{\star}$};
     \node at (4.4,3.0) {$H_{f}$};
     \draw (4.5,4.0) -- (4.9, 4.0);
     \draw (4.5,2.0) -- (4.9, 2.0);
     \draw[->, thick] (4.7,2.85) -- (4.7, 2.0);
     \draw[->, thick] (4.7,3.25) -- (4.7, 4.0);
\end{tikzpicture}
\caption{\small A one-dimensional nonconvex function $f$. The height difference between $x^\circ$ and $x^\bullet$ in this special case is the Morse saddle barrier $H_f$. See the formal definition in \Cref{def:barrier}.}
\label{fig:one_d_barrier}
\end{figure}
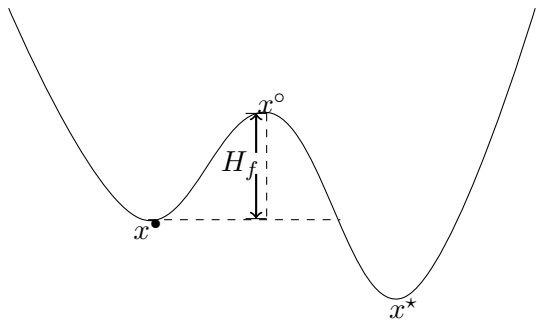

In passing, we remark that the discussion above can be made rigorous by invoking the theory of the Kramers escape rate, which shows that for this univariate case the hitting time satisfies
\[
\E \tau_{x^\circ} = (1 + o(1)) \frac{\pi}{\sqrt{-f''(x^\bullet) f''(x^\circ)}} \e^{\frac{2H_{f}}{s}}.
\]
See, for example, \cite{freidlin2012random,pavliotis2014stochastic}. Furthermore, we demonstrate the view from the theory of \textit{viscosity solution} and \textit{singular perturbation} in~\Cref{subsec: supplement-viscosity}.



\section{Why Learning Rate Decay?}
\label{sec:expl-benef-init}

As a widely used technique for training neural networks, learning rate decay refers to taking a large learning rate initially and then progressively reducing it during the training process. This technique has been observed to be highly effective especially in the minimization of nonconvex objective functions using stochastic optimization methods, with a very recent strand of theoretical effort toward understanding its benefits~\cite{you2019learning,li2019towards}. In this section, we offer a new and crisp explanation by leveraging the results in \Cref{sec:main-results}. To highlight the intuition, we primarily work with the continuous-time formulation of SGD. 

\begin{figure}[htb!]
\begin{minipage}[t]{0.321\linewidth}
\centering
\includegraphics[scale=0.11]{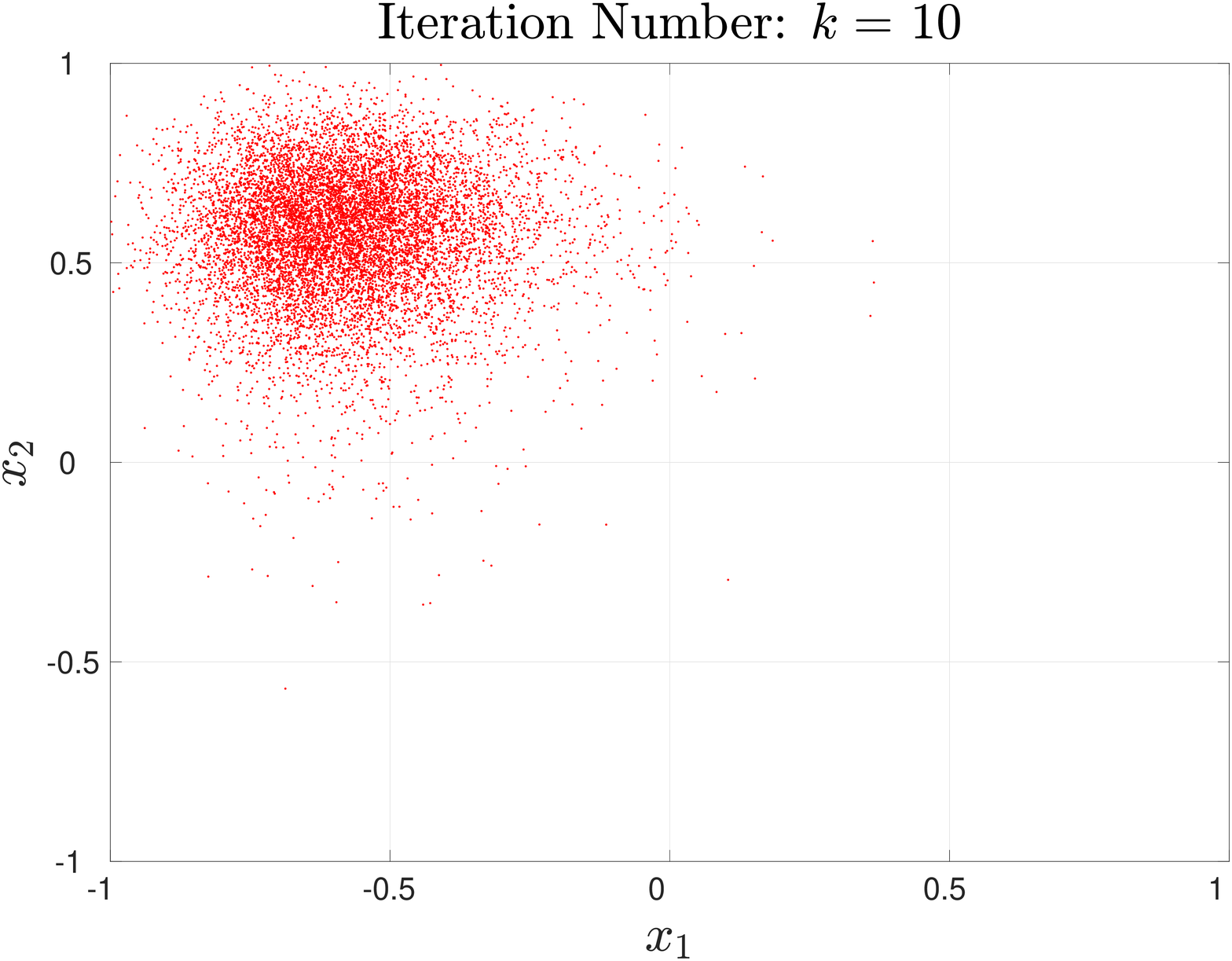}
\end{minipage}
\begin{minipage}[t]{0.321\linewidth}
\centering
\includegraphics[scale=0.11]{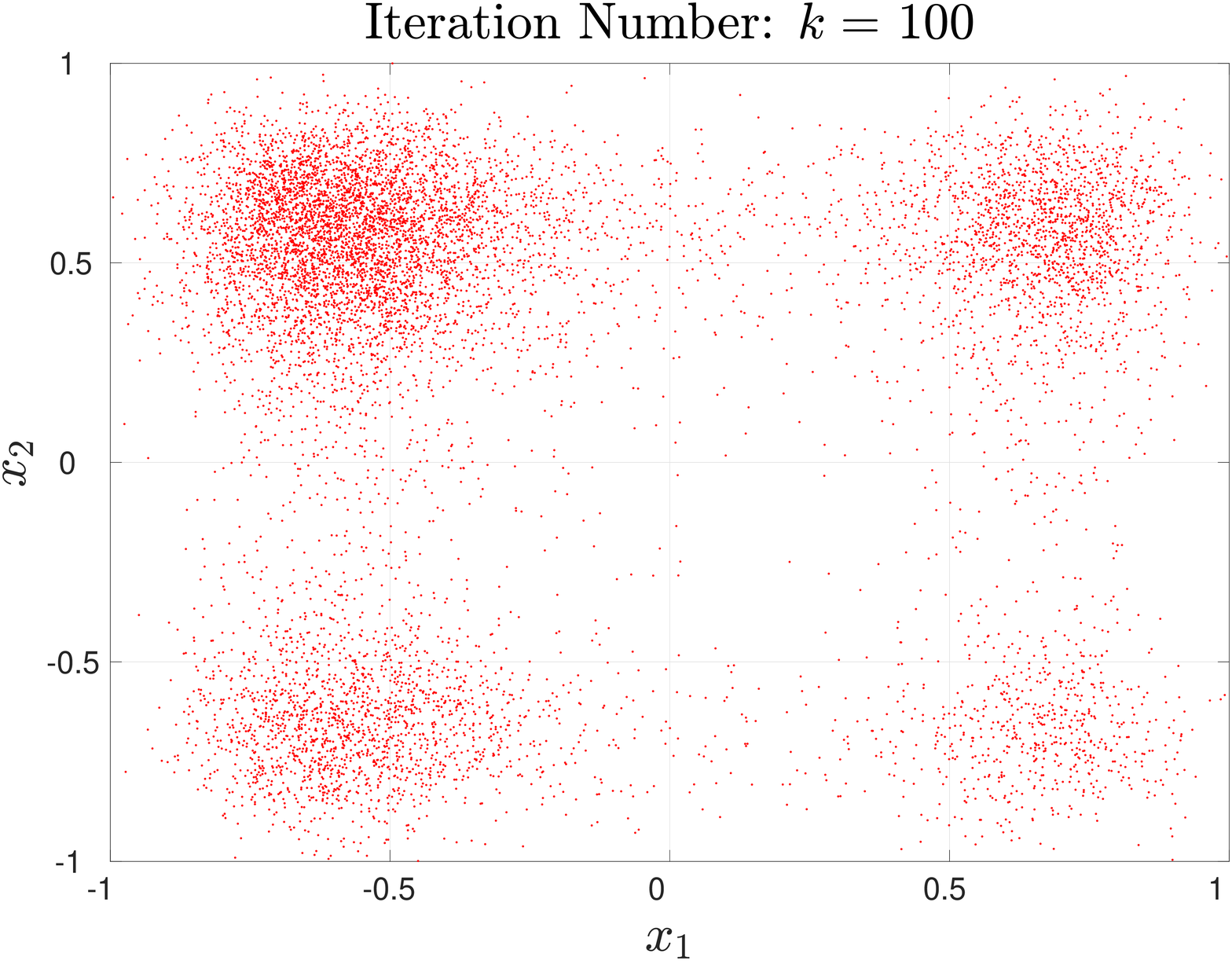}
\end{minipage}
\begin{minipage}[t]{0.321\linewidth}
\centering
\includegraphics[scale=0.11]{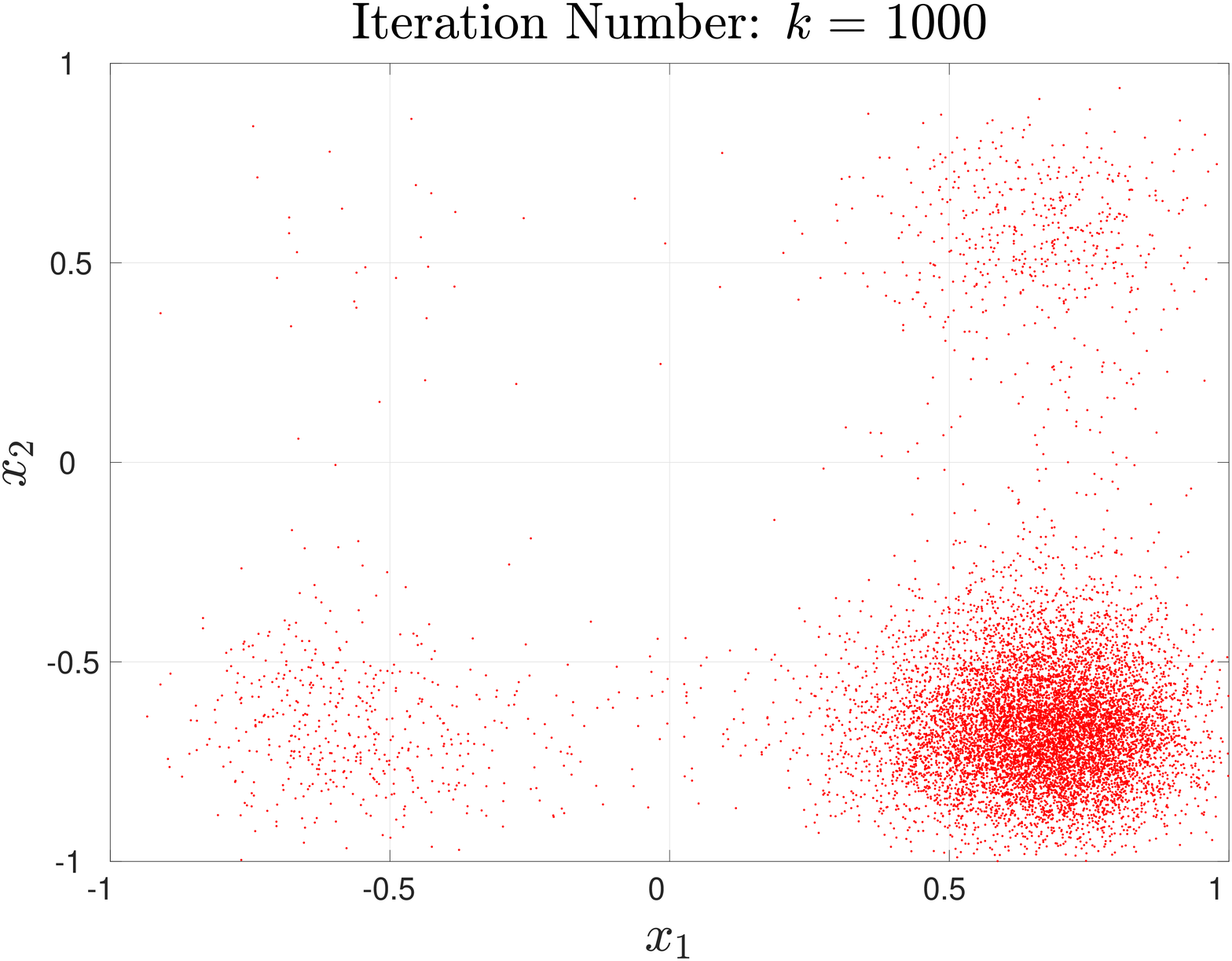}
\end{minipage}
\\
\begin{minipage}[t]{0.321\linewidth}
\centering
\includegraphics[scale=0.11]{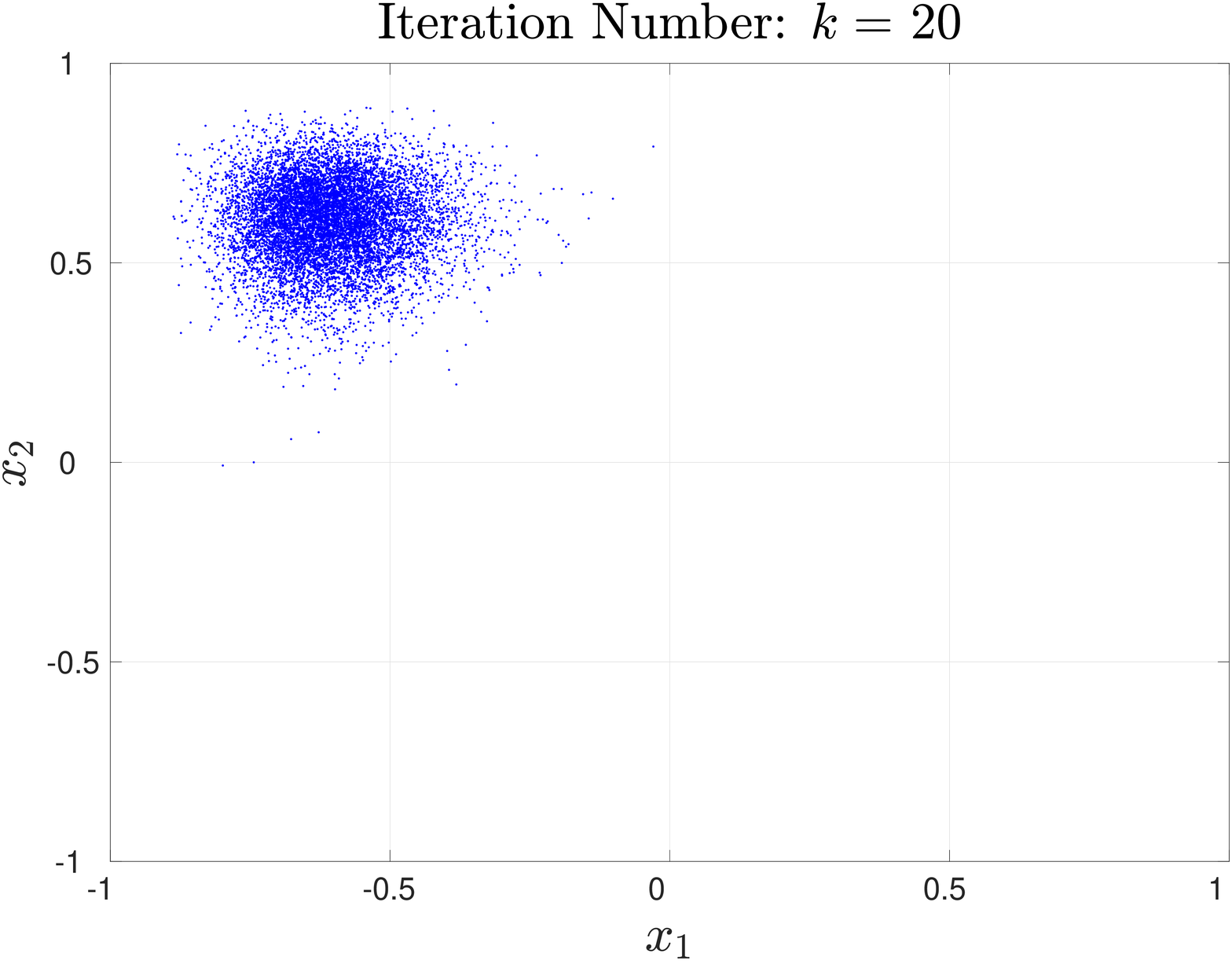}
\end{minipage}
\begin{minipage}[t]{0.321\linewidth}
\centering
\includegraphics[scale=0.11]{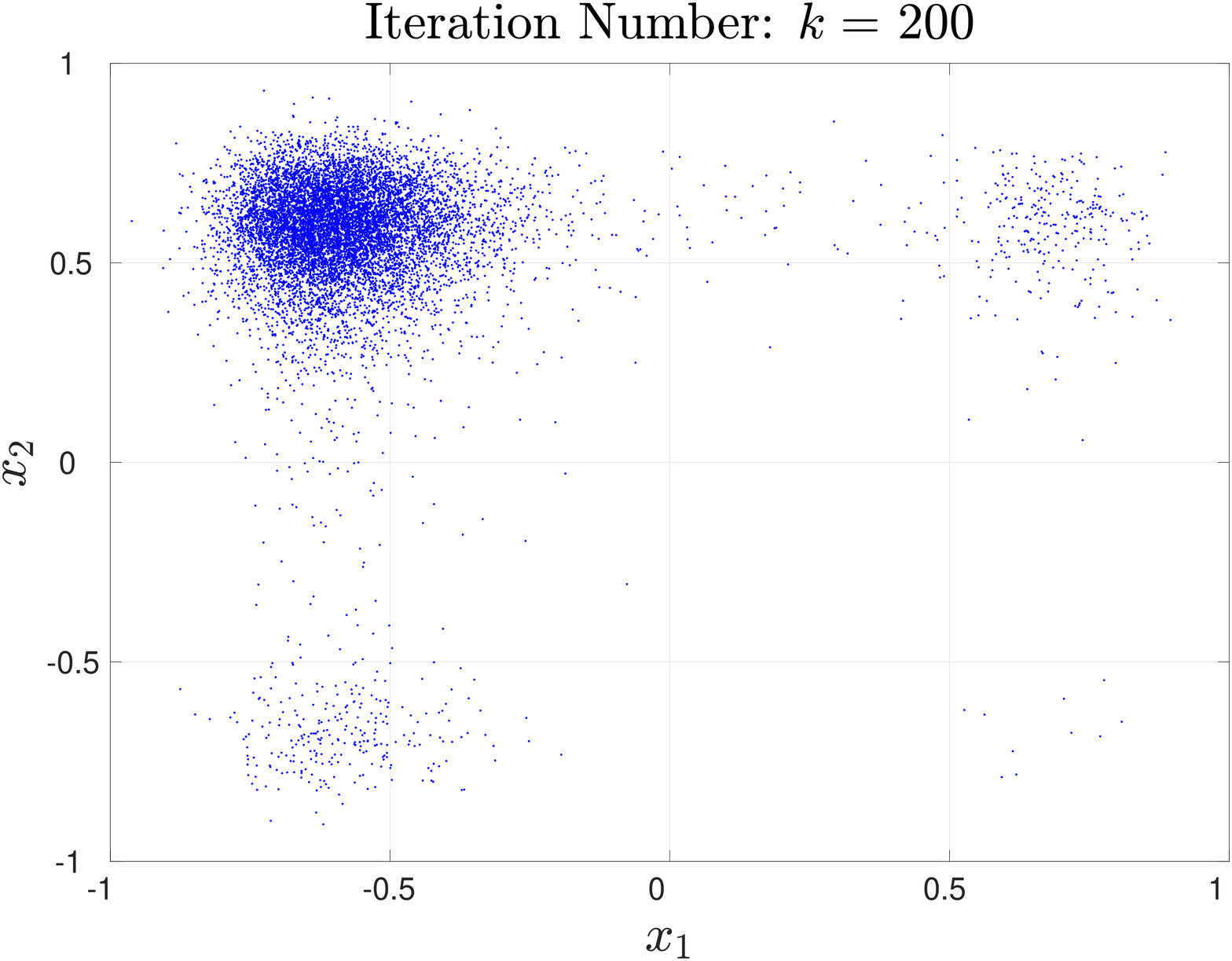}
\end{minipage}
\begin{minipage}[t]{0.321\linewidth}
\centering
\includegraphics[scale=0.11]{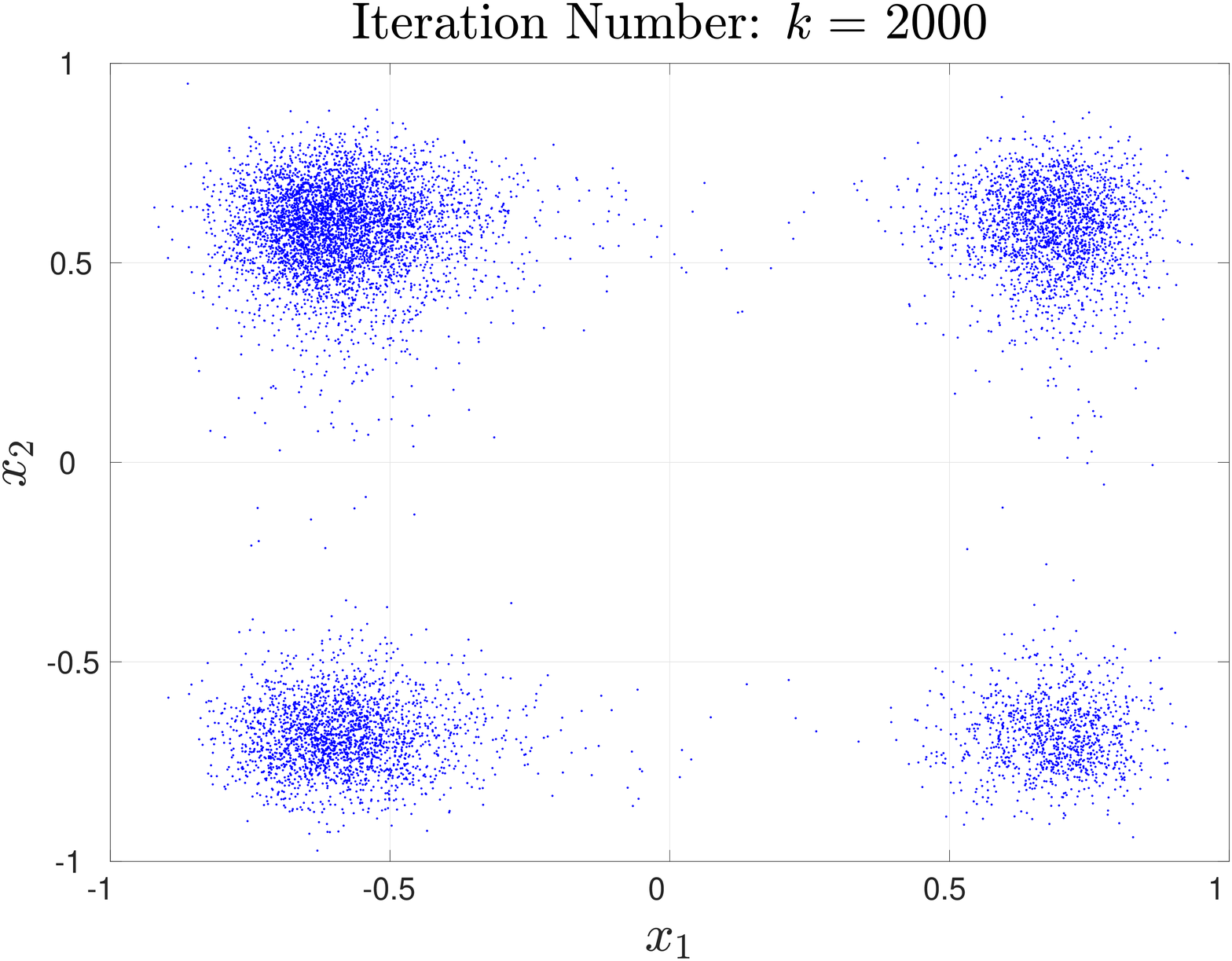}
\end{minipage}
\caption{\small Scatter plots of the iterates $x_k \in \R^2$ of SGD for minimizing the nonconvex function in \Cref{fig: sgd_nonconvex_traj-gen}. This function has four local minima, of which the bottom right one is the gloabl minimum. Each column corresponds to the same value of $t = ks$, and the first row and second row correspond to learning rates $0.1$ and $0.05$, respectively. The gradient noise is drawn from the standard normal distribution. Each plot is based on 10000 independent SGD runs using the noise generator ``state 1-10000'' in Matlab2019b, starting from an initial point $(-0.9, 0.9)$. } 
\label{fig: sgd_distri}
\end{figure}


For purposes of illustration, \Cref{fig: sgd_distri} presents numerical examples for this technique where the learning rate is set to $0.1$ or $0.05$. This figure clearly demonstrates that SGD with a larger learning rate converges much faster to the global minimum than SGD with a smaller learning rate. This comparison reveals that a large learning rate would render SGD able to quickly explore the landscape of the objective function and efficiently escape bad local minima. On the other hand, a larger learning rate would prevent SGD iterates from concentrating around a global minimum, leading to substantial suboptimality. This is clearly illustrated in \Cref{fig: sgd_distri_final}. As suggested by the heuristic work on learning rate decay, we see that it is important to decrease the learning rate to achieve better optimization performance whenever the iterates arrive near a local minimum of the objective function.

\begin{figure}[htb!]
\centering
\begin{minipage}[t]{0.45\linewidth}
\centering
\includegraphics[scale=0.15]{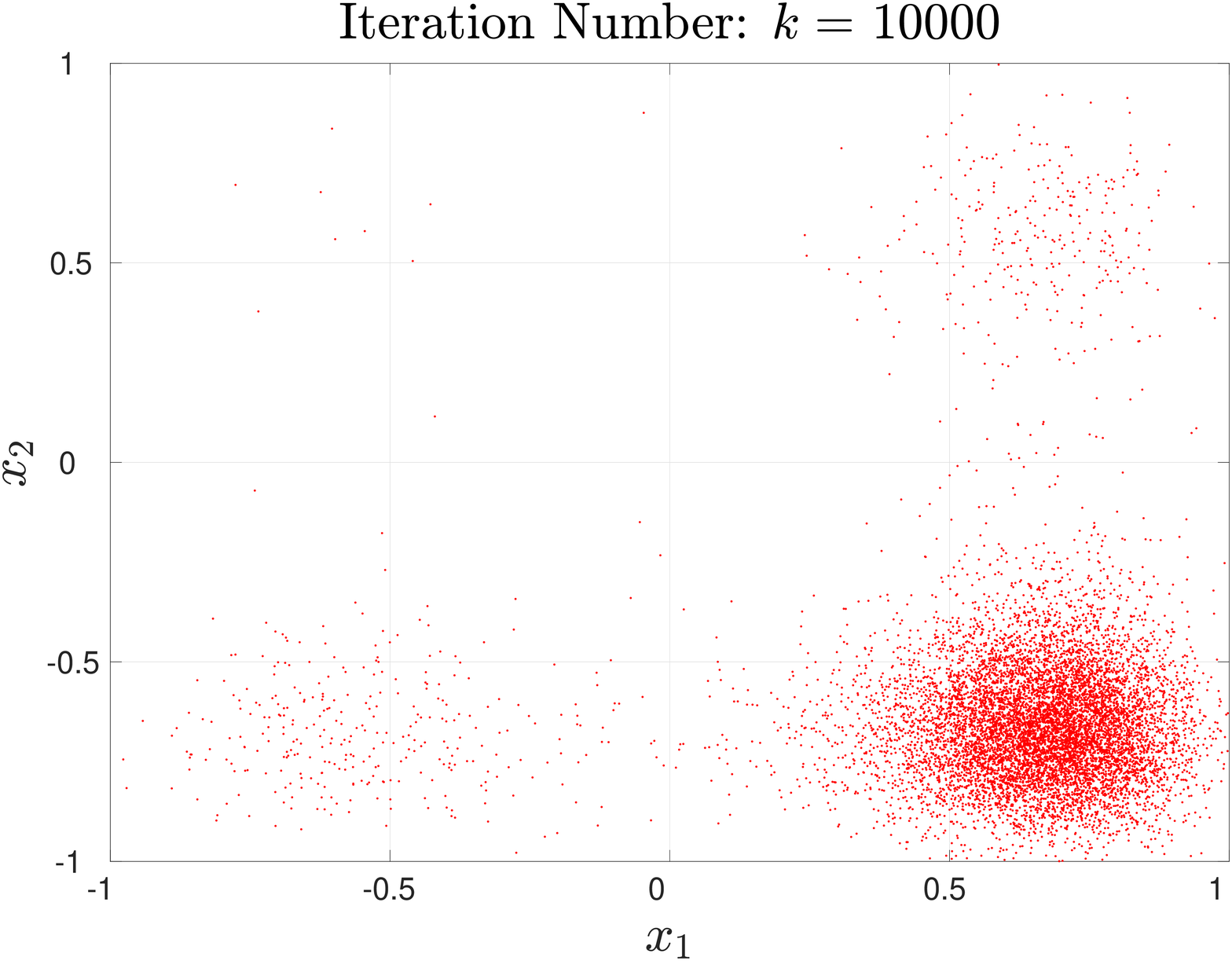}
\end{minipage}
\begin{minipage}[t]{0.45\linewidth}
\centering
\includegraphics[scale=0.15]{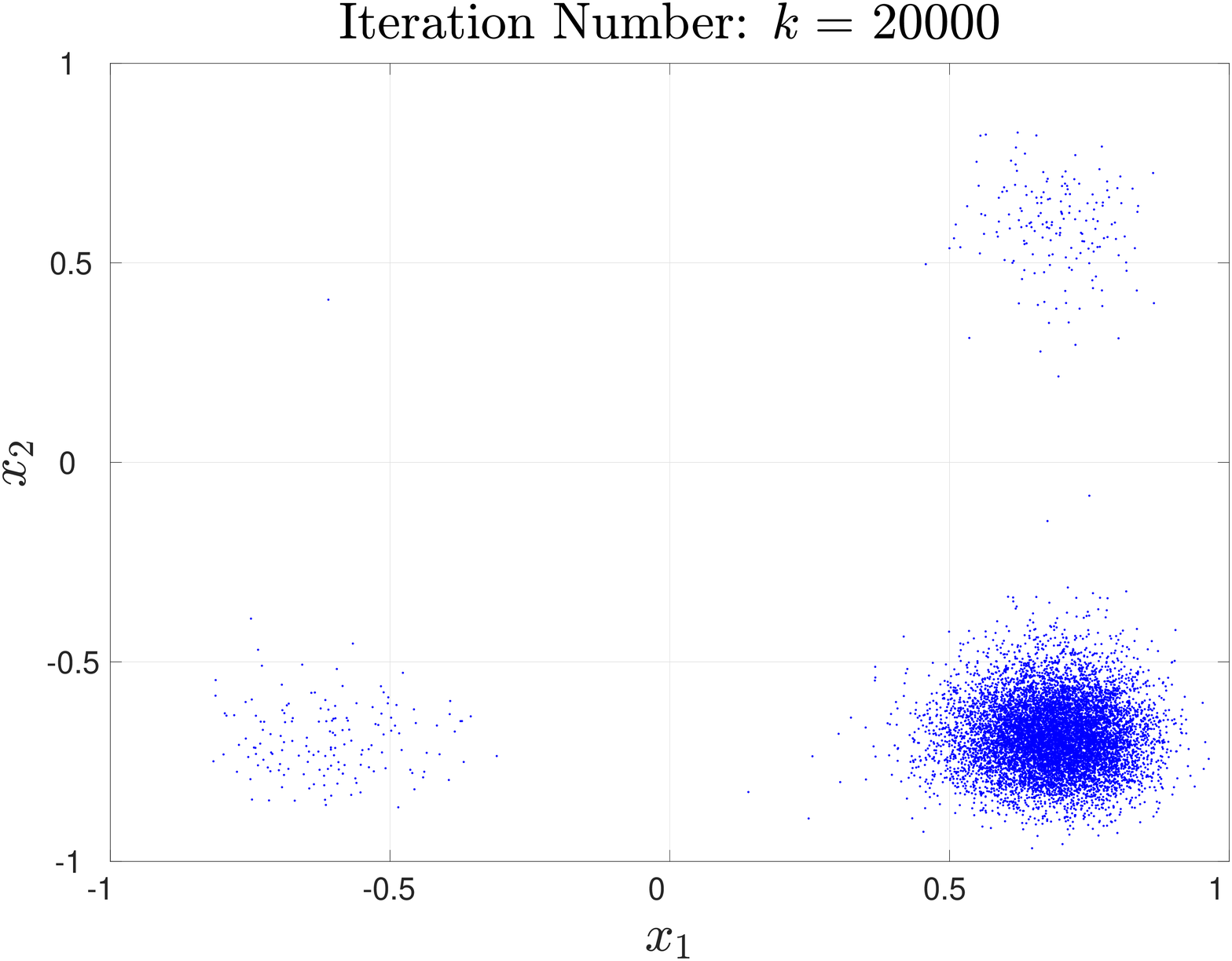}
\end{minipage}
\caption{\small The same setting as in \Cref{fig: sgd_distri}. Both plots correspond to the same value of $t =ks= 1000$.} 
\label{fig: sgd_distri_final}
\end{figure}


Despite its intuitive plausibility, the exposition above stops short of explaining why nonconvexity of the objective is crucial to the effectiveness of learning rate decay. Our results in \Cref{sec:main-results}, however, enable a concrete and crisp understanding of the vital importance of nonconvexity in this setting. Motivated by \eqref{eq:as_strong}, we consider an idealized risk function of the form $R(t) = as +  b\e^{-\lambda_s t}$, with $\lambda_s$ set to $\e^{-c/s}$, where $a, b$, and $c$ are positive constants for simplicity as opposed to the non-constants in the upper bound in \eqref{eqn: continuous-qualitative}. This function is plotted in \Cref{fig: exp2_s}, with two quite different learning rates, $s_1 = 0.1$ and $s_2 = 0.001$, as an implementation of learning rate decay. When the learning rate is $s_1 = 0.1$, from the right panel of \Cref{fig: exp2_s}, we see that rough stationarity is achieved at time $t = k s \approx 25$; thus, the number of iterations $k_{0.1} \approx 25/s = 250$. In the case of $s = 0.001$, from the left panel of \Cref{fig: exp2_s}, we see now it requires $ks \approx 2.5 \times 10^{44}$ to reach rough stationarity, leading to $k_{0.001} \approx 2.5 \times 10 ^{47}$. This gives
\[
\frac{k_{0.001}}{k_{0.1}} \approx 10^{45}.
\]
In contrast, the sharp dependence of $k_s$ on the learning rate $s$ is not seen for strongly convex functions, because $\lambda_s = \mu$ stays constant as the learning rate $s$ varies. Following the preceding example, we have
\[
\frac{k_{0.001}}{k_{0.1}} \approx 10^2.
\]

\begin{figure}[htb!]
\centering
\begin{minipage}[t]{0.45\linewidth}
\centering
\includegraphics[scale=0.15]{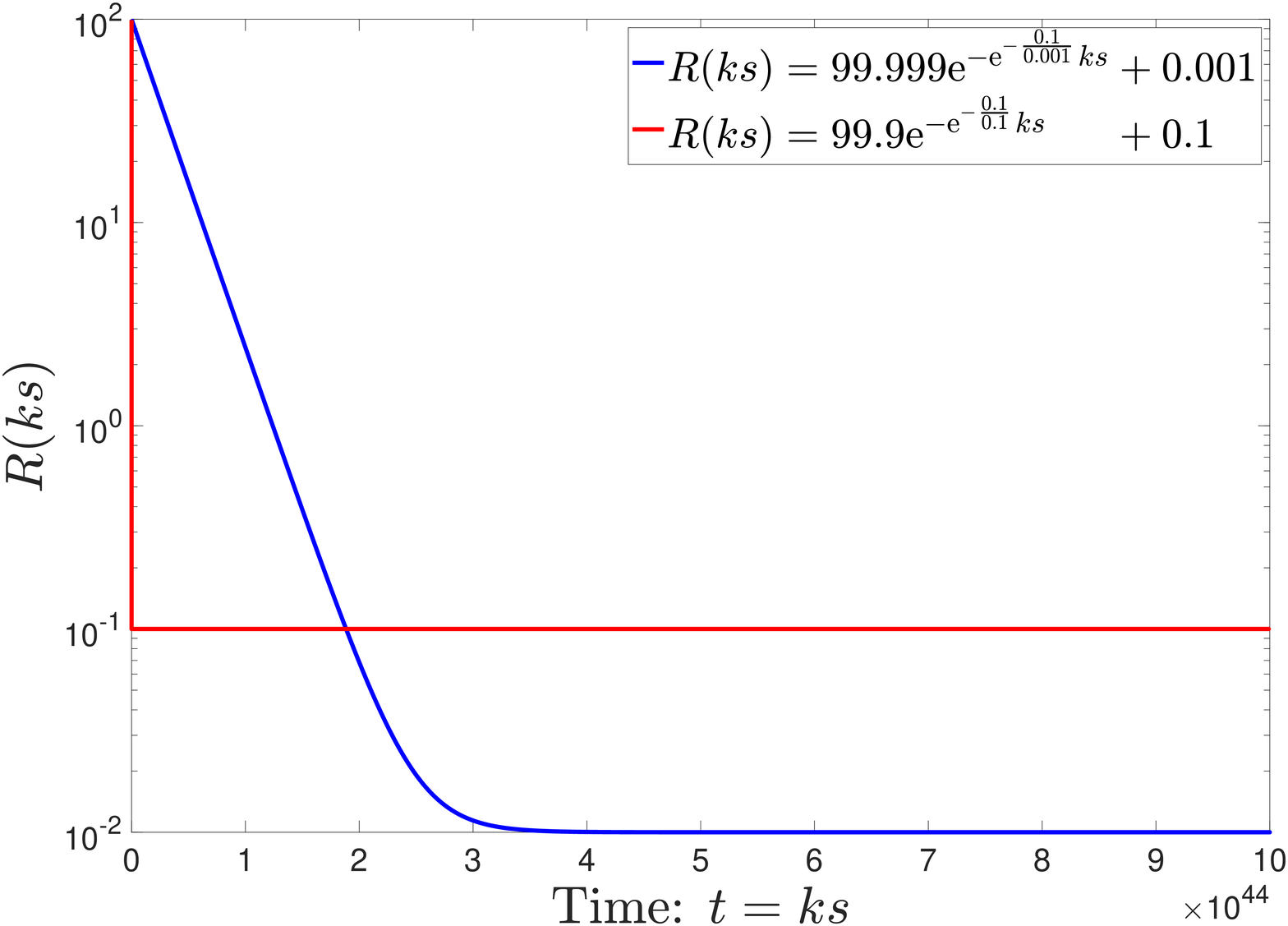}
\end{minipage}
\begin{minipage}[t]{0.45\linewidth}
\centering
\includegraphics[scale=0.15]{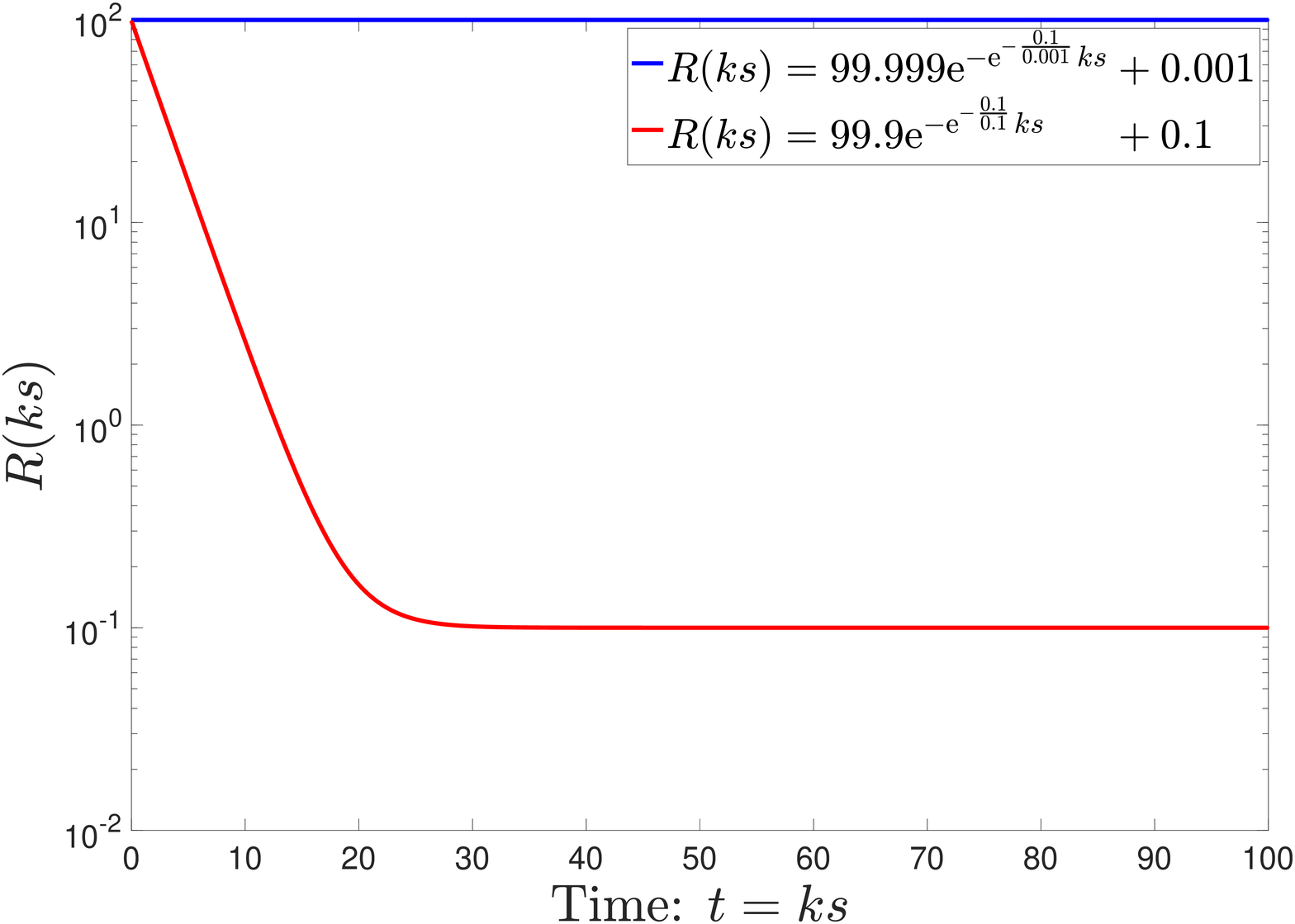}
\end{minipage}
\caption{\small Idealized risk function of the form $R(t) = as +  b\e^{-\e^{\frac{c}{s}} t}$ with the identification $t = ks$, which is adapted from \eqref{eq:as_strong}. The parameters are set as follows: $a = 1, b = 100 - s, c = 0.1$, and the learning rate is $s = 0.1$ or $0.001$. The right plot is a locally enlarged image of the left.} 
\label{fig: exp2_s}
\end{figure}


While a large initial learning rate helps speed up the convergence, \Cref{fig: exp2_s} also demonstrates that a larger learning rate leads to a larger value of the excess risk at stationarity, $\epsilon(s) \equiv \E f(X_s(\infty))  - f^\star$, which is indeed the claim of \Cref{thm: compare_step}. Leveraging \Cref{prop:convv}, we show below why annealing the learning rate at some point would improve the optimization performance. To this end, for any fixed learning rate $s$, consider a stopping time $T_s^{\delta}$ that is defined as 
\[
T_s^{\delta} := \inf_t \left\{ \left| \E  f(X_s(t)) - \E f(X_s(\infty))  \right| \le  \delta \epsilon(s) \right\},
\]
for a small $\delta > 0$. In words, the lr-dependent SDE~\eqref{eqn: sgd_high_resolution_formally} at time $T_s^{\delta}$ is approximately stationary since its risk $\E f(X_s(t)) - f^\star$ is mainly comprised of the excess risk at stationarity $\epsilon(s)$, with a total risk of no more than $(1 + \delta)\epsilon(s)$. From \Cref{prop:convv} it follows that (recall that $\rho$ is the initial distribution):
\begin{equation}\label{eq:t_d_s}
T_s^{\delta} \le \frac1{\lambda_s} \log \frac{C(s) \left\| \rho - \mu_{s} \right\|_{\mu_s^{-1}}}{\delta \epsilon(s)} = \frac{\e^{\frac{2H_{f}}{s}}}{\gamma + o(s)} \log \frac{C(s) \left\| \rho - \mu_{s} \right\|_{\mu_s^{-1}}}{\delta \epsilon(s)}.
\end{equation}
In addition to taking a large $s$, an alternative way to make $T_s^{\delta}$ small is to have an initial distribution $\rho$ that is close to the stationary distribution $\mu_s$. This can be achieved by using the technique of learning rate decay. More precisely, taking a larger learning rate $s_1$ for a while, at the end the distribution of the iterates is approximately the stationary distribution $\mu_{s_1}$, which serves as the initial distribution for SGD with a smaller learning rate $s_2$ in the second phase. Taking $\rho \approx \mu_{s_1}$, the factor $\left\| \rho - \mu_{s} \right\|_{\mu_s^{-1}}$ in \eqref{eq:t_d_s} for the second phase of learning rate decay is approximately
\begin{equation}\label{eq:c_small_begin}
\left\| \mu_{s_1} - \mu_{s_2} \right\|_{\mu_{s_2}^{-1}} = \left(\int (\mu_{s_1} - \mu_{s_2})^2 \mu_{s_2}^{-1} \dd x \right)^{\frac12} = \left(\int \frac{\mu_{s_1}^2}{\mu_{s_2}} \dd x - 1\right)^{\frac12}.
\end{equation}
Both $\mu_{s_1}$ and $\mu_{s_2}$ are decreasing functions of $f$ and, therefore, have the same modes. As a consequence, the integral of $\mu_{s_1}^2/\mu_{s_2}$ is small by appeal to the rearrangement inequality, thereby leading to fast convergence of SGD with learning rate $s_2$ to the stationary risk $\epsilon(s_2)$. In contrast, $\|\rho - \mu_{s_2}\|_{\mu_{s_2}^{-1}}$ would be much larger for a general random initialization $\rho$. Put simply, SGD with learning rate $s_2$ cannot achieve a risk of approximately $\epsilon(s_2)$ given the same number of iterations \textit{without} the warm-up stage using learning rate $s_1$. See \Cref{fig:arrow-large-small} for an illustration.


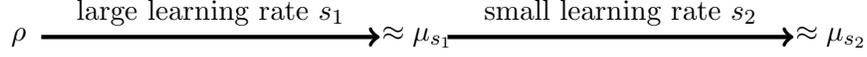
\begin{figure}[htb!]
\centering
\begin{tikzpicture}
     \node at (-0.3,0) {$\rho$};
     \node at (5,0) {$\approx \mu_{s_1}$};
     \node at (10.5, 0) {$\approx \mu_{s_2}$};
     \draw[->, ultra thick] (0,0) -- node [above]{large learning rate $s_1$} (4.5, 0);
     \draw[->, ultra thick] (5.4,0) --  node [above]{small learning rate $s_2$} (10, 0);
\end{tikzpicture}
\caption{\small Learning rate decay. The first phase uses a larger learning rate $s_1$, at the end of which the SGD iterates are approximately distributed as $\mu_{s_1}$. The second phase uses a smaller learning rate $s_2$ and at the end the distribution of the SGD iterates roughly follows $\mu_{s_2}$.}
\label{fig:arrow-large-small}
\end{figure}

\section{Proof of the Linear Convergence}
\label{sec: convergence-rate}

In this section, we prove \Cref{prop:convv} and \Cref{thm: compare_step}, leading to a complete proof of \Cref{thm: continuous-qualitative}.



\subsection{Proof of \Cref{prop:convv}}
\label{sec:proof-crefprop:convv}

To better appreciate the linear convergence of the lr-dependent SDE~\eqref{eqn: sgd_high_resolution_formally}, as established in \Cref{prop:convv}, we start by showing the convergence to stationarity without a rate. In fact, this intermediate result constitutes a necessary step in the proof of \Cref{prop:convv}.

\paragraph{Convergence without a rate.} Recall that we use $\rho$ to denote the initial probability density in the space $L^{2}(\mu_s^{-1})$. Superficially, it seems that the most natural space for probability densities is $L^{1}(\mathbb{R}^{d})$. However, we prefer to work in $L^{2}(\mu_s^{-1})$ since this function space has certain appealing properties that allow us to obtain the proof of the desired convergence results for the lr-dependent SDE. Formally, the following result says that any (nonnegative) function in $L^{2}(\mu_s^{-1})$ can be normalized to be a density function. The proof of this simple lemma is shown in Appendix~\ref{subsec: space-holder}. 


\begin{lem}\label{lem: space-holder}
Let $f$ satisfy the confining condition. Then, $L^{2}(\mu_s^{-1})$ is a subset of $L^{1}(\mathbb{R}^{d})$.
\end{lem}

The following result shows that the solution to the lr-dependent SDE converges to stationarity in terms of the dynamics of its probability densities over time.


\begin{lem}\label{thm: converge}
Let $f$ satisfy the confining condition and denote the initial distribution as $\rho \in L^{2}(\mu_s^{-1})$. Then, the unique solution $\rho_s(t, \cdot) \in C^{1} \left( [0, +\infty), L^{2}(\mu_s^{-1}) \right)$ to the Fokker--Planck--Smoluchowski equation~\eqref{eqn: Fokker-Planck} converges in $L^{2}(\mu_s^{-1})$ to the Gibbs invariant distribution $\mu_s$, which is specified by~\eqref{eqn: Gibbs}.

\end{lem}

Note that the existence and uniqueness of $\rho_s(t, \cdot)$ is ensured by \Cref{prop: unique-existence}. The convergence guarantee on $\rho_s(t, \cdot)$ in \Cref{thm: converge} relies heavily on the following lemma (\Cref{lem: equivlent-h}). This preparatory lemma introduces the transformation 
\[
h_s(t, \cdot) =\rho_s(t, \cdot) \mu_{s}^{-1} \in C^{1} \left( [0, +\infty), L^{2}(\mu_s) \right),
\]
which allows us to work in the space $L^{2}(\mu_s)$ in place of $L^{2}(\mu_s^{-1})$ (a measurable function $g$ is said to belong to $L^2(\mu_s)$ if $\|g\|_{\mu_s}  := \left( \int_{\mathbb{R}^{d}} g^{2} \dd \mu_s \right)^{\frac{1}{2}} < + \infty$\footnote{Here, $\dd \mu_s$ stands for the probability measure $\dd \mu_s \equiv \mu_s \dd x = \frac1{Z_s} \exp(-2f/s) \dd x$.}). It is not hard to show that $h_s$ satisfies the following equation
\begin{equation}\label{eqn: FPS-equiv}
\frac{\partial h_s}{\partial t} = -\nabla f \cdot \nabla h_s + \frac{s}{2} \Delta h_s,
\end{equation}
with the initial distribution $h_s(0, \cdot) = \rho \mu_s^{-1} \in  L^{2}(\mu_s)$. The linear operator 
\begin{equation}\label{eqn: gene-oper}
\mathscr{L}_s =  - \nabla f \cdot \nabla  + \frac{s}{2} \Delta 
\end{equation}
has a crucial property, as stated in the following lemma. Its proof is postponed to Appendix~\ref{subsec: equiv-h-rho}.


\begin{lem}\label{lem: equivlent-h}
The linear operator $\mathscr{L}_s$ in~\eqref{eqn: gene-oper} is self-adjoint and nonpositive in $L^{2}(\mu_s)$. Explicitly, for any $g_1, g_2$, this operator obeys
\[
\int_{\mathbb{R}^{d}} (\mathscr{L}_s g_1 ) g_2 \dd \mu_s = \int_{\mathbb{R}^{d}} g_1 \mathscr{L}_s g_2 \dd \mu_s = -\frac{s}{2} \int_{\mathbb{R}^{d}} \nabla g_1 \cdot \nabla g_2 \dd \mu_s.
\]
\end{lem}

\begin{proof}[Proof of \Cref{thm: converge}]
We have
\begin{align}
\frac{\dd}{\dd t} \left\| \rho_{s}(t, \cdot)- \mu_{s} \right\|_{\mu_s^{-1} } ^{2}     & =   \frac{\dd}{\dd t} \left\| h_s(t, \cdot) - 1 \right\|^{2}_{\mu_s} \nonumber \\
            & =   \frac{\dd}{\dd t} \int_{\mathbb{R}^{d}} \left( h_s(t, x) - 1 \right)^{2} \dd \mu_s  \nonumber \\
             & = 2 \int_{\mathbb{R}^{d}} (h_s - 1) \mathscr{L}_s (h_s - 1)\dd \mu_s \nonumber,
\end{align}
where the last equality is due to \eqref{eqn: FPS-equiv}. Next, we proceed by making use of \Cref{lem: equivlent-h}:
\begin{align}
2 \int_{\mathbb{R}^{d}} (h_s - 1)  \mathscr{L}_s (h_s - 1) \dd \mu_s & = -s\int_{\mathbb{R}^{d}} \nabla (h_s - 1) \cdot \nabla (h_s - 1) \dd \mu_s \nonumber\\
             & = -s\int_{\mathbb{R}^{d}} \| \nabla h_s\|^{2} \dd \mu_s \le 0. \label{eqn: converge-FK}
\end{align}
Thus, $\left\| \rho_{s}(t, \cdot)- \mu_{s} \right\|_{\mu_s^{-1} } ^{2}$ is a strictly decreasing function, decreasing asymptotically towards the equilibrium state
\[
\int_{\mathbb{R}^{d}} \| \nabla h_s\|^{2} \dd \mu_s = 0.
\]
This equality holds, however, only if $h_s(t, \cdot)$ is constant. Because both $\rho_s(t, \cdot)$ and $\mu_s$ are probability densities, this case must imply that $h_s(t, \cdot) \equiv 1$; that is, $ \rho_{s}(t, \cdot)\equiv \mu_{s} $. Therefore, $\rho_s(t, \cdot) \in C^{1} \left( [0, +\infty), L^2(\mu_s^{-1}) \right)$ converges to the Gibbs invariant distribution $\mu_s$ in $L^2(\mu_s^{-1})$.
\end{proof}



\paragraph{Linear convergence.} We turn towards the proof of linear convergence.  We first state a lemma which serves as a fundamental tool for us to prove a linear rate of convergence for \Cref{prop:convv}.  

\begin{lem}[Theorem A.1 in~\cite{villani2009hypocoercivity}]\label{thm: villani-poincare-inq}
If $f$ satisfies both the confining condition and the Villani condition, then there exists $\lambda_{s} > 0$ such that the measure $\dd \mu_s$ satisfies the following Poincar\'e-type inequality 
\begin{equation}\nonumber
\int_{\mathbb{R}^{d}} h^{2} \dd \mu_{s} - \left( \int_{\mathbb{R}^{d}} h \dd \mu_{s} \right)^{2} \leq  \frac{s}{2\lambda_{s}}\int_{\mathbb{R}^{d}} \|\nabla h\|^{2} \dd \mu_{s},
\end{equation}
for any $h$ such that the integrals above are well-defined.
\end{lem}


For completeness, we provide a proof of this Poincar\'{e}-type inequality in \Cref{subsec: villani-poincare-inq}. For comparison, the usual Poincar\'{e} inequality is put into use for a bounded domain, as opposed to the entire Euclidean space as in \Cref{thm: villani-poincare-inq}. In addition, while the constant in the Poincar\'{e} inequality in general depends on the dimension (see, for example, \cite[Theorem 1, Chapter 5.8]{evans2010partial}), $\lambda_s$ in \Cref{thm: villani-poincare-inq} is completely determined by geometric properties of the objective $f$. See details in \Cref{sec: estimate-rate}.



Importantly, \Cref{thm: villani-poincare-inq} allows us to obtain the following lemma, from which the proof of \Cref{prop:convv} follows readily. The proof of this lemma is given at the end of this subsection.
\begin{lem}\label{thm: rate-l2-1} 
Under the assumptions of \Cref{prop:convv}, $\rho_s(t, \cdot)$ converges to the Gibbs invariant distribution $\mu_s$ in $L^{2}(\mu_s^{-1})$ at the rate
\begin{equation}\label{eqn: rate-l2-1}
\left\| \rho_{s}(t, \cdot) - \mu_{s} \right\|_{\mu_{s}^{-1}} \leq \e^{- \lambda_{s} t }\left\|\rho - \mu_{s} \right\|_{\mu_{s}^{-1}}.
\end{equation}
\end{lem}

\begin{proof}[Proof of \Cref{prop:convv}]
Using \Cref{thm: rate-l2-1}, we get
\begin{align*}
\left| \E f(X_s(t)) - \E f(X(\infty)) \right|  & = \left|  \int_{\mathbb{R}^{d}} f(x) \left( \rho_{s}(t, x) - \mu_s(x) \right) \dd x\right| \\
&= \left|  \int_{\mathbb{R}^{d}} (f(x) - f^{\star}) \left( \rho_{s}(t, x) - \mu_s(x) \right) \dd x\right| \\
                                                                    & \leq  \left(\int_{\mathbb{R}^{d}}  (f(x) - f^{\star}) ^{2} \mu_s(x) \dd x \right)^{\frac{1}{2}} \left( \int_{\mathbb{R}^{d}}  \left( \rho_{s}(t, x) - \mu_s(x) \right)^{2} \mu_s^{-1} \dd x \right)^{\frac{1}{2}} \\
                                                                    &  \leq C(s) \e^{-\lambda_s  t} \left\|\rho- \mu_{s} \right\|_{\mu_s^{-1}},
\end{align*}
where the first inequality applies the Cauchy-Schwarz inequality and
\[
C(s) = \left(\int_{\mathbb{R}^{d}}  (f - f^{\star}) ^{2} \mu_s \dd x \right)^{\frac{1}{2}}
\]
is an increasing function of $s$.

\end{proof}

We conclude this subsection with the proof of \Cref{thm: rate-l2-1}.
\begin{proof}[Proof of \Cref{thm: rate-l2-1}]
It follows from~\eqref{eqn: converge-FK} that
\[
\frac{\dd}{\dd t} \left\| \rho_s(t, \cdot) - \mu_s \right\|_{\mu_s^{-1} } ^{2} = - s \int_{\mathbb{R}^{d}} \| \nabla h_{s} \|^{2}  \dd \mu_{s}.
\]
Next, using Lemma~\ref{thm: villani-poincare-inq} and recognizing the equality $\int_{\mathbb{R}^{d}} h_s\dd \mu_s = \int_{\mathbb{R}^{d}} \rho_s(t, x) \dd x = 1$, we get
\begin{align}
\frac{\dd}{\dd t} \left\| \rho_s(t, \cdot) - \mu_s \right\|_{\mu_s^{-1} } ^{2}  & \leq - 2 \lambda_{s} \left( \int_{\mathbb{R}^{d}} h_s^{2} \dd \mu_{s} - \left( \int_{\mathbb{R}^{d}} h_s \dd \mu_{s} \right)^{2} \right) \nonumber \\
& = - 2 \lambda_{s} \left( \int_{\mathbb{R}^{d}} h_s^{2} \dd \mu_{s} - 1 \right) \nonumber \\
& =  - 2 \lambda_{s} \int_{\mathbb{R}^{d}} (h_s - 1)^{2}  \dd \mu_{s} \nonumber \\
& = - 2\lambda_{s}\left\|\rho_s(t,\cdot) - \mu_s \right\|_{\mu_{s}^{-1}} ^{2}. \nonumber
\end{align}
Integrating both sides yields \eqref{eqn: rate-l2-1}, as desired.

\end{proof}




\subsection{Proof of \Cref{thm: compare_step}}
\label{sec:proof-crefthm:-comp}

Next, we turn to the proof of \Cref{thm: compare_step}. We first state a technical lemma, deferring its proof to \Cref{subsec: proof-lem-deriv-ep=0}.   
\begin{lem}\label{lem: deriv_epsilon=0}
Under the assumptions of \Cref{thm: compare_step}, the excess risk at stationarity $\epsilon(s)$ satisfies
 \[
 \frac{\dd\epsilon(0)}{\dd s} = 0.
 \]
 \end{lem}

Using~\Cref{lem: deriv_epsilon=0}, we now finish the proof of \Cref{thm: compare_step}. 
\begin{proof}[Proof of \Cref{thm: compare_step}]

Letting $g = f - f^{\star}$, we write the excess risk at stationarity as
\begin{equation}\nonumber
\epsilon(s) = \E f(X_s(\infty)) - f^{\star} = \frac{\int_{\mathbb{R}^{d}} g \e^{- \frac{2g}{s}} \dd x }{\int_{\mathbb{R}^{d}}  \e^{- \frac{2g}{s}} \dd x},
\end{equation}
which yields the following derivative:
\begin{align}\nonumber
\frac{\dd \epsilon(s)}{\dd s} & = \frac{ \frac{2}{s^{2}} \int_{\mathbb{R}^{d}} g^{2} \e^{-\frac{2g}{s}} \dd x  \int_{\mathbb{R}^{d}}  \e^{-\frac{2g}{s}} \dd x  - \frac{2}{s^{2}} \left( \int_{\mathbb{R}^{d}} g \e^{-\frac{2g}{s}} \dd x \right)^{2}  }{\left( \int_{\mathbb{R}^{d}} \e^{-\frac{2g}{s}} \dd x\right)^{2}}.
\end{align}
Making use of the Cauchy-Schwarz inequality, the derivative satisfies $\frac{\dd \epsilon(s)}{\dd s} \geq 0$ for all $s > 0$. In fact, the equality holds only in the case of a constant $f$ is a constant, which contradicts both the confining condition and the Villani condition. Hence, the inequality can be strengthened to
\[
\frac{\dd \epsilon(s)}{\dd s} > 0,
\]
for $s > 0$. Consequently, we have proven that the excess risk $\epsilon(s)$ at stationarity is a strictly increasing function of $s \in [0, +\infty)$.

Next, from Fatou's lemma we get
\begin{align*}
& \epsilon(0) \leq \limsup_{s \rightarrow 0^+} \epsilon(s) \leq  \int_{\mathbb{R}^{d} } \lim_{s \rightarrow 0^+} g \mu_{s} \dd x = f^\star - f^\star = 0 \\
& \epsilon(0) \geq \liminf_{s \rightarrow 0^+} \epsilon(s) \geq  \int_{\mathbb{R}^{d} } \lim_{s \rightarrow 0^+} g \mu_{s} \dd x = f^\star - f^\star = 0.
\end{align*}
As a consequence, $\epsilon(0) = 0$. \Cref{lem: deriv_epsilon=0} shows that for any $S > 0$, there exists $A = A_{S}$ such that $ 0 \le \frac{\dd \epsilon(s)}{\dd s} \le A$ for all $0\le s \le S$. This fact, combined with $\epsilon(0) = 0$, immediately gives $\epsilon(s) \le As$ for all $0 \le s \le S$.



\end{proof}


\section{Geometrizing the Exponential Decay Constant}
\label{sec: estimate-rate}

Having established the linear convergence to stationarity for the lr-dependent SDE, we now offer a quantitative characterization of the exponential decay constant $\lambda_s$ for a class of nonconvex objective functions. This is crucial for us to obtain a clear understanding of the dynamics of SGD and especially its dependence on the learning rate in the nonconvex setting.


\subsection{Connection with a Schr\"odinger operator}
\label{subsubsec: schrodinger}


We begin by deriving a relationship between the lr-dependent SDE~\eqref{eqn: sgd_high_resolution_formally} and a Schr\"odinger operator. Recall that the probability density $\rho_s(t, \cdot)$ of the SDE solution is assumed to be in $L^{2}(\mu_s^{-1})$. Consider the transformation
\begin{equation}\nonumber
\psi_s(t, \cdot) =  \frac{\rho_s(t, \cdot)}{\sqrt{\mu_s}}  \in L^{2}(\mathbb{R}^{d}).
\end{equation}
This transformation allows us to equivalently write the Fokker--Planck--Smoluchowski equation~\eqref{eqn: Fokker-Planck} as
\begin{equation}\label{eq:schrodinger}
\frac{\partial \psi_s}{\partial t} = \frac{s}{2} \Delta \psi_s - \left( \frac{\| \nabla f\|^{2}}{2s} - \frac{\Delta f}{2}\right) \psi_s = - \frac{-s \Delta  + V_s}{2}\psi_s,
\end{equation}
with the initial condition $\psi_s(0, \cdot) = \frac{\rho}{\sqrt{\mu_s}} \in  L^{2}(\mathbb{R}^{d})$. This is a Schr\"odinger equation with the associated operator $-s \Delta  + V_s$, where the potential
\[
V_s = \frac{\| \nabla f\|^{2}}{s} - \Delta f
\]
is positive for sufficiently large $\|x\|$ due to the Villani condition.


Now, we collect some basic facts concerning the spectrum of the Schr\"odinger operator $-s \Delta  + V_s$. First, it is a positive semidefinite operator, as shown below. Recognizing the uniqueness of the Gibbs distribution~\eqref{eqn: Gibbs}, it is not hard to show that $\sqrt{\mu_s}$ is the unique eigenfunction of $-s \Delta  + V_s$ with a corresponding eigenvalue of zero. Using this fact, from the proof of \Cref{thm: rate-l2-1}, we get
\begin{align*}
\left\langle (-s \Delta  + V_s) \psi_s(t, \cdot), \psi_s(t, \cdot) \right\rangle & = \left\langle (-s \Delta  + V_s) (\psi_s(t, \cdot) - \sqrt{\mu_s}), \psi_s(t, \cdot) - \sqrt{\mu_s} \right\rangle \\
& = -  \frac{\dd}{\dd t} \left\langle \psi_{s}(t, \cdot)- \sqrt{\mu_{s}}, \psi_{s}(t, \cdot)- \sqrt{\mu_{s}} \right\rangle  \\
                                                                                                                       & = -  \frac{\dd}{\dd t} \left\| \rho_{s}(t, \cdot)- \mu_{s} \right\|_{\mu_s^{-1} } ^{2}  \\
    & = s \int_{\R^d} \|\nabla (\rho_s(t, \cdot)\mu_s^{-1})\|^2 \dd \mu_s\\
    &\geq 0,
\end{align*}
where $\langle \cdot, \cdot \rangle$ denotes the standard inner product in $L^2(\R^d)$. 
In fact, this inequality can be extended to $\left\langle (-s \Delta  + V_s) g, g \right\rangle \ge 0$ for any $g$. This verifies the positive semidefiniteness of the Schr\"odinger operator $-s \Delta  + V_s$. 

Next, making use of the fact that $\frac1s V_s(x) \goto +\infty$ as $\|x\| \goto +\infty$, we state the following well-known result in spectral theory---that the Schr\"odinger operator has a purely discrete spectrum in $L^{2}(\mathbb{R}^{d})$~\cite{hislop2012introduction}.
\begin{lem}[Theorem 10.7 in~\cite{hislop2012introduction}]\label{thm: discrete-spectrum-schrodinger}
Assume that $V$ is continuous, and $V(x) \rightarrow +\infty$ as $\|x\| \rightarrow +\infty$. Then the operator $-\Delta + V$ has a purely discrete spectrum.
\end{lem}

Taken together, the positive semidefiniteness of $-s \Delta  + V_s$ and \Cref{thm: discrete-spectrum-schrodinger} allow us to order the eigenvalues of $-s \Delta  + V_s$ in $L^2(\mathbb{R}^d)$ as
\[
0 = \zeta _{s,0} < \zeta _{s,1} \leq \cdots \leq \zeta _{s,\ell} \leq \cdots < +\infty.
\]
A crucial fact from this representation is that the exponential decay constant $\lambda_{s}$ in Theorem~\ref{thm: rate-l2-1} can be set to
\begin{equation}\label{eq:schrodinger_ex_decay}
\lambda_{s} = \frac12 \zeta _{s,1}.
\end{equation}
To see this, note that $\psi_s(t, \cdot) - \sqrt{\mu_s}$ also satisfies \eqref{eq:schrodinger} and is orthogonal to the null eigenfunction $\sqrt{\mu_s}$. Therefore, the norm of $\psi_s(t, \cdot) - \sqrt{\mu_s}$ must decay exponentially at a rate determined by half of the smallest positive eigenvalue of $H_s$.\footnote{Here, the norm of $\psi_{s}(t, \cdot) - \sqrt{\mu_s}$ is induced by the inner product in $L^{2}(\mathbb{R}^d)$. That is,
\[
\|\psi(t, \cdot) - \sqrt{\mu_s}\|_{L^2(\mathbb{R}^d)} = \sqrt{\left\langle \psi(t, \cdot) - \sqrt{\mu_s}, \psi(t, \cdot) - \sqrt{\mu_s} \right\rangle}.
\]
} That is, we have
\[
\begin{aligned}
\left\langle \psi_s(t, \cdot) - \sqrt{\mu_s}, \psi_s(t, \cdot) - \sqrt{\mu_s} \right\rangle &\le \e^{-2\frac{\zeta _{s,1}}{2} t} \left\langle \psi_s(0, \cdot) - \sqrt{\mu_s}, \psi_s(0, \cdot) - \sqrt{\mu_s} \right\rangle\\
&= \e^{-\zeta_{s,1} t} \left\langle \psi_s(0, \cdot) - \sqrt{\mu_s}, \psi_s(0, \cdot) - \sqrt{\mu_s} \right\rangle,
\end{aligned}
\]
which is equivalent to
\[
\|\rho_s(t, \cdot) - \mu_s\|_{\mu_s^{-1}} \le \e^{-\frac{\zeta _{s,1}}{2} t} \|\rho - \mu_s\|_{\mu_s^{-1}}.
\]
As such, we can take $\lambda_{s} = \frac12 \zeta_{s,1}$ in the proof of \Cref{thm: rate-l2-1}.



As a consequence of this discussion, we seek to study the Fokker--Planck--Smoluchowski equation~\eqref{eqn: Fokker-Planck} by analyzing the spectrum of the linear Schr\"odinger operator~\eqref{eq:schrodinger}, especially its smallest positive eigenvalue $\delta_{s,1}$. To facilitate the analysis, a crucial observation is that this Schr\"odinger operator is equivalent to the \emph{Witten-Laplacian},
\begin{equation}\label{eqn: witten-laplacian}
\Delta_{f}^s := s (-s \Delta  + V_s)  = - s^{2} \Delta + \|\nabla f\|^{2} - s \Delta f,
\end{equation}
by a simple scaling. Denoting by the eigenvalues of the Witten-Laplacian as $0 = \delta_{s,0} < \delta_{s, 1} \leq \cdots \le \delta_{s, \ell} \le \cdots < +\infty$, we obtain the simple relationship
\[
\delta_{s, \ell} = s\zeta _{s,\ell},
\]
for all $\ell$.

The spectrum of the Witten-Laplacian has been the subject of a large literature~\cite{helffer2005hypoelliptic,bovier2005metastability,nier2004quantitative,arnold1999topological}, and in the next subsection, we exploit this literature to derive a closed-from expression for the first positive eigenvalue of the Witten-Laplacian, thereby obtaining the dependence of the exponential decay constant on the learning rate for a certain class of nonconvex objective functions~\cite{herau2011tunnel,michel2019small}.


\subsection{The spectrum of the Witten-Laplacian: nonconvex Morse functions}
\label{sec:basics-morse-theory}

We proceed by imposing the mild condition on the objective function that its first-order and second-order derivatives cannot be both degenerate anywhere. Put differently, the objective function is a Morse function.  This allows us to use the theory of Morse functions to provide a geometric interpretation of the spectrum of the Witten-Laplacian. 


\paragraph{Basics of Morse theory.} 
We give a brief introduction to Morse theory at the minimum level that is necessary for our analysis. Let $f$ be an infinitely differentiable function defined on $\R^n$. A point $x$ is called a critical point if the gradient $\nabla f(x) = 0$. A function $f$ is said to be a Morse function if for any critical point $x$, the Hessian $\nabla^{2} f(x)$ at $x$ is nondegenerate; that is, all the eigenvalues of the Hessian are nonzero. The objective $f$ is assumed to be a Morse function throughout \Cref{sec:basics-morse-theory}. Note also that we refer to a point $x$ as a local minimum if $x$ is a critical point and all eigenvalues of the Hessian at $x$ are positive.


Next, we define a certain type of saddle point. To this end, let $\eta_1(x) \ge \eta_2(x) \ge \cdots \ge \eta_d(x)$ be the eigenvalues of the Hessian $\nabla^{2} f(x)$ at $x$.\footnote{Note that here we order the eigenvalues from the largest to the smallest, as opposed to the case of the Schr\"odinger operator previously.} A critical point $x$ is said to be an \emph{index-1 saddle point} if the Hessian at $x$ has exactly one negative eigenvalue, that is, $\eta_1(x) \geq \cdots \geq \eta_{d-1}(x) > 0,\; \eta_d(x) < 0$. Of particular importance to this paper is a special kind of index-$1$ saddle point that will be used to characterize the exponential decay constant. Letting $\mathcal{K}_{\nu} := \big\{x \in \mathbb{R}^{d}: f(x) < \nu \big\}$ denote the sublevel set at level $\nu$, for any index-1 saddle point $x$, it is not hard to show that the set $\mathcal{K}_{f(x)} \cap \{x': \|x' - x\| < r\}$ can be partitioned into two connected components, say $C_{1}(x, r)$ and $C_{2}(x, r)$, if the radius $r$ is sufficiently small. Using this fact, we give the following definition.

\begin{defn}\label{defn: separating-saddle}
Let $x$ be an index-$1$ saddle point and $r > 0$ be sufficiently small. If $C_1(x, r)$ and $C_{2}(x, r)$ are contained in two different (maximal) connected components of the sublevel set $\mathcal{K}_{f(x)}$, we call $x$ an index-$1$ \textit{separating} saddle point.

\end{defn}

The remainder of this section aims to relate index-1 separating saddle points to the convergence rate of the lr-dependent SDE. For ease of reading, the remainder of the paper uses $x^\circ$ to denote an index-1 separating saddle point and writes $\mathcal{X}^\circ$ for the set of all these points. To give a geometric interpretation of \Cref{defn: separating-saddle}, let $x_1^\bullet$ and $x_2^\bullet$ denote local minima in the two maximal connected components of $\mathcal{K}_{f(x^\circ)}$, respectively. Intuitively speaking, the index-1 separating saddle point $x^\circ$ is the bottleneck of any path connecting the two local minima. More precisely, along a path connecting $x_1^\bullet$ and $x_2^\bullet$, by definition the function $f$ must attain a value that is at least as large as $f(x^\circ)$. In this regard, the function value at $x^\circ$ plays a fundamental role in determining how long it takes for the lr-dependent SDE initialized at $x_1^\bullet$ to arrive at $x_2^\bullet$. See an illustration in \Cref{fig: separating-saddle-1}.


\begin{figure}[htb!]
\centering
\begin{minipage}[t]{0.45\linewidth}
\centering
\includegraphics[scale=0.15]{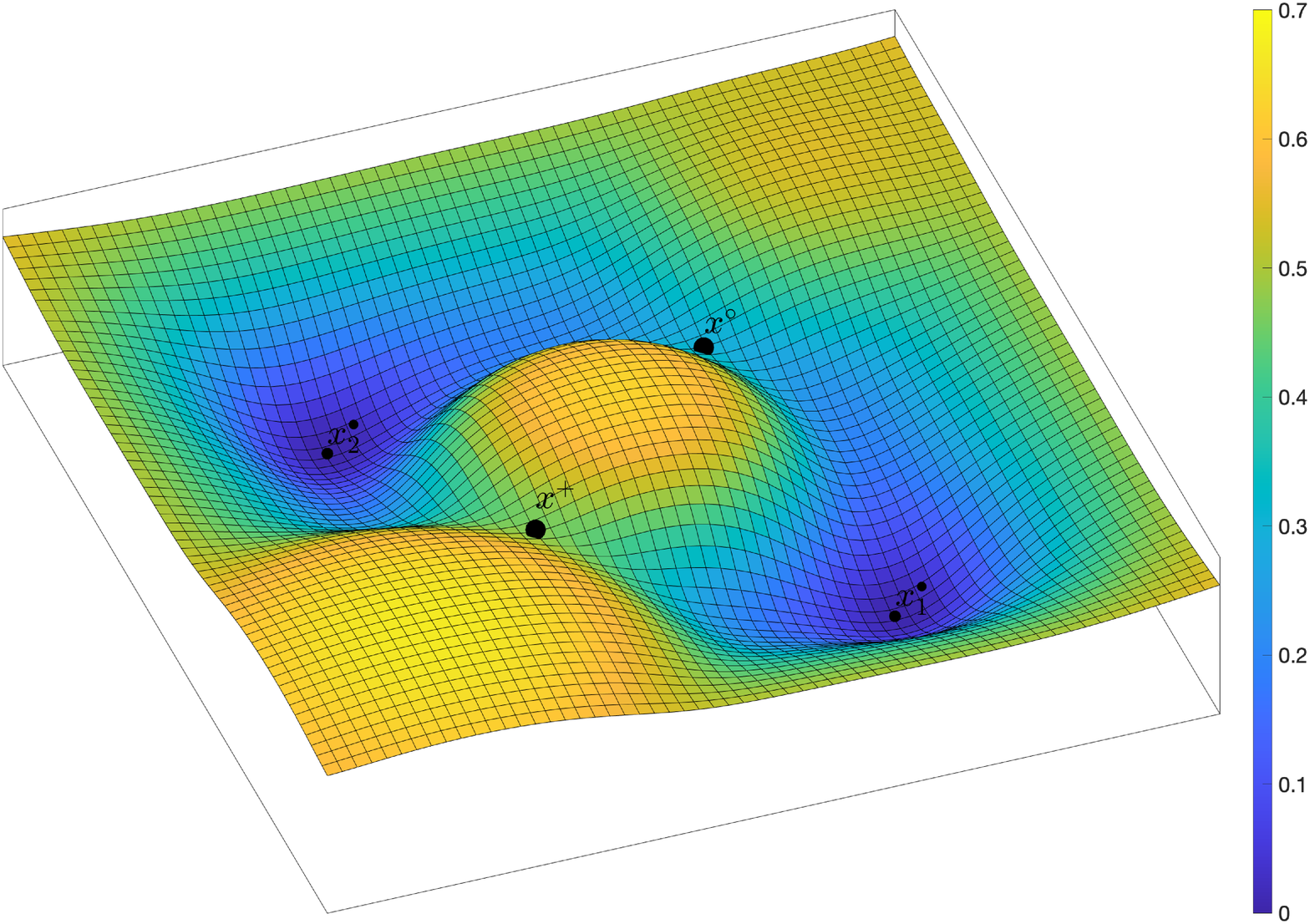}
\end{minipage}
\begin{minipage}[t]{0.45\linewidth}
\centering
\includegraphics[scale=0.15]{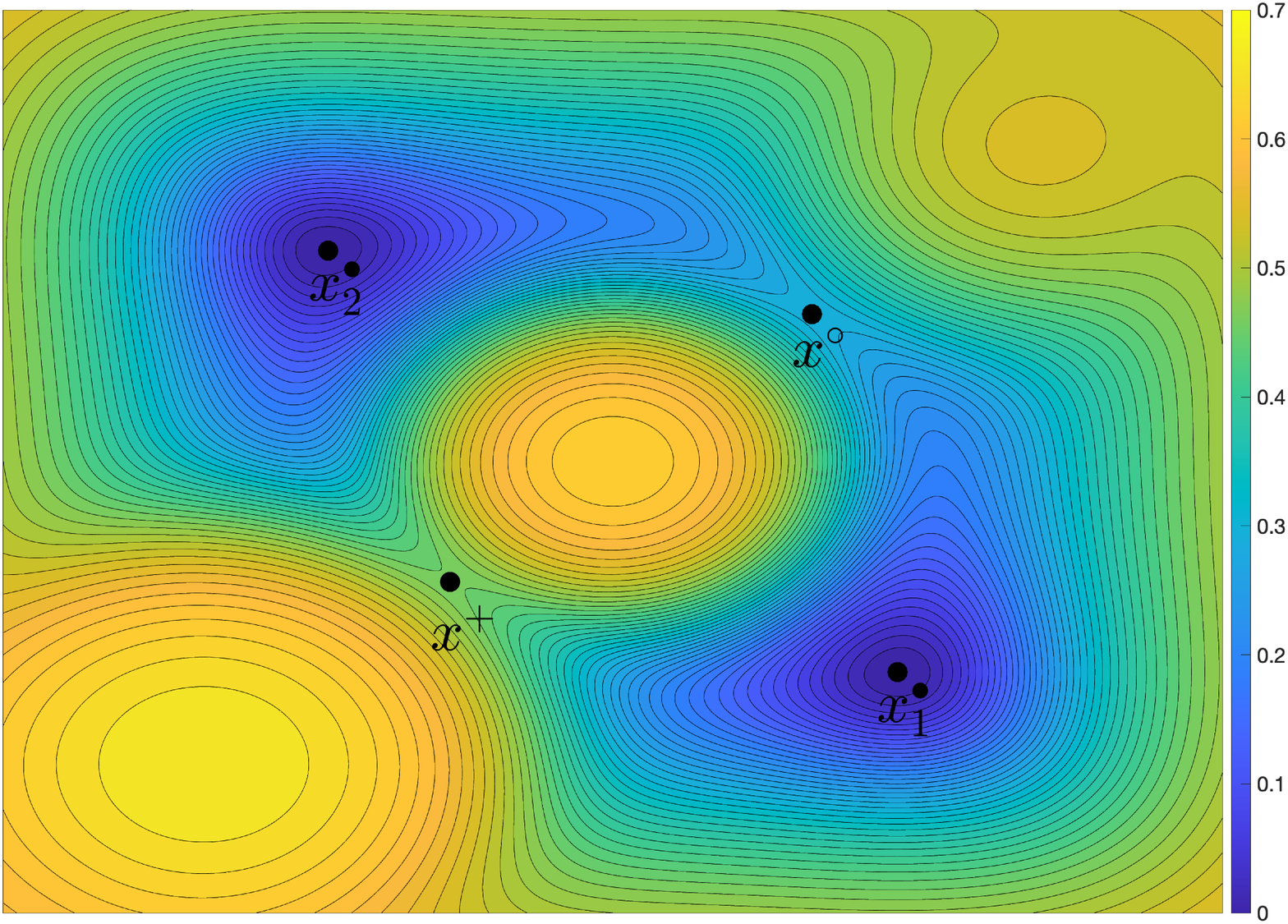}
\end{minipage}\\
\begin{minipage}[t]{0.45\linewidth}
\centering
\includegraphics[scale=0.15]{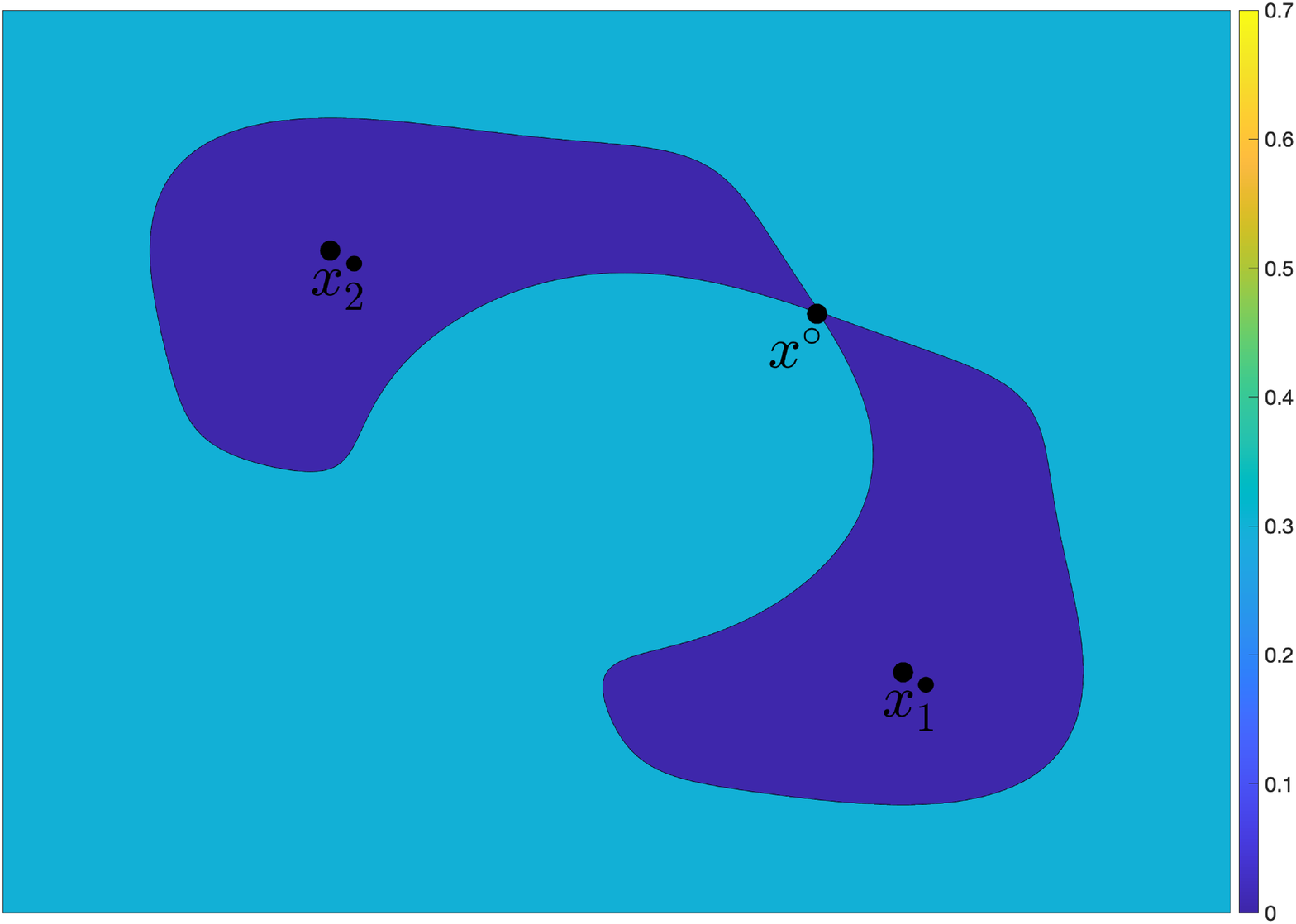}
\end{minipage}
\begin{minipage}[t]{0.45\linewidth}
\centering
\includegraphics[scale=0.15]{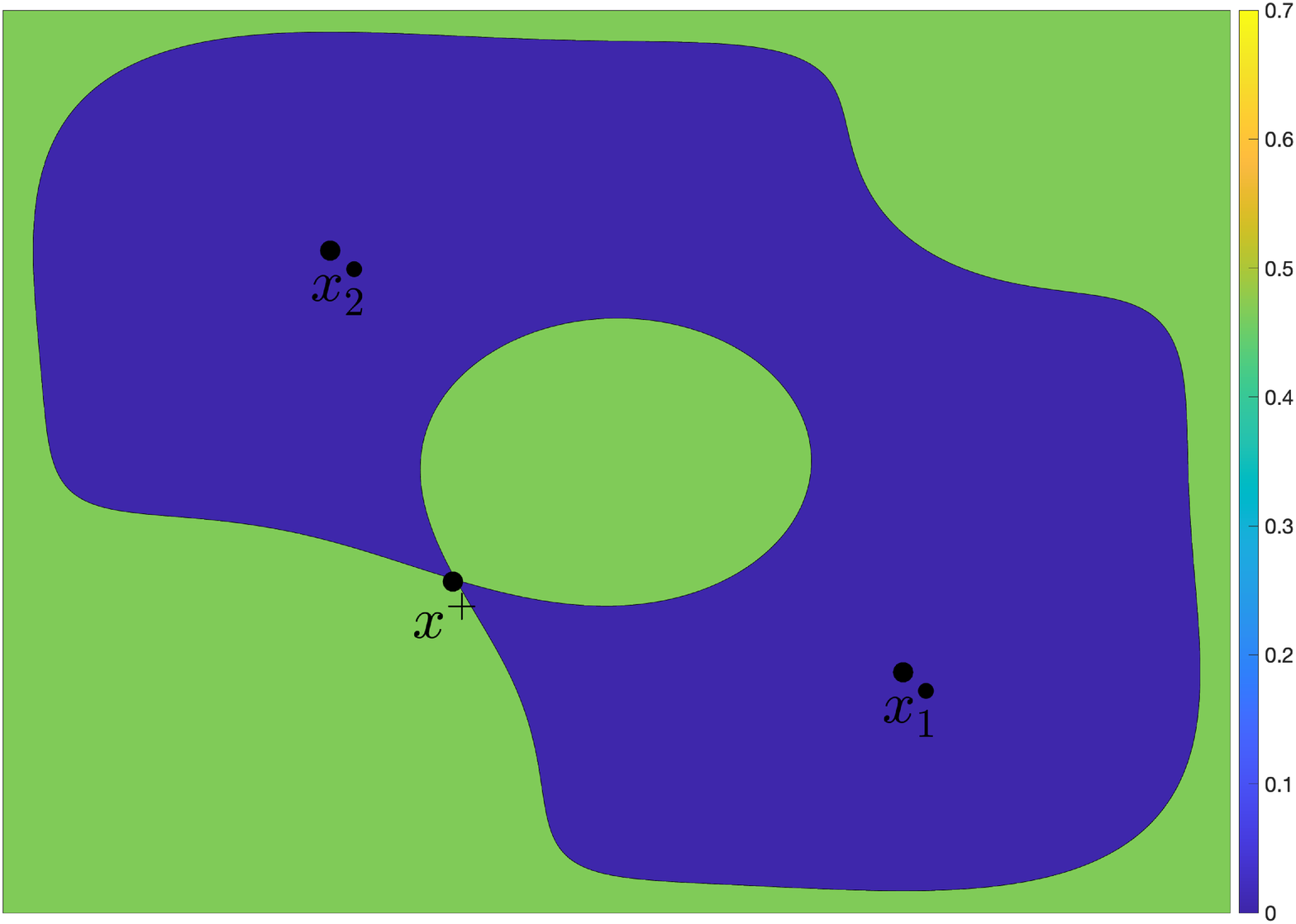}
\end{minipage}
\caption{\small The landscape of a two-dimensional nonconvex Morse function. Here, $x_1^\bullet$ and $x_2^\bullet$ denote two local minima. Both $x^\circ$ and $x^+$ are index-1 saddle points, but only the former is an index-1 separating saddle point since $f(x^\circ) < f(x^\bullet)$. In the two bottom plots, the deep blue regions form the sublevel sets at $f(x^\circ)$ or $f(x^\bullet)$. Note that the sublevel set induced by $x^\circ$ is the union of two connected components.} 
\label{fig: separating-saddle-1}
\end{figure}


As is assumed in this section, $f$ is a Morse function and satisfies both the confining and the Villani conditions; in this case, it can be shown that the number of the critical points of $f$ is finite. Thus, denote by $n^\circ$ the number of index-$1$ separating saddle points of $f$ and let $n^\bullet$ denote the number of local minima.

\paragraph{H\'{e}rau--Hitrik--Sj\"ostrand's generic case.} To describe the labeling procedure, consider the set of the objective values at index-$1$ separating saddle points $\mathcal{V} = \{ f(x^{\circ}): x^{\circ} \in \mathcal{X}^{\circ} \}$. This is a finite set and we use $I$ to denote the cardinality of this set. Write $\mathcal{V} = \{\nu_1, \ldots, \nu_I\}$ and sort these values as
\begin{equation}\label{eqn: sort-decrease-mono-rigorous}
+\infty = \nu_0 > \nu_1 > \cdots > \nu_I,
\end{equation}
where by convention $\nu_{0} = +\infty$ corresponds to a fictive saddle point at infinity.

Next, we follow \cite{herau2011tunnel} and define a type of connected components of sublevel set.  
\begin{defn}\label{defn: connected-component}
A connected component $E$ of the sublevel set $\mathcal{K}_{\nu}$ for some $\nu \in \mathcal{V}$ is called a \textit{critical component} if either  $\partial E \cap \mathcal{X}^\circ \neq \varnothing$ or $E = \mathbb{R}^{d}$, where $\partial E$ is the boundary of $E$. 

\end{defn}

In this definition, the case of $E = \R^d$ applies only if $\nu = \nu_0 = +\infty$. If $\nu = \nu_i$ for some $1 \le i \le I$ is only attained by one index-1 separating saddle point, the sublevel set $\mathcal{K}_{\nu_{i}}$ has two critical components. See \Cref{defn: separating-saddle} for more details.

With the preparatory notions above in place, we describe the following procedure for labeling index-1 separating saddle points and local minima~\cite{herau2011tunnel}. See \Cref{fig:generic} for an illustration of this process.
\begin{itemize}
\item[1.] Let $E^{0}_{1} := \R^d$. Note that the global minimum $x^{\star}$ is contained in $E^{0}_{1}$ and denote
\[
x_0^\bullet := x^{\star} = \underset{x \in E^{0}_{1}}{\mathrm{argmin}}\, f(x).
\] 
Let $\mathcal{X}^{\bullet}_{0}$ denote the singleton set $\{x^{\star}\}$.

\item[2.] Let $E^{1}_{j}$ for $j =1, \ldots, m_1$ be the critical components of the sublevel set $\mathcal{K}_{\nu_1}$. Note that $E^{1}_{1} \cup \cdots \cup E^{1}_{m_1}$ is a (proper) subset of $\mathcal{K}_{\nu_{1}}$. Without loss of generality, assume $x^\star \in E^{1}_{m_1}$. Then, we select $x_{1,j_1}^\bullet$ as
             \[
             x_{1, j_1}^\bullet= \underset{x \in E^{1}_{j_1} }{\mathrm{argmin}} \, f(x).
             \]
Define $\mathcal{X}^{\bullet}_{1} := \{ x_{1,1}^\bullet, \ldots, x_{1, m_1-1}^\bullet \}$.             
\item[3.]  For $i = 2, \ldots, I$, let $E^{i}_j$ for $j =1, \ldots, m_i$ be the critical components of the sublevel set $\mathcal{K}_{\nu_{i}}$. Without loss of generality, we assume that the critical components are ordered such that there exists an integer $k_i \leq m_i$ satisfying 
\[
\left(\bigcup_{j=1}^{k_{i}} E_{j}^{i}\right) \bigcap \left(\bigcup_{\ell=0}^{i-1} \mathcal{X}^{\bullet}_{\ell} \right)=\varnothing 
\]
and
\[
E_{j}^{i} \bigcap \left(\bigcup_{\ell=0}^{i-1} \mathcal{X}_{\ell}^{\bullet} \right)\neq \varnothing,
\]
for any $j=k_{i}+1,  \ldots, m_i$. Set $x^\bullet_{i,j}$ to
\[
x_{i,j}^\bullet = \underset{x \in E^{i}_{j}}{\mathrm{argmin}} \, f(x),
\]
for $j =  1, \ldots, k_i$. Define $\mathcal{X}_{i}^{\bullet} :=\{ x_{i,1}^\bullet, \ldots, x_{i, k_i}^\bullet \} $.             


\end{itemize}

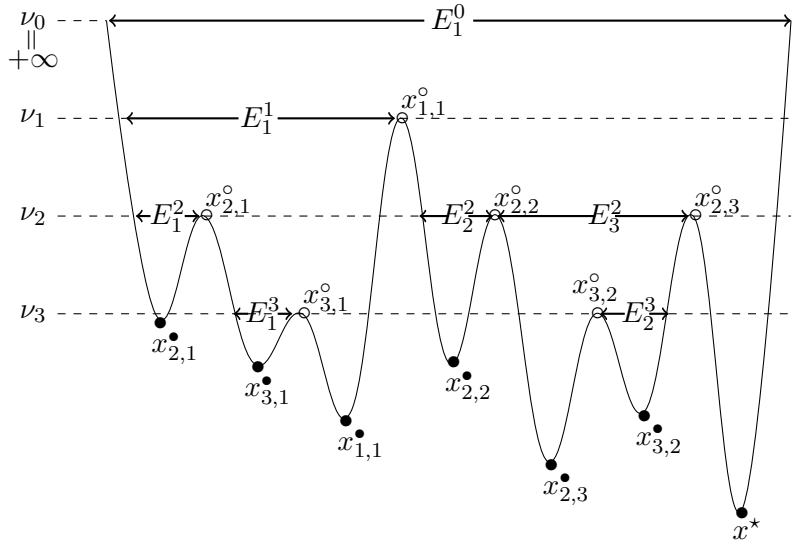
\begin{figure}[htb!]
\centering
\begin{tikzpicture}[scale=0.65]
     \draw  plot[smooth, tension=.7] coordinates{(0, 10) (1, 4) (2, 6) (3,3) (4,4) (5,2) (6, 8) (7, 3) (8, 6) (9, 1) (10,4)(11, 2)(12,6) (13,0)(14,10)};
     \draw[dashed] (-1,10) -- (0,10);
     \draw[dashed] (-1,  8) -- (0.4,  8);
     \draw[dashed] (6,  8) -- (14,  8);
     \draw[dashed] (-1,  6) -- (0.6,  6);
     \draw[dashed] (2.0,  6) -- (6.5,  6);
     \draw[dashed] (12,  6) -- (14,  6);
     \draw[dashed] (-1,  4) -- (10,  4);
     \draw[dashed] (11.5,  4) -- (14,  4);
     \node at (-1.5,10) {$\nu_0$};
     \node at (-1.5,9.6) {$\rotatebox{90}{=}$};
     \node at (-1.5,9.2) {$+\infty $};
     \node at (-1.5,  8) {$\nu_1$};
     \node at (-1.5,  6) {$\nu_2$};    
     \node at (-1.5,  4) {$\nu_3$};
     \node at (13, -0.1) {$\bullet$};
     \node at (13.1, -0.4) {$x^{\star}$};
     \node at (4.9, 1.8) {$\bullet$};
     \node at (5.2, 1.3) {$x_{1,1}^{\bullet}$};
     \node at (6.05, 8) {$\circ$};
     \node at (6.5, 8.3) {$x_{1,1}^{\circ}$};     
     \node at (1.1, 3.8) {$\bullet$};
     \node at (1.4, 3.3) {$x_{2,1}^{\bullet}$}; 
     \node at (2.05, 6) {$\circ$};
     \node at (2.5, 6.3) {$x_{2,1}^{\circ}$}; 
     \node at (7.1, 3.0) {$\bullet$};
     \node at (7.4, 2.5) {$x_{2,2}^{\bullet}$}; 
     \node at (7.95, 6) {$\circ$};
     \node at (8.4, 6.3) {$x_{2,2}^{\circ}$};      
     \node at (9.1, 0.9) {$\bullet$};      
     \node at (9.4, 0.4) {$x_{2,3}^{\bullet}$}; 
     \node at (12.05, 6) {$\circ$};
     \node at (12.5, 6.3) {$x_{2,3}^{\circ}$}; 
     \node at (3.1, 2.9) {$\bullet$};     
     \node at (3.3, 2.4) {$x_{3,1}^{\bullet}$};
     \node at (4.05, 4) {$\circ$};
     \node at (4.5, 4.3) {$x_{3,1}^{\circ}$}; 
     \node at (11.0, 1.9) {$\bullet$};     
     \node at (11.3, 1.4) {$x_{3,2}^{\bullet}$};  
     \node at (10.05, 4) {$\circ$};
     \node at (10.0, 4.5) {$x_{3,2}^{\circ}$}; 


    \node at (7, 10) {$E_{1}^{0}$};
    \draw[->, thick] (7.5, 10) -- (14, 10);
    \draw[->, thick] (6.5, 10) -- (0.05,  10);    
    \node at (3.1,8) {$E_{1}^{1}$};
    \draw[->, thick] (3.5, 8) -- (5.9, 8);
    \draw[->, thick] (2.7, 8) -- (0.4, 8);    
    \node at (1.25,6) {$E_{1}^{2}$};
    \draw[->, thick] (1.6, 6) -- (1.9, 6);
    \draw[->, thick] (0.9, 6) -- (0.6, 6);    
    \node at (7.2,6) {$E_{2}^{2}$};
    \draw[->, thick] (7.4, 6) -- (7.9, 6);
    \draw[->, thick] (7.0, 6) -- (6.4, 6);  
    \node at (10.2,6) {$E_{ 3}^{2}$};
    \draw[->, thick] (10.4, 6) -- (11.9, 6);
    \draw[->, thick] (10.0, 6) -- (8.0, 6);  
    \node at (3.2,4) {$E_{1}^{3}$};
   \draw[->, thick] (3.5, 4) -- (3.8, 4);
   \draw[->, thick] (2.9, 4) -- (2.6, 4);      
    \node at (10.9,4) {$E_{ 2}^{3}$};
    \draw[->, thick] (11.1, 4) -- (11.5, 4);
    \draw[->, thick] (10.6, 4) -- (10.1, 4);    
\end{tikzpicture}
\caption{\small A generic one-dimensional Morse function. The labeling process gives rise to a one-to-one correspondence between the local minimum $x^\bullet_{ij}$ and the index-1 separating saddle point $x^\circ_{i,j}$ (which are also local maxima) for all $i,j$.}
\label{fig:generic}
\end{figure}


To make the labeling process above valid, however, we need to impose the following assumption on the objective. This assumption is generic in the sense that it should be satisfied by a \textit{generic} Morse function.
\begin{assumption}[Generic case \cite{herau2011tunnel}]\label{assump: generic-assumption}
For every critical component $E^{i}_j$ selected in the labeling process above, where $i = 0, 1, \ldots, I$,  we assume that
\begin{itemize}
\item The minimum $x^\bullet_{i,j}$ of $f$ in any critical component $E^{i}_j$ is unique.
\item If $E_{j}^{i} \cap \mathcal{X}^\circ \neq \varnothing$, there exists a unique $x_{i,j}^{\circ} \in E_{j}^{i} \cap \mathcal{X}^\circ$ such that $f(x_{i,j}^{\circ}) = \max\limits_{x \in E_{j}^{i} \cap \mathcal{X}^\circ} f(x)$. In particular, $E_j^i \cap \mathcal{K}_{f(x_{i,j}^{\circ})}$ is the union of two distinct critical components. 

\end{itemize}
\end{assumption}


The first condition in this assumption requires that there exists a unique minimum of the objective $f$ in every critical component $E^{i}_{j}$. In particular, the global minimum $x^\star$ is unique under this assumption. In addition, the second condition requires that among all index-$1$ separating saddle points in $E^i_j$, if any, $f$ attains the maximum at exactly one of these points.


Under~\Cref{assump: generic-assumption}, the above labeling process includes all the local minima of $f$. Moreover, it reveals a remarkable result: there exists a bijection between the set of local minima and the set of index-$1$ separating saddle points (including the fictive one) $\mathcal{X}^{\circ} \cup \{\infty\}$. As shown in the labeling process, for any local minimum $x_{i,j}^\bullet$, we can relate it to the index-1 separating saddle point at which $f$ attains the maximum in the critical component $E_j^i$. See \Cref{fig:generic} for an illustrative example. Interestingly, this shows that the number of local minima is always larger than the number of index-1 separating saddle points by one; that is, $n^\circ = n^\bullet - 1$.

In light of these facts, we can relabel the index-1 separating saddle points $x^\circ_\ell$ for $\ell = 0, 1, \ldots, n^{\circ}$ with $x^\circ_0 = \infty$, and the local minima $x^\bullet_\ell$ for $\ell = 0, 1, \ldots, n^{\bullet}-1$ with $x^\bullet_0 = x^\star$, such that
\begin{equation}
\label{eqn: saddle-minima-pair}
f(x^\circ_0) - f(x^\bullet_0) > f(x^\circ_1) - f(x^\bullet_1) \ge \ldots \ge f(x^\circ_{n^{\bullet}-1}) - f(x^\bullet_{n^{\bullet}-1}),
\end{equation}
where $f(x^\circ_0) - f(x^\bullet_0) = f(\infty) - f(x^\star) = +\infty$. A detailed description of this bijection is given in~\cite[Proposition 5.2]{herau2011tunnel}. 

With the pairs $(x^\circ_{\ell}, x^\bullet_{\ell})$ in place, we readily state the following fundamental result concerning the first $n^{\bullet}-1$ smallest positive eigenvalues of the Witten-Laplacian $\Delta_{f}^{s}$ in~\eqref{eqn: witten-laplacian}. Recall that the nonconvex Morse function $f$ satisfies the confining condition and the Villani condition.

\begin{prop}[Theorem 1.2 in~\cite{herau2011tunnel}]\label{thm: HHS-generic}
Under \Cref{assump: generic-assumption} and the assumptions of \Cref{thm: continuous-quantative}, there exists $s_0 > 0$ such that for any $s \in (0, s_0]$, the first $n^{\bullet}-1$ smallest positive eigenvalues of the Witten-Laplacian $\Delta^s_f$ associated with $f$ satisfy
\begin{equation}\nonumber
\delta_{s,\ell} = s\left(\gamma_\ell + o(s) \right)  \e^{- \frac{2 (f(x_{\ell}^\circ) - f(x_\ell^\bullet) )}{s}}
\end{equation}
for $\ell =1, 1, \ldots,  n^{\bullet}-1$, where
\begin{equation}\label{eq:gamma_expre}
\gamma_\ell = \frac{|\eta_{d}(x_{\ell}^\circ)|}{\pi} \left( \frac{\det( \nabla^{2} f(x_\ell^\bullet) )}{- \det (\nabla^{2} f(x_\ell^\circ) )}\right)^{\frac{1}{2}},
\end{equation}
and $\eta_{d}(x_\ell^\circ)$ is the unique negative eigenvalue of $\nabla^{2} f(x_\ell^\circ)$.

\end{prop}

Using \Cref{thm: HHS-generic} in conjunction with the simple relationship between the exponential decay constant and the spectrum of the Schr\"{o}dinger operator/Witten-Laplacian~\eqref{eq:schrodinger_ex_decay}, it is a stone's throw to prove \Cref{thm: continuous-quantative} when $f$ is generic. First, we give the definition of the \emph{Morse saddle barrier}.
\begin{defn}\label{def:barrier} 
Let $f$ satisfy the assumptions of \Cref{thm: continuous-quantative}. We call $H_{f}= f(x_1^\circ) - f(x_1^\bullet)$ the Morse saddle barrier of $f$.
\end{defn}

\begin{proof}[Proof of \Cref{thm: continuous-quantative} in the generic case]
By \Cref{thm: HHS-generic}, we can set the exponential decay constant to
\[
\lambda_s = \frac{1}{2s} \delta_{s,1} = \left(\frac{|\eta_{d}(x_1^\circ)|}{2\pi} \left( \frac{\det( \nabla^{2} f(x_1^\bullet) )}{- \det (\nabla^{2} f(x_1^\circ) )}\right)^{\frac{1}{2}} + o(s) \right) \e^{- \frac{2 H_f}{s}}
\]
in \Cref{thm: continuous-quantative}. Taking $\alpha = \frac12\frac{|\eta_{d}(x_1^\circ)|}{2\pi} \left( \frac{\det( \nabla^{2} f(x_1^\bullet) )}{- \det (\nabla^{2} f(x_1^\circ) )}\right)^{\frac{1}{2}}$ in \eqref{eqn: main-lambda-nonconvex}, we complete the proof when $f$ falls into the generic case.

\end{proof}

However, the generic assumption for the labeling process is complex, leading to the lack of a geometric interpretation of the objective function required for the labeling process. To gain further insight, we present a simplifying assumption that is a special case of \Cref{assump: generic-assumption}. This simplification is due to~\cite{nier2004quantitative}. 


\begin{assumption}[Simplified generic case~\cite{nier2004quantitative}]
\label{assump: simplified}
The objective functions $f$ takes different values at its local minima and index-1 separating saddle points. That is, letting $x_1$ be a local minimum or an index-1 separating saddle point, and $x_2$ likewise, then $f(x_1) \neq f(x_2)$.
Furthermore, the differences $f(x_{\ell_1}^\circ) - f(x_{\ell_2}^\bullet)$ are distinct for any $\ell_1$ and $\ell_2$.

\end{assumption}

The following result follows immediately from \Cref{thm: HHS-generic}.
\begin{coro}[Theorem 3.1 in~\cite{nier2004quantitative}]\label{thm: Nier-generic}
Under \Cref{assump: simplified} and the assumptions of \Cref{thm: continuous-quantative}, \Cref{thm: HHS-generic} holds. Therefore, \Cref{thm: continuous-quantative} holds in this case.

\end{coro}


\paragraph{Michel's degenerate case.} 

We say that a Morse function is \emph{degenerate} if it satisfies the assumptions of \Cref{thm: continuous-quantative} but not \Cref{assump: generic-assumption}. To violate the generic assumption, for example, we can change the objective value $f(x_{3,1}^{\bullet})$ to $f(x_{1,1}^{\bullet})$ or change $f(x_{3,2}^{\bullet})$ to $f(x_{2,3}^{\bullet})$ in~\Cref{fig:generic}. In this situation, the first condition in~\Cref{assump: generic-assumption} is not satisfied. Alternatively, if the objective value at $x_{3,1}^{\circ}$ is changed to $f(x_{2,1}^{\circ})$, the second condition in~\Cref{assump: generic-assumption} is not met. \Cref{fig:degenerate} presents an example of a degenerate Morse function.

The main challenge in the degenerate case is the lack of uniqueness of the pairs $(x^\circ_{\ell}, x^\bullet_{\ell})$ derived from the labeling process. Nevertheless, the uniqueness can be maintained if we work on the function values. Explicitly, the labeling process can be adapted to the degenerate case and still yields unique pairs $(f(x^\circ_{\ell}), f(x^\bullet_{\ell}))$ obeying
\[
f(\infty) - f(x^\star) = f(x^\circ_0) - f(x^\bullet_0) > f(x^\circ_1) - f(x^\bullet_1) \ge \ldots \ge f(x^\circ_{n^{\bullet}-1}) - f(x^\bullet_{n^{\bullet}-1}).
\]
In particular, the number of local minima remains larger than that of index-1 separating saddle points by one in this case. The following result extends \Cref{thm: HHS-generic} to the degenerate case, which is adapted from Theorem 2.8 of~\cite{michel2019small}.

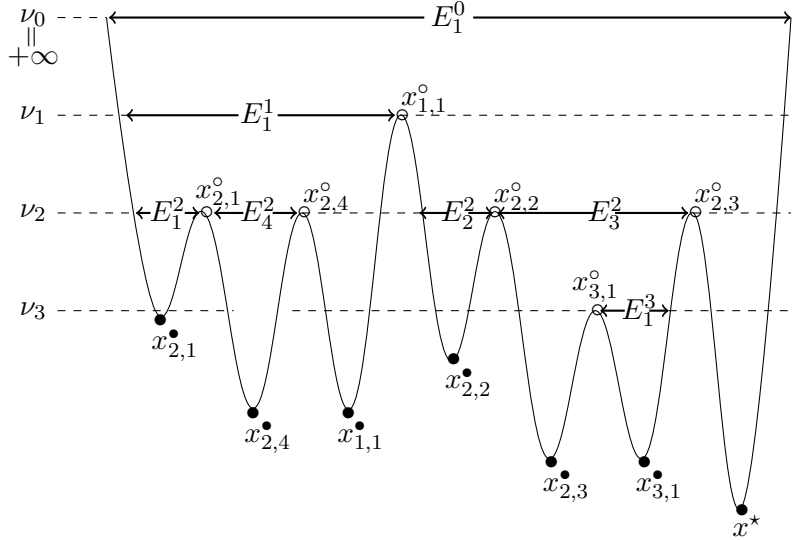
\begin{figure}[b!ht]
\centering
\begin{tikzpicture}[scale=0.65]
     \draw  plot[smooth, tension=.7] coordinates{(0, 10) (1, 4) (2, 6) (3,2) (4,6) (5,2) (6, 8) (7, 3) (8, 6) (9, 1) (10,4)(11, 1)(12,6) (13,0)(14,10)};
     \draw[dashed] (-1,10) -- (0,10);
     \draw[dashed] (-1,  8) -- (0.4,  8);
     \draw[dashed] (6,  8) -- (14,  8);
     \draw[dashed] (-1,  6) -- (0.6,  6);
     \draw[dashed] (4.0,  6) -- (6.5,  6);
     \draw[dashed] (12,  6) -- (14,  6);
     \draw[dashed] (-1,  4) -- (2.6,  4);
     \draw[dashed] (3.8, 4) -- (10,   4);     
     \draw[dashed] (11.5,  4) -- (14,  4);
     \node at (-1.5,10) {$\nu_0$};
     \node at (-1.5,9.6) {$\rotatebox{90}{=}$};
     \node at (-1.5,9.2) {$+\infty $};
     \node at (-1.5,  8) {$\nu_1$};
     \node at (-1.5,  6) {$\nu_2$};    
     \node at (-1.5,  4) {$\nu_3$};
     \node at (13, -0.1) {$\bullet$};
     \node at (13.1, -0.4) {$x^{\star}$};
     \node at (4.95, 1.9) {$\bullet$};
     \node at (5.2, 1.4) {$x_{1,1}^{\bullet}$};
     \node at (6.05, 8) {$\circ$};
     \node at (6.5, 8.3) {$x_{1,1}^{\circ}$};     
     \node at (1.1, 3.8) {$\bullet$};
     \node at (1.4, 3.3) {$x_{2,1}^{\bullet}$}; 
     \node at (2.05, 6) {$\circ$};
     \node at (2.3, 6.4) {$x_{2,1}^{\circ}$}; 
     \node at (7.1, 3.0) {$\bullet$};
     \node at (7.4, 2.5) {$x_{2,2}^{\bullet}$}; 
     \node at (7.95, 6) {$\circ$};
     \node at (8.4, 6.3) {$x_{2,2}^{\circ}$};      
     \node at (9.1, 0.9) {$\bullet$};      
     \node at (9.4, 0.4) {$x_{2,3}^{\bullet}$}; 
     \node at (12.05, 6) {$\circ$};
     \node at (12.5, 6.3) {$x_{2,3}^{\circ}$}; 
     \node at (3.0, 1.9) {$\bullet$};     
     \node at (3.3, 1.4) {$x_{2,4}^{\bullet}$};
     \node at (4.05, 6) {$\circ$};
     \node at (4.5, 6.3) {$x_{2,4}^{\circ}$}; 
     \node at (11.0, 0.9) {$\bullet$};     
     \node at (11.3, 0.4) {$x_{3,1}^{\bullet}$};  
     \node at (10.05, 4) {$\circ$};
     \node at (10.0, 4.5) {$x_{3,1}^{\circ}$}; 


    \node at (7, 10) {$E_{1}^{0}$};
    \draw[->, thick] (7.5, 10) -- (14, 10);
    \draw[->, thick] (6.5, 10) -- (0.05,  10);    
    \node at (3.1,8) {$E_{1}^{1}$};
    \draw[->, thick] (3.5, 8) -- (5.9, 8);
    \draw[->, thick] (2.7, 8) -- (0.4, 8);    
    \node at (1.25,6) {$E_{1}^{2}$};
    \draw[->, thick] (1.6, 6) -- (1.9, 6);
    \draw[->, thick] (0.9, 6) -- (0.6, 6);    
    \node at (7.2,6) {$E_{2}^{2}$};
    \draw[->, thick] (7.4, 6) -- (7.9, 6);
    \draw[->, thick] (7.0, 6) -- (6.4, 6);  
    \node at (10.2,6) {$E_{ 3}^{2}$};
    \draw[->, thick] (10.4, 6) -- (11.9, 6);
    \draw[->, thick] (10.0, 6) -- (8.0, 6);  
    \node at (3.1,6) {$E_{4}^{2}$};
   \draw[->, thick] (3.4, 6) -- (3.9, 6);
   \draw[->, thick] (2.8, 6) -- (2.2, 6);      
    \node at (10.9,4) {$E_{ 1}^{3}$};
    \draw[->, thick] (11.1, 4) -- (11.5, 4);
    \draw[->, thick] (10.6, 4) -- (10.1, 4);    
\end{tikzpicture}
\caption{\small A degenerate one-dimensional Morse function. The labeling of its index-$1$ separating saddle points $x_{i,j}^{\circ}$ and local minima $x_{i,j}^{\bullet}$ is not unique. Nevertheless, the labeling process gives a unique one-to-one correspondence between the function values at the two types of points. See~\Cref{fig:generic} for a comparison.}
\label{fig:degenerate}
\end{figure}

\begin{prop}[Theorem 2.8 in~\cite{michel2019small}]\label{thm: Michel-degerate}
Assume that the assumptions of \Cref{thm: continuous-quantative} are satisfied but not \Cref{assump: generic-assumption}. Then, there exists $s_0 > 0$ such that for any $s \in (0, s_0]$, the first $n^{\bullet}-1$ smallest positive eigenvalues of the Witten-Laplacian $\Delta^s_f$ associated with $f$ satisfy
\begin{equation}\nonumber
\delta_{s,\ell} = s\left(\gamma_\ell + o(s) \right)  \e^{- \frac{2 H_{f,\ell}}{s}},
\end{equation}
for $\ell = 1, \ldots, n^{\bullet} - 1$, where $f(x_{\ell}^{\circ}) - f(x_{\ell}^{\bullet}) \leq H_{f, \ell} \leq f(x_{1}^{\circ}) - f(x^\star)$. The constants $H_{f, \ell}$ and $\gamma_{\ell}$ all depend only on the function $f$.

\end{prop}

Taken together, \Cref{thm: HHS-generic} and \Cref{thm: Michel-degerate} give a full proof of \Cref{thm: continuous-quantative}. As is clear, the Morse saddle barrier in \Cref{def:barrier} for the degenerate case is set to $H_f = H_{f,1}$. For completeness, we remark that this result applies to \Cref{assump: generic-assumption}, in which case we conclude that $H_{f,\ell}= f(x_{\ell}^{\circ}) - f(x_{\ell}^{\bullet})$ and $\gamma_\ell$ is given the same as \eqref{eq:gamma_expre}. As such, \Cref{thm: HHS-generic} is implied by \Cref{thm: Michel-degerate}.

\section{Discussion}
\label{sec: conclusion}

In this paper, we have presented a theoretical perspective on the convergence of SGD in nonconvex optimization as a function of the learning rate. Introducing the notion of an lr-dependent SDE, we have leveraged modern tools for the study of diffusions, in particular the spectral theory of diffusion operators, to analyze the dynamics of SGD in a continuous-time model. Specifically, we have shown that the solution to the SDE converges linearly to stationarity under certain regularity conditions and we have presented a concise expression for the linear rate of convergence with transparent dependence on the learning rate for nonconvex Morse functions. Our results show that the linear rate is a constant in the strongly convex case, whereas it decreases rapidly as the learning rate decreases in the nonconvex setting.  We have thus uncovered a fundamental distinction between convex and nonconvex problems. As one implication, we note that noise in the gradients plays a more determinative role in stochastic optimization with nonconvex objectives as opposed to convex objectives. We also note that our results provide a justification for the use of a large initial learning rate in training neural networks.

We propose several directions for future research to consolidate and extend the framework for analyzing stochastic optimization methods via SDEs. A pressing question is to better characterize the gap between the stationary distribution of the lr-dependent SDE and that of the discrete SGD~\cite{kronfeld1993dynamics,pavliotis2014stochastic,dieuleveut2017bridging}. Explicitly, can we improve the upper bound in \Cref{prop: approx}? A related question is whether \Cref{thm: main1} can be improved to $\E f(x_{k}) - f^\star \leq O(s + (1 - \lambda_s s)^{k})$, with the hidden coefficients having less dependence on the time horizon $ks$. A possible approach to overcoming this difficulty in the discrete regime is to obtain a discrete version of the Poincar\'e inequality in $\mathbb{R}^{d}$ (\Cref{thm: villani-poincare-inq}). From a different angle, it is noteworthy that $(s/2) \Delta \rho_{s}$ in the Fokker--Planck--Smoluchowski equation~\eqref{eqn: Fokker-Planck} corresponds to vanishing
viscosity in fluid mechanics. \Cref{subsec: supplement-viscosity} presents several open problems from this viewpoint. To widen the scope of this framework, it is important to extend our results to the setting where the gradient noise is heavy-tailed~\cite{simsekli2019tail}.

From a practical standpoint, our work offers several promising avenues for future research in deep learning. First, a seemingly straightforward direction is to extend our SDE-based analysis to various learning rate schedules used in practice in training deep neural networks, such as diminishing learning rate and cyclical learning rates~\cite{bottou2018optimization,smith2017cyclical}. More broadly, it is of great interest to use SDEs to study and improve on practical variants of SGD, including RMSProp and Adam~\cite{tieleman2012lecture,kingma2014adam}. Second, our results would likely to be useful in guiding the choice of hyperparameters of deep neural networks from an optimization viewpoint. For instance, recognizing the essence of the exponential decay constant $\lambda_s$ in determining the convergence rate of SGD, how to choose the neural network architecture and the loss function so as to get a small value of the Morse saddle barrier $H_f$? Finally, we wonder if the lr-dependent SDE might give insights into generalization properties of neural networks such as local elasticity~\cite{he2019local} and implicit regularization~\cite{zhang2016understanding,gunasekar2018characterizing}.




{\small
\subsection*{Acknowledgments}
We would like to thank Zhuang Liu and Yu Sun for helpful conversations about practical experience in deep learning. This work was supported in part by NSF through CAREER DMS-1847415, CCF-1763314, and CCF-1934876, and the Wharton Dean’s Research Fund.  We also recognize support from the Mathematical Data Science program of the Office of Naval Research under grant number N00014-18-1-2764.

\bibliographystyle{alpha}
\newcommand{\etalchar}[1]{$^{#1}$}

}

\clearpage
\appendix

\section{Technical Details for Sections~\ref{sec: intro} and \ref{sec: preliminaries}}
\label{sec: proof_preliminaries}

\subsection{Approximating differential equations}
\label{subsec: supplement-DE-algorithms}

\Cref{fig:chart} presents a diagram that shows approximating surrogates for GD, SGD, and SGLD at multiple scales. In the case of SGD, for example, the inclusion of only $O(1)$ terms leads to the ODE $\dot{X} = -\nabla f(X)$, whereas the inclusion of up to $O(\sqrt{s})$ terms leads to the lr-dependent SDE \eqref{eqn: sgd_high_resolution_formally}. For GD and SGLD, $O(\sqrt{s})$ terms are not found in the expansion as in the derivation of \eqref{eqn: sgd_high_resolution_formally}. The $O(\sqrt{s})$-approximation, therefore, leads to the same differential equation as the $O(1)$-approximation for both GD and SGLD.


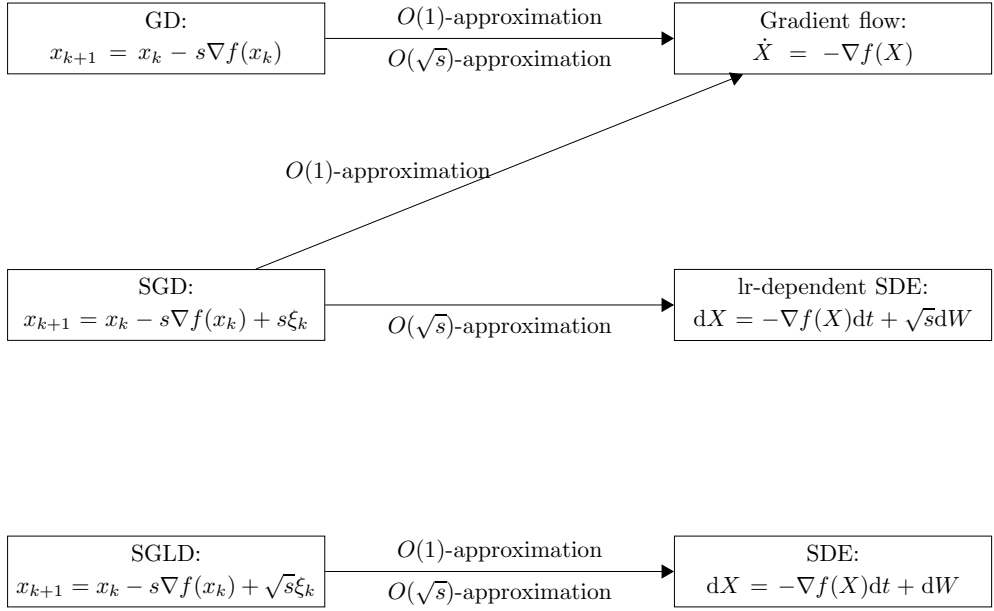
\begin{figure}[htb!]
\begin{center}
\resizebox{0.8\textwidth}{!}{%
\begin{tikzpicture}[node distance=4cm]
    \node (n00) [box, draw=black] {\small GD:\\ $x_{k+1} = x_{k} - s\nabla f(x_k)$};
    \node (n01) [box, draw=black,  below of=n00] {\small SGD: \\$x_{k+1} = x_{k} - s\nabla f(x_k) + s \xi_k$};
    \node (n10) [box, draw=black,  below of=n01] {\small  SGLD: \\$x_{k+1} = x_{k} - s\nabla f(x_k) + \sqrt{s} \xi_k$};
    \node (n11) [box, draw=black,   right  of=n00, xshift=+6cm] {\small  Gradient flow: \\$\dot{X} = -\nabla f(X)$};
    \node (n100) [box, draw=black,  below of=n11] {\small lr-dependent SDE: \\$\dd X = -\nabla f(X) \dd t + \sqrt{s} \dd W$};
    \node (n101) [box, draw=black,  below of=n100] {\small SDE: \\$\dd X = -\nabla f(X) \dd t + \dd W$};
 
    \draw [arrow] (n00)  to node[above] {\small $O(1)$-approximation} node[below]{\small $O(\sqrt{s})$-approximation} (n11);
    \draw [arrow] (n10)  to node[above] {\small $O(1)$-approximation} node[below]{\small $O(\sqrt{s})$-approximation} (n101);
   \draw [arrow] (n01)  to node[left] {\small $O(1)$-approximation} (n11);
    \draw [arrow] (n01)  to node[below]{\small $O(\sqrt{s})$-approximation} (n100);
     
\end{tikzpicture}
}%
\end{center}
\caption{\small Diagram showing the relationship between three discrete algorithms and their $O(1)$-approximating and $O(1) + O(\sqrt{s})$-approximating differential equations. Note that the inclusion of only $O(1)$-terms does not distinguish between GD and SGD.}
\label{fig:chart}
\end{figure}



\subsection{Derivation of the Fokker--Planck--Smoluchowski equation}
\label{subsec: fok-plk}

To derive the lr-dependent Fokker--Planck--Smoluchowski equation \eqref{eqn: Fokker-Planck}, we first state the following lemma.
\begin{lem}[It\^o's lemma]
\label{lem: ito-lem}
For any $f \in C^\infty(\mathbb{R}^{d})$ and $g \in C^\infty([0, +\infty) \times \mathbb{R}^{d})$, let $X_s(t)$ be the solution to the lr-dependent SDE~\eqref{eqn: sgd_high_resolution_formally}. Then, we have
\begin{equation}\label{eqn: ito-formula}
\dd g(t, X_{s}(t)) =  \left( \frac{\partial g}{\partial t} - \nabla f \cdot \nabla g + \frac{s}{2} \Delta g \right) \dd t  + \sqrt{s} \left(\sum_{i=1}^{d} \frac{\partial g}{\partial x_{i}}\right) \dd W. 
\end{equation}
\end{lem}


From this lemma, we get
\begin{align}
\frac{\dd \mathbb{E}[g(t, X_{s}(t)) | X_{s}(t')]}{ \dd t } =  \ & \frac{\partial \mathbb{E}[g(t, X_{s}(t)) | X_{s}(t')]}{\partial t} - \nabla f \cdot \nabla \mathbb{E}[g(t, X_{s}(t)) | X_{s}(t')] \nonumber \\
& + \frac{s}{2} \Delta \mathbb{E}[g(t, X_{s}(t)) | X_{s}(t')], \label{lem: ito-lem-expectation}
\end{align}
for $t \ge t'$. Setting $v_s(t', x)  = \mathbb{E}[g(t, X_{s}(t)) | X_{s}(t') = x]$, from \eqref{lem: ito-lem-expectation} we see that $v_s(t', x)$ satisfies the following differential equation:
\begin{equation}
\label{eqn: backward-kolmogorov}
\left\{ \begin{aligned}
& \frac{\partial v_s}{\partial t'} = \nabla f \cdot \nabla v_s - \frac{s}{2} \Delta v_s \\
& v_{s}(t, x) = g(t, x).
\end{aligned} \right.
\end{equation}
Recognizing the invariance of translation of time and letting $u_{s}(t-t', x) = v_s(t', x)$, we can reduce \eqref{eqn: backward-kolmogorov} to the following backward Fokker--Planck--Smoluchowski equation:
\begin{equation}
\label{eqn: backward-kolmogorov-1}
\left\{ \begin{aligned}
& \frac{\partial u_s}{\partial t} = -\nabla f \cdot \nabla u_s + \frac{s}{2} \Delta u_s \\
& u_{s}(0, x) = g(t, x).
\end{aligned} \right.
\end{equation}

Next, from the Chapman--Kolmogorov equation, we get
\[
\rho_{s}(t, x) = \int_{\mathbb{R}^{d}} \rho_{s}(t, x | 0, y) \rho_{s}(0, y) \dd y,
\]
and by switching the order of the integration, we obtain
\begin{align}
\int_{\mathbb{R}^{d}} u_{s}(0, x) \rho_{s}(t, x) \dd x & = \int_{\mathbb{R}^{d}} g(x) \rho_{s}(t, x) \dd x  \nonumber \\
                                                                                 & = \int_{\mathbb{R}^{d}} g(x)  \left( \int_{\mathbb{R}^{d}}  \rho_{s}(t, x | 0, y) \rho_{s}(0, y) \dd y \right) \dd x  \nonumber  \\
                                                                                 & = \int_{\mathbb{R}^{d}}  \rho_{s}(0, y) \left(  \int_{\mathbb{R}^{d}} g(x) \rho_{s}(t, x | 0, y) \dd x \right) \dd y \nonumber  \\
                                                                                 & = \int_{\mathbb{R}^{d}}  \rho_{s}(0, y) u_{s}(t, y) \dd y \nonumber  \\
                                                                                 & = \int_{\mathbb{R}^{d}}  \rho_{s}(0, x) u_{s}(t, x) \dd x. \label{eqn: u_rho-exchange}
\end{align}
Making use of the~backward Fokker--Planck--Smoluchowski equation~\eqref{eqn: backward-kolmogorov-1} and switching the order of integration~\eqref{eqn: u_rho-exchange}, we get
\begin{align*}
\int_{\mathbb{R}^{d}} u_s(0, x)  \frac{ \partial \rho_{s}(t, x)} {\partial t} \bigg|_{t = 0} \dd x & = \frac{\dd}{\dd t} \int_{\mathbb{R}^{d}} u_s(0, x)   \rho_{s}(t, x) \dd x \bigg|_{t = 0}\\
                                                                                                                      & =  \frac{\dd}{\dd t} \int_{\mathbb{R}^{d}} u_s(t, x)   \rho_{s}(0, x) \dd x \bigg|_{t = 0}\\
                                                                                                                      & = \int_{\mathbb{R}^{d}} \frac{\partial u_s(t, x)}{\partial t} \bigg|_{t = 0}  \rho_{s}(0, x) \dd x \\
                                                                                                                      & = \int_{\mathbb{R}^{d}} \left( -\nabla f(x) \cdot \nabla u_s(0, x) + \frac{s}{2} \Delta u_s(0, x) \right) \rho_{s}(0, x) \dd x \\
                                                                                                                      & =  \int_{\mathbb{R}^{d}} u_{s}(0, x) \left( \nabla \cdot (\rho_{s}(0, x) \nabla f(x) )  + \frac{s}{2} \Delta \rho_s(0, x) \right)  \dd x.
\end{align*}
Hence, we derive the forward Fokker--Planck--Smoluchowski equation at $t = 0$ for an arbitrary smooth function $u_s(0, x) = g(t,x)$. Noting that $t = 0$ can be replaced by any time $t$, we complete the derivation of the Fokker--Planck--Smoluchowski equation.

\subsection{The uniqueness of Gibbs invariant distribution}
\label{subsec: proof-unique-steady}

We begin by proving that the probability density $\mu_s$ is an invariant distribution of~\eqref{eqn: Fokker-Planck}. Plugging
\begin{equation}\nonumber
\nabla \mu_s =  - \frac{2}{s} \left( \nabla f \right)  \mu_s
\end{equation}
into \eqref{eqn: Fokker-Planck} gives
\begin{align}
\nabla \cdot (\mu_s \nabla f)  = \nabla \mu_s \cdot \nabla f + \mu_s \Delta f   = - \frac{2}{s} \|\nabla f\|^{2} \mu_s + (\Delta f) \mu_s \label{eqn: advection}
\end{align}
and
\begin{equation}
\Delta \mu_s = - \frac{2}{s} \nabla f \cdot \nabla \mu_s - \frac{2}{s} \mu_s \Delta f = \frac{4}{s^2} \|\nabla f\|^2 \mu_s - \frac{2}{s} \mu_s \Delta f. \label{eqn: diffusion}
\end{equation}
Combining~\eqref{eqn: advection} and~\eqref{eqn: diffusion} yields
\[
\nabla \cdot (\mu_s \nabla f) + \frac{s }{2} \Delta  \mu_s  = 0.
\]

We now proceed to show that the probability density $\mu_s$ is unique. To derive a contradiction, we assume that there exists another distribution $\vartheta_s$ satisfying the Fokker--Planck--Smoluchowski equation:
\begin{equation}
\label{eqn: steady-state-FPS}
\nabla \cdot (\vartheta_s \nabla f) + \frac{s }{2} \Delta  \vartheta_s  = 0.
\end{equation}
Write $\varpi _s = \vartheta_s \mu_s^{-1}$ and recall the operator $\mathscr{L}_s$ defined in~\Cref{sec:proof-crefprop:convv}. We can rewrite \eqref{eqn: steady-state-FPS} as
\[
\mathscr{L}_{s} \varpi_s = 0.
\]
Using \Cref{lem: equivlent-h}, we have
\[
0 = \int_{\mathbb{R}^{d}} (\mathscr{L}_{s} \varpi_s)  \varpi_{s} \dd \mu_s = - \frac{s}{2} \int_{\mathbb{R}^{d}} \| \nabla  \varpi_s \|^{2} \dd \mu_s \leq 0.
\]
Hence, $\varpi_s$ must be a constant on $\mathbb{R}^{d}$. Furthermore, since both $\mu_s$ and $\vartheta_s$ are probability densities, it must be the case that $\varpi_s \equiv 1$. In other words, $\vartheta_s$ is identical to $\mu_s$. The proof is complete.

\subsection{Proof of \Cref{prop: unique-existence}}
\label{subsec: proof_wellposedness}

Recall that Section~\ref{subsubsec: schrodinger} shows that the transition probability density $\rho_{s}(t, x)$ in $C^1([0, +\infty), L^{2}(\mu_s^{-1}))$ governed by the Fokker--Planck--Smoluchowski equation~\eqref{eqn: Fokker-Planck} is equivalent to the function $\psi_{s}(t, x)$ in $C^1([0, +\infty), L^2(\mathbb{R}^{d}))$ governed by~\eqref{eq:schrodinger}. Moreover, in~\Cref{subsubsec: schrodinger}, we have shown that the spectrum of the Schr\"odinger operator $-s\Delta + V_s$ satisfies
\[
0 = \zeta_{s,0} < \zeta_{s,1}  \leq \cdots \leq \zeta_{s,\ell} \leq \cdots < +\infty.
\]
Since $L^{2}(\mathbb{R}^{d})$ is a Hilbert space, there exists a standard orthogonal basis corresponding to the spectrum of $-s\Delta + V_s$:
\[
\mu_s=\phi_{s,0},\; \phi_{s,1},\; \ldots,\; \phi_{s,\ell}, \; \ldots \in L^{2}(\mathbb{R}^d).
\]
Then, for any initialization $\psi_{s}(0, x) \in L^{2}(\mathbb{R}^{d})$, there exist constants $c_{\ell}$ ($\ell = 1,2, \ldots$) such that 
\[
\psi_{s}(0, \cdot) = \sqrt{\mu_s} + \sum_{\ell=1}^{+\infty}c_{\ell} \phi_{s,\ell}. 
\]
Thus, the solution to the partial differential equation~\eqref{eq:schrodinger} is 
\[
\psi_{s}(t, \cdot) = \sqrt{\mu_s} + \sum_{\ell=1}^{+\infty}c_{\ell}\e^{-\zeta_{s,\ell}t} \phi_{s, \ell}. 
\]
Recognizing the transformation $\psi_{s}(t, \cdot) =\rho_{s}(t, \cdot)/\sqrt{\mu_s}$, we recover 
\[
\rho_{s}(t, \cdot) = \mu_s + \sum_{\ell=1}^{+\infty}c_{\ell}\e^{-\zeta_{s,\ell}t} \phi_{s,\ell} \sqrt{\mu_s}.
\]
Note that $\zeta_{s,\ell}$ is positive for $\ell \geq 1$. Thus, the proof is finished.

\section{Technical Details for Section~\ref{sec:main-results}}
\label{sec: proof-main-results}


\subsection{Proof of~\Cref{prop:lambda_str_mu}} 
\label{subsec: proof-mu-strong}

Here, we prove~\Cref{prop:lambda_str_mu} using the Bakry--Emery theorem, which is a Poincar\'{e}-type inequality for $\mu$-strongly convex functions. As a direct consequence of this theorem, the exponential decay constant for strongly convex objectives does not depend on the learning rate $s$ and the ambient dimension $d$. 


\begin{lem}[Bakry--Emery theorem]\label{thm: bakry-emry-poincare-inq}
Let $f$ be an infinitely differentiable function defined on $\mathbb{R}^{d}$. If $f$ is $\mu$-strongly convex, then the measure $\dd \mu_s$ satisfies the Poincar\'e-type inequality as in \Cref{thm: villani-poincare-inq} with $\lambda_s = \mu$; that is, for any smooth function $h$ with a compact support,
\[
\int_{\mathbb{R}^{d}} h^{2} \dd \mu_{s} - \left( \int_{\mathbb{R}^{d}} h \dd \mu_{s} \right)^{2} \leq  \frac{s}{2\mu}\int_{\mathbb{R}^{d}} \|\nabla h\|^{2} \dd \mu_{s}.
\]
\end{lem}


\Cref{thm: bakry-emry-poincare-inq} serves as the main technical tool in the proof of \Cref{prop:lambda_str_mu}. Its proof is in \Cref{subsec: proof-bakry-emry-poincare-inq}. Now, we prove the following result using \Cref{thm: bakry-emry-poincare-inq}.

\begin{lem}\label{thm: rate-l2-2}
Under the same assumptions as in~\Cref{prop:lambda_str_mu}, $\rho_s(t,\cdot)$ converges to the Gibbs distribution $\mu_s$ in $L^{2}(\mu_s^{-1})$ at the rate
\begin{equation}\label{eqn: rate-l2-2}
\left\| \rho_{s}(t, \cdot) - \mu_{s} \right\|_{\mu_s^{-1}} \leq \e^{- \mu t }\left\| \rho_{s} - \mu_{s} \right\|_{ \mu_s^{-1}}.
\end{equation}
\end{lem}

\begin{proof}[Proof of \Cref{thm: rate-l2-2}]
It follows from~\eqref{eqn: converge-FK} that
\[
\frac{\dd}{\dd t} \left\| \rho_s(t, \cdot) - \mu_s \right\|_{\mu_s^{-1} } ^{2} = - s \int_{\mathbb{R}^{d}} \| \nabla h_{s} \|^{2}  \dd \mu_{s}.
\]
Next, using~\Cref{thm: bakry-emry-poincare-inq} and recognizing the equality $\int_{\mathbb{R}^{d}} h_s\dd \mu_s = \int_{\mathbb{R}^{d}} \rho_s(t, x) \dd x = 1$, we get
\begin{align}
\frac{\dd}{\dd t} \left\| \rho_s(t, \cdot) - \mu_s \right\|_{\mu_s^{-1} } ^{2}  & \leq - 2 \mu \left( \int_{\mathbb{R}^{d}} h_s^{2} \dd \mu_{s} - \left( \int_{\mathbb{R}^{d}} h_s \dd \mu_{s} \right)^{2} \right) \nonumber \\
& = - 2 \mu \left( \int_{\mathbb{R}^{d}} h_s^{2} \dd \mu_{s} - 1 \right) \nonumber \\
& = - 2 \mu \int_{\mathbb{R}^{d}} (h_s - 1)^{2}  \dd \mu_{s} \nonumber \\
& = - 2 \mu \left\|\rho_s(t,\cdot) - \mu_s \right\|_{\mu_{s}^{-1}} ^{2}. \nonumber
\end{align}
Integrating both sides yields~\eqref{eqn: rate-l2-2}, as desired.
\end{proof}

Leveraging \Cref{thm: rate-l2-2}, we proceed to complete the proof of~\Cref{prop:lambda_str_mu}.

\begin{proof}[Proof of~\Cref{prop:lambda_str_mu}]
Using \Cref{thm: rate-l2-2}, we get
\begin{align*}
\left| \E f(X_s(t)) - \E f(X(\infty)) \right|  & = \left|  \int_{\mathbb{R}^{d}} f(x) \left( \rho_{s}(t, x) - \mu_s(x) \right) \dd x\right| \\
&= \left|  \int_{\mathbb{R}^{d}} (f(x) - f^{\star}) \left( \rho_{s}(t, x) - \mu_s(x) \right) \dd x\right| \\
                                                                    & \leq  \left(\int_{\mathbb{R}^{d}}  (f(x) - f^{\star}) ^{2} \mu_s(x) \dd x \right)^{\frac{1}{2}} \left( \int_{\mathbb{R}^{d}}  \left( \rho_{s}(t, x) - \mu_s(x) \right)^{2} \mu_s^{-1} \dd x \right)^{\frac{1}{2}} \\
                                                                    &  \leq C(s) \e^{-\mu  t} \left\|\rho- \mu_{s} \right\|_{\mu_s^{-1}},
\end{align*}
where the first inequality applies the Cauchy-Schwarz inequality and
\[
C(s) = \left(\int_{\mathbb{R}^{d}}  (f - f^{\star}) ^{2} \mu_s \dd x \right)^{\frac{1}{2}}
\]
is an increasing function of $s$.
\end{proof}
\subsubsection{Proof of Lemma~\ref{thm: bakry-emry-poincare-inq}}
\label{subsec: proof-bakry-emry-poincare-inq}

We introduce two operators $\Gamma_s$ and $\Gamma_{s,2}$ that are built on top of the linear operator $\mathscr{L}_{s}$ defined in~\eqref{eqn: gene-oper}. For any $g_{1}, g_{2} \in L^2(\mu_s)$, let
\begin{equation}
\label{eqn: gamma1-L}
\Gamma_s(g_{1}, g_{2}) = \frac{1}{2} \left[ \mathscr{L}_s(g_{1}g_{2}) - g_{1} \mathscr{L}_sg_{2} - g_{2} \mathscr{L}_sg_{1} \right]
\end{equation}
and 
\begin{equation}
\label{eqn: gamma2-L}
\Gamma_{s,2}(g_{1}, g_{2}) = \frac{1}{2} \left[ \mathscr{L}_s\Gamma_s(g_{1}, g_{2}) - \Gamma_s(g_{1}, \mathscr{L}_sg_{2}) - \Gamma_s(g_{2}, \mathscr{L}_sg_{1}) \right].
\end{equation}
A simple relationship between the two operators is described in the following lemma.
\begin{lem}\label{lem: curvature-estimate}
Under the same assumptions as in~\Cref{thm: bakry-emry-poincare-inq}, for any $g \in  L^2(\mu_s)$ we have
\begin{equation}\nonumber
\Gamma_{s,2}(g, g) \geq \mu \Gamma_s(g, g).
\end{equation}
\end{lem}

\begin{proof}[Proof of Lemma~\ref{lem: curvature-estimate}]

Note that
\[
\mathscr{L}_s(g_{1}g_{2}) =  - g_{1} (\nabla f \cdot \nabla g_{2}) - g_{2} (\nabla f \cdot \nabla g_{1})\nonumber + \frac{s}{2} \left( g_{1} \Delta g_{2} + g_{2} \Delta g_{1} + 2 \nabla g_{1} \cdot \nabla g_{2} \right) 
\]
and 
\[
 g_{1} \mathscr{L}_sh_{2} = - g_{1} \nabla f \cdot \nabla g_{2} + \frac{s}{2} g_{1} \Delta g_{2}, \qquad
 g_{2} \mathscr{L}_sg_{1} = - g_{2} \nabla f \cdot \nabla g_{1} + \frac{s}{2} g_{2} \Delta g_{1}. 
\]
Then, the operator $\Gamma_s$ must satisfy
\begin{equation}\label{eqn: gamma1-L-final-calcuation}
\Gamma_s(g, g) = \frac{s}{2} \left( \nabla g \cdot \nabla g \right).
\end{equation}
Next, together with the equality
\[
\frac{1}{2} \Delta ( \| \nabla g \| ^{2}) = \nabla g \cdot \nabla (\Delta g) + \mathbf{Tr}[(\nabla^{2}g)^{T} (\nabla^{2}g)],
\]
we obtain that the operator $\Gamma_{s,2}$ satisfies
\begin{equation}
\label{eqn: gamma2-L-final-calcuation}
\Gamma_{s,2}(g, g) = \frac{s}{2} (\nabla g)^{T} \nabla^{2} f (\nabla g) + \frac{s^2}{4}\mathbf{Tr}[(\nabla^{2}g)^{T} (\nabla^{2}g)],
\end{equation}
where $\mathbf{Tr}$ is the standard trace of a squared matrix. Recognizing that the objective $f$ is $\mu$-strongly convex, a comparison between \eqref{eqn: gamma1-L-final-calcuation} and~\eqref{eqn: gamma2-L-final-calcuation} completes the proof.

\end{proof}

Recall that $h_s(t, \cdot) \in L^{2}(\mu_s)$ is the solution to the partial differential equation~\eqref{eqn: FPS-equiv}, with the initial condition $h_{s}(0, \cdot) = h$. Define
\begin{equation}
\label{eqn: Lamda-poincare}
\Lambda_{1,s}(t) = \int_{\mathbb{R}^{d}} h^{2}_{s}(t, \cdot) \dd \mu_{s}.
\end{equation}


The following lemma considers the derivatives of $\Lambda_{1, s}(t)$.
\begin{lem}
\label{lem: derivative-1}
Under the same assumptions as in~\Cref{thm: bakry-emry-poincare-inq}, we have
\begin{equation}
\label{eqn: derivative}
\left\{\begin{aligned}
& \dot{ \Lambda}_{1,s}(t) = -2 \int_{\mathbb{R}^{d}} \Gamma_s(h_s, h_s) \dd \mu_{s}\\
& \ddot{\Lambda}_{1,s}(t) = 4 \int_{\mathbb{R}^{d}} \Gamma_{s,2}(h_s, h_s)  \dd \mu_{s}.
\end{aligned}\right.
\end{equation}
\end{lem}
\begin{proof}[Proof of Lemma~\ref{lem: derivative-1}]

Taking together \eqref{eqn: converge-FK} and~\eqref{eqn: gamma1-L-final-calcuation}, we have
\[
\int_{\mathbb{R}^{d}} \Gamma_s(h_s, h_s) \mu_{s}\dd \mu_{s} = - \int_{\mathbb{R}^{d}}  h_s\mathscr{L}_sh_s \dd \mu_{s} .
\]
Since $h_s(t, \cdot) \in L^{2}(\mu_s)$ is the solution to the partial differential equation~\eqref{eqn: FPS-equiv}, we get
\[
\dot{ \Lambda}_{1,s}(t) = 2\int_{\mathbb{R}^{d}}  h_s\mathscr{L}_sh_s \dd \mu_{s}  = - 2 \int_{\mathbb{R}^{d}} \Gamma_s(h_s, h_s)  \dd \mu_{s}.
\]
Furthermore, by the definition of $\Gamma_{s,2}$  and integration by parts, we have\footnote{See the calculation in~\cite{bakry2013analysis}.}
\[
\int_{\mathbb{R}^{d}} \Gamma_{s,2}(h_s, h_s) \dd \mu_{s} = \int_{\mathbb{R}^{d}}   (\mathscr{L}_sh_s)^{2} \dd \mu_{s}.
\]
From~\Cref{lem: equivlent-h}, we know that the linear operator $\mathscr{L}_s$ is self-adjoint. Then, we obtain the  second derivative as
\[
\ddot{ \Lambda}_{1,s}(t) = 2 \int_{\mathbb{R}^{d}} (\mathscr{L}_sh_s )^{2} \dd \mu_{s} +  2 \int_{\mathbb{R}^{d}} h_s \mathscr{L}_s^{2} h_s  \dd \mu_{s} = 4 \int_{\mathbb{R}^{d}} \Gamma_{s,2}(h_s, h_s) \dd \mu_{s}. 
\]
\end{proof}

Finally, we complete the proof of~\Cref{thm: bakry-emry-poincare-inq}.

\begin{proof}[Proof of Lemma~\ref{thm: bakry-emry-poincare-inq}]
Using~\Cref{lem: curvature-estimate} and~\Cref{lem: derivative-1}, we obtain the following inequality:
\begin{equation}
\label{eqn: inequal-integral}
\ddot{ \Lambda}_{1,s}(t) \geq  - 2\mu \dot{\Lambda}_{1,s}(t). 
\end{equation}
From the definition of $\Lambda_{1,s}(t)$, we have
\[
\Lambda_{1,s}(0) - \Lambda_{1,s}(\infty) = \int_{\mathbb{R}^{d}} h ^{2}  \dd \mu_{s} -  \left( \int_{\mathbb{R}^{d}} h  \dd \mu_{s} \right)^2,
\]
where the second term on the right-hand side follows from~\Cref{thm: converge} and 
\[
\int_{\mathbb{R}^d}h\dd\mu_s = \int_{\mathbb{R}^d}\rho\dd x = 1.
\]
By~\Cref{thm: converge}, we get $h_{s}(\infty, \cdot) \equiv 1$, which together with~\eqref{eqn: derivative} gives
\[
\dot{\Lambda}_{1,s}(0) - \dot{\Lambda}_{1,s}(\infty) = -2\int_{\mathbb{R}^{d}} \Gamma_s(h, h) \dd \mu_{s} = - s \int_{\mathbb{R}^{d}} \| \nabla h\|^{2}  \dd \mu_{s}.
\]
The final equality follows from~\eqref{eqn: gamma1-L-final-calcuation}. Integrating both sides of the inequality~\eqref{eqn: inequal-integral}, we have
\[
-2\mu \left(\Lambda_{1,s}(0) - \Lambda_{1,s}(\infty) \right)\leq  \dot{\Lambda}_{1,s}(0) - \dot{\Lambda}_{1,s}(\infty), 
\]
which completes the proof.
\end{proof}


\subsection{Convergence in $L^{1}(\mathbb{R}^d)$}
\label{sbsec: l1-space-density}

\Cref{thm: rate-l2-2} can be extended to the more general and natural function space $L^1(\mathbb{R}^d)$. From~\Cref{lem: space-holder}, we know that $L^{2}(\mu_s^{-1}) \subset L^{1}(\mathbb{R}^d)$. This is formulated in the following lemma.

\begin{lem}\label{thm: rate-l1-2}
Under the same assumptions as in~\Cref{prop:lambda_str_mu}, $\rho_s(t,\cdot)$ converges to the Gibbs distribution $\mu_s$ in $L^{1}(\mathbb{R}^d)$ at the rate\footnote{Note that the $L^1(\mathbb{R}^d)$ norm, $\left\| \rho_{s}(t, \cdot) - \mu_{s} \right\|_{L^1(\mathbb{R}^d)}$, is defined as
\[
\left\| \rho_{s}(t, \cdot) - \mu_{s} \right\|_{L^1(\mathbb{R}^d)} = \int_{\mathbb{R}^d}| \rho_{s}(t, \cdot) - \mu_{s}| \dd x.
\]
}
\begin{equation}
\label{eqn: rate-l1-core}
\left\| \rho_{s}(t, \cdot) - \mu_{s} \right\|_{L^1(\mathbb{R}^d)} \leq \sqrt{2}\e^{- \mu t } H(\rho| \mu_{s}),
\end{equation}
where the relative entropy is
\[
H(\rho | \mu_{s}) = \int_{\mathbb{R}^{d}} \rho \log \left( \frac{\rho}{\mu_{s}}\right)
\dd x.
\]
\end{lem}

Similarly, the proof of~\Cref{thm: rate-l1-2} will be based on the following log-Sobolev type inequality.


\begin{lem}[log-Sobolev inequality]
\label{thm: log-sobolev-inq}
Let $f$ be an infinitely differentiable function defined on $\mathbb{R}^{d}$. If $f$ is $\mu$-strongly convex, then the measure $\dd \mu_s$ satisfies a log-Sobolev inequality. That is,
\begin{equation}
\label{eqn: log-sobolev}
\mathbf{Ent}[h^{2}]  \leq \frac{s}{\mu} \int_{\mathbb{R}^{d}} \| \nabla h\|^{2}  \dd \mu_{s},
\end{equation}
where the entropy is defined as
\[
\mathbf{Ent}[h^{2}] = \int_{\mathbb{R}^{d}} h^{2} \log h^{2}  \dd \mu_{s} -  \int_{\mathbb{R}^{d}} h^{2}  \dd \mu_{s} \left(\log  \int_{\mathbb{R}^{d}} h^{2}  \dd \mu_{s} \right).
\]
\end{lem}

\begin{proof}[Proof of Theorem~\ref{thm: rate-l1-2}]
By the Csisz\'ar--Kullback inequality, we have
\begin{align}
\label{eqn: CK-inequal}
\left\| \rho_{s}(t, \cdot) - \mu_{s} \right\|_{L^{1}(\mathbb{R}^{d})}^{2}  \leq 2 H(\rho_{s}(t, \cdot) | \mu_{s}) = 2 \int_{\mathbb{R}^{d}} h_s \log h_s   \dd \mu_s. 
\end{align}
From \Cref{lem: derivative-log-2}, we have 
\begin{align}
\frac{\dd H(\rho_{s}(t, \cdot) | \mu_{s})}{\dd t}  = - \int_{\mathbb{R}^{d}} h_s \Gamma_s(\log h_s, \log h_s)  \dd \mu_s =  - 2s   \int_{\mathbb{R}^{d}}  \| \nabla \sqrt{h_s}\|^{2}  \dd \mu_s. \label{eqn: entropy-derivative} 
\end{align}
Finally, using~\Cref{thm: log-sobolev-inq}, we obtain the estimate for the derivative of the entropy~\eqref{eqn: entropy-derivative} as
\begin{align*}
\frac{\dd H(\rho_{s}(t, \cdot) | \mu_{s})}{\dd t} \leq - 2s \int_{\mathbb{R}^{d}} (h_s \log h_s) \mu_s \dd x = - 2 \mu H(\rho_{s}(t, \cdot) | \mu_{s}),
\end{align*}
which completes the proof.
\end{proof}


\subsubsection{Proof of~\Cref{thm: log-sobolev-inq}}
\label{subsubsec: proof-log-sobolev-inq}
We consider the integral
\begin{equation}
\label{eqn: lambda-log-soboleve}
\Lambda_{2,s}(t) =  H(\rho_{s}(t, \cdot) | \mu_{s}) = \int_{\mathbb{R}^{d}} h_{s}(t, \cdot)\log ( h_{s}(t, \cdot)  \dd \mu_{s},
\end{equation}
with $h_{s}(0, x) = h^{2}(x)$. We now proceed to find the derivatives of $\Lambda_{2, s}(t)$ with respect to time $t$, which is formulated as the following lemma. 

\begin{lem}
\label{lem: derivative-log-2}
Under the same assumptions as in~\Cref{thm: log-sobolev-inq}, we have
\begin{equation}
\label{eqn: derivative-log-2}
\left\{\begin{aligned}
& \dot{ \Lambda}_{2,s}(t) = - \int_{\mathbb{R}^{d}} h_s  \Gamma_s(\log h_s, \log h_s) \ \dd \mu_{s}  \\
&\ddot{ \Lambda}_{2,s}(t)= 2 \int_{\mathbb{R}^{d}} h_s(  \Gamma_{s, 2}( \log h_s, \log h_s)  \dd \mu_{s}.
\end{aligned}\right.
\end{equation}
\end{lem}

\begin{proof}[Proof of Lemma~\ref{eqn: derivative-log-2}]
Recall the linear operator $\mathscr{L}_{s}$ defined in~\eqref{eqn: gene-oper}. Using integration by parts, we have 
\[
 \int_{\mathbb{R}^{d}} \left(1 + \log h_s \right) \mathscr{L}_sh_s \dd \mu_s =  - \frac{s}{2} \int_{\mathbb{R}^{d}} \nabla  \left(1 + \log h_s \right)  \cdot \nabla h_s \dd \mu_s.
 \]
Since $h_{s}(t,\cdot)$ is the solution to the partial differential equation~\eqref{eqn: FPS-equiv}, we obtain the first derivative as 
\begin{align*}
\dot{\Lambda}_{2,s}(t) & = \int_{\mathbb{R}^{d}} \left(1 + \log h_s \right) \mathscr{L}_sh_s  \dd \mu_s  = - \int_{\mathbb{R}^{d}} h_s  \Gamma_s(\log h_s, \log h_s) \dd \mu_{s} . 
\end{align*}
Furthermore, recognizing the definitions of $\Gamma_{s,2}$, $\Gamma_{s}$ and $\mathscr{L}_{s}$, we have 
\[
\left\{ \begin{aligned}
        & \mathscr{L}_s\left( \frac{\Gamma_{s}(h_s, h_s)}{h_s} \right)   = - \frac{\Gamma_{s}(h_s, h_s) \mathscr{L}_s h_s }{h_s^2} + \frac{\mathscr{L}_s \Gamma_{s}(h_s, h_s) }{h_s}  - \frac{\Gamma_{s}(h_s, \Gamma_s(h_s, h_s)) }{h_s^2} + \frac{2( \Gamma_{s}(h_s, h_s) )^2}{h_s^3} \\
        & \Gamma_{s, 2}(\log h_s, \log h_s) = \frac{\Gamma_{s, 2}(h_s, h_s)}{ h_s^2} - \frac{\Gamma_s(h_s, \Gamma_s(h_s, h_{s})}{ h_s^3} + \frac{ ( \Gamma_{s, 2}(h_s, h_s))^2}{h_s^4}.
        \end{aligned} \right. 
\]
On the other hand, the integral with the measure $\dd \mu_{s}$ for the linear operator $\mathscr{L}_{s}$ acting on $\gamma_{s}(h_s,h_s)/h_s$ is zero; that is,
\[
\int_{\mathbb{R}^{d}} \mathscr{L}_s\left( \frac{\Gamma_{s}(h_s, h_s)}{h_s} \right) \dd \mu_s = 0.
\]
Then the first equality above by the definitions can be written as
\[
\int_{\mathbb{R}^{d}} \frac{\Gamma_{s}(h_s, h_s) \mathscr{L}_s h_s }{h_s^2}  \dd \mu_s =  \int_{\mathbb{R}^{d}}  \left(    \frac{\mathscr{L}_s \Gamma_{s}(h_s, h_s) }{h_s}    - \frac{\Gamma_{s}(h_s, \Gamma_s(h_s, h_s)) }{h_s^2} + \frac{2( \Gamma_{s}(h_s, h_s) )^2}{h_s^3} \right)  \dd \mu_s. 
\]
Finally, we obtain the second derivative  as
\begin{align*}
\ddot{\Lambda}_{2,s}(t) & = \int_{\mathbb{R}^{d}} \left[ \frac{\Gamma_{s}(h_s, h_s) \mathscr{L}_s h_s(t, x)}{(h_s(t, x))^{2}} - \frac{2\Gamma_{s}(h_s(t, x), \mathscr{L}_s h_s(t, x)) }{h_s(t, x)} \right] \dd \mu_s \\
                        & =  2 \int_{\mathbb{R}^{d}} \left(\frac{\Gamma_{s, 2}(h_s, h_s)}{h_s}  - \frac{\Gamma_{s}(h_s, \Gamma_s(h_s, h_s))) }{h_s^2} + \frac{2( \Gamma_{s}(h_s, h_s) )^2}{h_s^3} \right)  \dd \mu_s \\
                        & = 2 \int_{\mathbb{R}^{d}} h_s(  \Gamma_{s, 2}( \log h_s, \log h_s)  \dd \mu_{s}.
\end{align*}
This proof is complete.

\end{proof}

Next, we complete the proof of Lemma~\ref{thm: log-sobolev-inq}.

\begin{proof}[Proof of Lemma~\ref{thm: log-sobolev-inq}]
Using~\Cref{lem: curvature-estimate} and~\Cref{lem: derivative-log-2}, we obtain
\[
\ddot{ \Lambda}_{2,s}(t) \geq  - 2\mu \dot{\Lambda}_{2,s}(t). 
\]
Since $\dot{\Lambda}_{2,s}(t)$ is positive, we get
\begin{equation}\label{eqn: inequal-integral-log}
\dot{\Lambda}_{2,s}(t) \geq \dot{\Lambda}_{2,s}(0) \e^{- 2 \mu t}. 
\end{equation}
By the definition of $\Lambda_{2,s}(t)$, we have
\[
\Lambda_{2,s}(0) - \Lambda_{2,s}(\infty) = \int_{\mathbb{R}^{d}}h^2 \log h^2  \dd \mu_{s} -  \left(\int_{\mathbb{R}^{d}} h^2  \dd \mu_{s} \right) \log  \left(\int_{\mathbb{R}^{d}} h^2   \dd \mu_{s} \right),
\]
where the second term in the right-hand side follows from~\Cref{thm: converge} and 
\[
\int_{\mathbb{R}^d}h\dd\mu_s = \int_{\mathbb{R}^d}\rho\dd x = 1.
\]

By~\Cref{thm: converge}, we know that $h_{s}(\infty, \cdot) \equiv 1$. Plugging it into~\eqref{eqn: derivative-log-2}, we get
\[
\dot{\Lambda}_{2,s}(0) = - \int_{\mathbb{R}^{d}} h^{2} \Gamma_{s}(\log h^{2}, \log h^{2})  \dd \mu_s = - 2s \int_{\mathbb{R}^{d}} \| \nabla h\|^{2} \dd \mu_s. 
\]
where the last equality follows from~\eqref{eqn: gamma1-L-final-calcuation}. Integrating both sides of the inequality~\eqref{eqn: inequal-integral-log}, we have
\[
\Lambda_{2,s}(0) - \Lambda_{2,s}(\infty) \leq -\dot{\Lambda}_{2,s}(0) \int_{0}^{+\infty} \e^{-2\mu t} \dd t   \leq  - \frac{1}{2\mu} \dot{\Lambda}_{2,s}(0).
\]
This concludes the proof.

\end{proof}


\subsection{Proof of Proposition~\ref{prop: approx}}
\label{subsec: proof-approx}


By \Cref{prop: unique-existence}, let $\rho_s(t, \cdot) \in C^{1}([0, +\infty), L^2(\mu_s^{-1}))$ denote the unique transition probability density of the solution to the lr-dependent SDE. Taking an expectation, we get
\[
\mathbb{E}[X_s(t)] = \int_{\mathbb{R}^{d}} x \rho_s(t, x) \dd x  .
\]
Hence, the uniqueness has been proved. Using the Cauchy--Schwarz inequality and Theorem~\ref{thm: rate-l2-1}, we obtain:
\begin{align*}
\left\| \mathbb{E}[X_s(t)] \right\| & \leq  \left\| \int_{\mathbb{R}^{d}} x (\rho_{s}(t, \cdot) - \mu_s) \dd x   \right\| +  \left\| \int_{\mathbb{R}^{d}} x  \mu_s \dd x   \right\| \\
                                                & \leq \left(\int_{\mathbb{R}^d} \|x\|^{2} \mu_s dx \right)^{\frac{1}{2}} \left( \left\| \rho_s(t, \cdot) - \mu_{s} \right\|_{\mu_s^{-1}} + 1 \right)    \\
                                                & \leq  \left(\int_{\mathbb{R}^d} \|x\|^{2}  d \mu_s \right)^{\frac{1}{2}} \left( e^{-\lambda_{s}t}\left\| \rho- \mu_{s} \right\|_{\mu_s^{-1}} + 1 \right)  \\
                                                & < + \infty,   
 \end{align*}
where the integrability $\int_{\mathbb{R}^{d}} \|x\|^{2} \mu_s(x) \dd x $ is due to the fact that the objective $f$ satisfies the Villani condition. The existence of a global solution to the lr-dependent SDE~\eqref{eqn: sgd_high_resolution_formally} is thus established.

For the strong convergence, the~lr-dependent SDE~\eqref{eqn: sgd_high_resolution_formally} corresponds to the Milstein scheme in numerical methods. The original result is obtained by Milstein~\cite{mil1975approximate} and Talay~\cite{talay1982analyse, pardoux1985approximation}, independently. We refer the readers to~\cite[Theorem 10.3.5 and Theorem 10.6.3]{kloeden1992approximation}, which studies numerical schemes for stochastic differential equation. For the weak convergence, we can obtain numerical errors by using both the Euler-Maruyama scheme and Milstein scheme. The original result is obtained by Milstein~\cite{mil1986weak} and Talay~\cite{pardoux1985approximation, talay1984efficient} independently and~\cite[Theorem 14.5.2]{kloeden1992approximation} is
also a well-known reference. Furthermore, there exists a more accurate estimate of $B(T)$ shown in~\cite{bally1996law}.
The original proofs in the aforementioned references only assume finite smoothness such as $C^{6}(\mathbb{R}^d)$ for the objective function. 


\subsection{Connection with vanishing viscosity}
\label{subsec: supplement-viscosity}

Taking $s = 0$, the zero-viscosity steady-state equation of the Fokker--Planck--Smoluchowski equation~\eqref{eqn: Fokker-Planck} reads
\begin{equation}\label{eqn: FKS-steady-zero-vis}
\nabla \cdot (\mu_{0} \nabla f) = 0.
\end{equation}
A solution to this zero-viscosity steady-state equation takes the form
\begin{equation}\label{eqn: steady-state-soln-zero-vis}
\mu_{0}(x) = \sum_{i=1}^{m} c_{i}\delta(x - x_i), \quad \text{with}\quad \sum_{i=1}^m c_{i} = 1,
\end{equation}
where $x_i$'s are critical points of the objective $f$. As is clear, the solution is not unique. However, we have shown previously that the invariant distribution $\mu_s$ is unique and converges to
\[
\mu_{s \rightarrow 0}(x) = \delta(x - x^{\star})
\]
in the sense of distribution, which is a special case of \eqref{eqn: steady-state-soln-zero-vis}. Clearly, when there exists more than one critical point, $\mu_{s \rightarrow 0}(x)$ is different from $\mu_0(x)$ in general. In contrast, $\mu_{s \rightarrow 0}(x)$ and $\mu_0(x)$ must be the same for (strictly) convex functions. In light of this comparison, the correspondences between the case $s > 0$ and the case $s = 0$ are fundamentally different in nonconvex and convex problems.

Next, we consider the rate of convergence in the convex setting. Let
\[
f(x) = \frac{1}{2} \theta x^{2},
\]
where $\theta > 0$. Plugging into the Fokker-Planck-Smoluchowski equation~\eqref{eqn: Fokker-Planck},  we have
\begin{equation}
\label{eqn: FK-1D}
\left\{ \begin{aligned}
& \frac{\partial \rho_s}{\partial t} = \theta \frac{\partial (x \rho_s)}{\partial x} + \frac{s}{2} \frac{\partial^{2} \rho_s}{\partial x^{2}} \\
 & \rho(0, \cdot) = \rho \in L^{2}( \sqrt{s \pi / \theta}  \e^{\theta x^{2}/s})) .
 \end{aligned} \right.
\end{equation}
The solution to \eqref{eqn: FK-1D} is 
\begin{equation}
\label{eqn: soln-FK-1D}
\rho_s(t,x) = \sqrt{ \frac{\theta}{\pi s  \left( 1 - \ee^{-2\theta t }\right) } } \exp \left[ - \frac{\theta}{s } \frac{\left( x - x_{0} \ee^{- \theta t}\right)^{2}}{1 - \ee^{-2\theta t}} \right].
\end{equation}
For any $\phi(x) \in L^{2}( \sqrt{s \pi / \theta} \ee^{\theta x^{2}/s})$, we have
\begin{align*}
\left\langle \rho_s, \phi  \right\rangle & = \left\langle \sqrt{ \frac{\theta}{\pi s  \left( 1 - \ee^{-2\theta t }\right) } } \exp \left[ - \frac{\theta}{s } \frac{\left( x - x_{0} \ee^{- \theta t}\right)^{2}}{1 - \ee^{-2\theta t}} \right], \phi(x)  \right\rangle \\
                                                                             & = \left\langle  \frac{1}{\sqrt{2 \pi}} \ee^{- \frac{x^{2}}{2}}, \phi\left( \sqrt{\frac{s  (1 - \ee^{-2\theta t}) }{2\theta} } \cdot x + x_{0} \ee^{-\theta t} \right)   \right\rangle \\
                                                                             & \rightarrow \phi\left( x_{0} \ee^{-\theta t} \right) = \left\langle \delta(x - x_{0} \e^{-\theta t}), \phi(x)  \right\rangle
\end{align*}
as $s \rightarrow 0$, where $\delta(x - x_{0} \ee^{-\theta t})$ denotes the solution to the following zero-viscosity equation
\begin{equation}\label{eqn: zero-vis-Fokker-Planck}
\frac{\partial \rho_0}{\partial t}  =   \nabla \cdot \left(\rho_0 \nabla f\right).
\end{equation}
Furthermore, using the following inequality
\[
\left\| \phi\left( \sqrt{\frac{s (1 - \ee^{-2\theta t}) }{2\theta} } \cdot x + x_{0} \ee^{-\theta t} \right) - \phi\left( x_{0} \ee^{-\theta t} \right) \right\|_{\infty} \leq  O\left(\sqrt{s} \right), 
\]
we get $\left\langle \rho(t, x), \psi(x)  \right\rangle \rightarrow \langle \delta(x - x_{0} \ee^{-\theta t}), \psi(x)  \rangle$ at the rate $O\left(\sqrt{s} \right)$ for a test function $\psi$.

The phenomenon presented above is called \emph{singular perturbation}. It appears in mathematical models of boundary layer phenomena~\cite[Chapter 2.2, Example 1 and Example 2]{chorin1990mathematical}, WKB theory for Schr\"{o}dinger equations~\cite[Supplement 4A]{gasiorowicz2007quantum}, KAM theory for circle diffeomorphisms~\cite[Chapter 2, Section 11]{arnol?d2012geometrical} and that for Hamilton systems~\cite[Appendix 8]{arnol2013mathematical}. Moreover, the singular perturbation phenomenon shows that there exists a fundamental distinction between the $O(1)$-approximating ODE for SGD and the lr-dependent SDE~\eqref{eqn: sgd_high_resolution_formally}. In particular, the learning rate $s \rightarrow 0$ in the Fokker--Planck--Smoluchowski equation~\eqref{eqn: Fokker-Planck} corresponds to vanishing viscosity. The vanishing viscosity phenomenon was originally observed in fluid mechanics~\cite{chorin1990mathematical, kundu2008fluid}, particularly in the degeneration of the
Navier--Stokes equation to the Euler equation~\cite{chen1999vanishing}. As a milestone, the vanishing viscosity method has been used to study the Hamilton--Jacobi equation~\cite{crandall1983viscosity, evans1980solving,crandall1984some}. In fact, the Fokker--Planck--Smoluchowski equation~\eqref{eqn: Fokker-Planck} and its stationary equation are a form of Hamilton--Jacobi equation with a viscosity term, for which the Hamiltonian is
\begin{equation}\label{eqn: Hamiltonian-FKS}
H(x, \rho, \nabla \rho) = \Delta f \rho + \nabla f \cdot \nabla \rho.
\end{equation}
The Hamiltonian~\eqref{eqn: Hamiltonian-FKS} is different from the classical case~\cite{lions1982generalized, cannarsa2004semiconcave, evans2010partial}, which is generally nonlinear in $\nabla \rho$ (cf.\ Burger's equation). Although the Hamiltonian depends linearly on $\rho$ and $\nabla \rho$, the coefficients depend on $\Delta f$ and $\nabla f$. Hence, it is not reasonable to apply directly the well-established theory of Hamilton--Jacobi equations~\cite{crandall1983viscosity, evans1980solving, crandall1984some, lions1982generalized, cannarsa2004semiconcave, evans2010partial} to the Fokker--Planck--Smoluchowski equation~\eqref{eqn: Fokker-Planck} and its stationary equation. Furthermore, for the aforementioned example, which proves the $O(\sqrt{s})$ convergence for the Fokker--Planck--Smoluchowski equation with the quadratic potential $f(x) = \frac{\theta}{2}x^2$, is also a viscosity solution to the Hamilton--Jacobi equation~\cite{crandall1983viscosity}, since the
Hamiltonian~\eqref{eqn: Hamiltonian-FKS} for the quadratic potential degenerates to
\[
H(x, \rho, \nabla \rho) = 2\mathbf{tr}(A) \rho + 2Ax \cdot \nabla \rho,
\]
where $f(x) = x^{T}Ax$ and $A$ is positive definite and symmetric. Thus, we remark that the general theory of viscosity solutions to Hamilton--Jacobi equations cannot be used directly to prove the theorems in the main body of this paper.

In closing, we present several open problems.


\begin{itemize}
\item Consider the stationary solution $\mu_s(x)$ to the Fokker--Planck--Smoluchowski equation~\eqref{eqn: Fokker-Planck}. For a convex or strongly convex objective $f$ with Lipschitz gradients, can we quantify the rate of convergence? Does the rate of convergence remain $O(\sqrt{s})$?


\item Let $T > 0$ be fixed and consider the solution $\rho_s(t, x)$ to the Fokker--Planck--Smoluchowski equation~\eqref{eqn: Fokker-Planck} in $[0, T]$. For a convex or strongly convex objective $f$ with Lipschitz gradients, does the solution  to the  Fokker--Planck--Smoluchowski equation~\eqref{eqn: Fokker-Planck} converge to the solution to its zero-viscosity equation~\eqref{eqn: zero-vis-Fokker-Planck}? Is the rate of convergence still $O(\sqrt{s})$?

\item Consider the solution $\rho_s(t, x)$ to the Fokker--Planck--Smoluchowski equation~\eqref{eqn: Fokker-Planck} in $[0, +\infty)$. For a convex or strongly convex objective $f$ with Lipschitz gradients, does the global solution to the  Cauchy problem of the Fokker--Planck--Smoluchowski equation~\eqref{eqn: Fokker-Planck}  converge to the solution of its zero-viscosity equation~\eqref{eqn: zero-vis-Fokker-Planck}? Is the rate of convergence still $O(\sqrt{s})$?

\end{itemize}

\section{Technical Details for Section~\ref{sec: convergence-rate}}
\label{sec: proof_steady}

\subsection{Proof of Lemma~\ref{lem: space-holder}}
\label{subsec: space-holder}
From the Cauchy--Schwarz inequality, we get
\begin{align*}
\int_{\mathbb{R}^{d}} |g(x)| \dd x         & =      \int_{\mathbb{R}^{d}} |g(x)| \ee^{\frac{f(x)}{s}} \mathrm{e}^{-\frac{f(x)}{s}}\dd x \\
                                                       & \leq  \left( \int_{\mathbb{R}^{d}} g^{2}(x) \ee^{ \frac{2f(x)}{s}} \dd x  \right)^{\frac{1}{2}}  \left( \int_{\mathbb{R}^{d}} \ee^{- \frac{2f(x)}{s} } \dd x  \right)^{\frac{1}{2}} \\
                                                       & < +\infty.
\end{align*}
This completes the proof.


\subsection{Proof of Lemma~\ref{lem: equivlent-h}}
\label{subsec: equiv-h-rho}

Recall that the linear operator $\mathscr{L}_s$ in~\eqref{eqn: gene-oper} is defined as
\[
\mathscr{L}_s = - \nabla f \cdot \nabla + \frac{s}{2}\Delta f.
\]
Note that we have
\begin{align*}
\int_{\mathbb{R}^{d}}   \left(  \mathscr{L}_s g_{1} \right) g_{2}  \dd \mu_s     &  =    
\int_{\mathbb{R}^{d}}   \left(  - \nabla g_{1} \cdot \nabla f + \frac{s}{2} \Delta g_{1}\right) g_{2}  \dd \mu_s \\
                                                              & = - \frac{1}{Z_s}\int_{\mathbb{R}^{d}}   (\nabla g_{1} \cdot \nabla f) g_{2} \ee^{-\frac{2f}{s}} \dd x +  \frac{s}{2Z_s} \int_{\mathbb{R}^{d}} (\Delta g_{1})   g_{2} \ee^{-\frac{2f}{s}} \dd x\\
                                                                                          & = - \frac{1}{Z_s}\int_{\mathbb{R}^{d}}   (\nabla g_{1} \cdot \nabla f) g_{2} \ee^{-\frac{2f}{s}} \dd x -  \frac{s}{2Z_s} \int_{\mathbb{R}^{d}} \nabla g_{1} \cdot \nabla (g_{2} \ee^{-\frac{2f}{s}}) \dd x \\
                                                                                      & = - \frac{s}{2Z_s}  \int_{\mathbb{R}^{d}} (\nabla g_{1} \cdot \nabla g_{2} )\ee^{-\frac{2f}{s}} \dd x \\
                                                                                      & = - \frac{s}{2}  \int_{\mathbb{R}^{d}} (\nabla g_{1} \cdot \nabla g_{2} )  \dd \mu_s. 
\end{align*}
Therefore, $\mathscr{L}_s$ is self-adjoint in $L^2(\mu_s)$ and is non-positive.

\subsection{Proof of Lemma~\ref{thm: villani-poincare-inq}}
\label{subsec: villani-poincare-inq}

For completeness, we show below the original proof of \Cref{thm: villani-poincare-inq} from \cite{villani2009hypocoercivity} in detail. Let $V_{s} = \| \nabla f\|^{2}/s - \Delta f$, then for any $h \in C_{c}^{\infty}(\mathbb{R}^{d})$ with mean-zero condition
\begin{equation}
\label{eqn: mean-zero-cond}
\int_{\mathbb{R}^{d}} h  \dd \mu_{s} = 0,
\end{equation}
we can obtain the following key inequality \cite{deuschel2001large}
\begin{equation}
\label{eqn: key-estimate-vp}
 \int_{\mathbb{R}^{d}} V_{s} h^{2}   \dd \mu_{s}
\leq s\int_{\mathbb{R}^{d}} \| \nabla h\|^{2} \mu_{s} \dd \mu_{s}.
\end{equation}
To show \eqref{eqn: key-estimate-vp}, note that
\begin{align*}
0&\leq \lefteqn{\int_{\mathbb{R}^{d}} \left\| \nabla \left( h \e^{- \frac{f}{s}} \right) \right\|^{2} \dd x} \\
& = \int_{\mathbb{R}^{d}}\left\| (\nabla h) \e^{- \frac{f}{s}} - \frac{h}{s} (\nabla f) \e^{- \frac{f}{s} }\right\|^{2} \dd x  \\
                                                                                                                                 & = \int_{\mathbb{R}^{d}} \| \nabla h\|^{2} \e^{- \frac{2f}{s}} \dd x - \frac{2}{s} \int_{\mathbb{R}^{d}} h \nabla h \cdot \nabla f \e^{- \frac{2f}{s}} \dd x + \left( \frac{1}{s} \right)^{2} \int_{\mathbb{R}^{d}} h^2 \| \nabla f\|^{2} \e^{- \frac{2f}{s}} \dd x \\
                                                                                                                                 & = \int_{\mathbb{R}^{d}} \| \nabla h\|^{2} \e^{- \frac{2f}{s}} \dd x - \frac{1}{s} \int_{\mathbb{R}^{d}} ( \nabla h^2 \cdot \nabla f) \e^{- \frac{2f}{s}} \dd x + \left( \frac{1}{s} \right)^{2} \int_{\mathbb{R}^{d}} h^2 \| \nabla f\|^{2} \e^{- \frac{2f}{s}} \dd x \\
                                                                                                                                 & =  \int_{\mathbb{R}^{d}} \| \nabla h\|^{2} \e^{- \frac{2f}{s}} \dd x + \frac{1}{s} \int_{\mathbb{R}^{d}}   h^2  \nabla \cdot \left( (\nabla f) \e^{- \frac{2f}{s}} \right) \dd x + \left( \frac{1}{s} \right)^{2} \int_{\mathbb{R}^{d}} h^2 \| \nabla f\|^{2} \e^{- \frac{2f}{s}} \dd x \\
                                                                                                                                 & = \int_{\mathbb{R}^{d}} \| \nabla h\|^{2} \e^{- \frac{2f}{s}} \dd x + \frac{1}{s} \int_{\mathbb{R}^{d}} (h^2  \Delta f) \e^{- \frac{2f}{s}} \dd x  - \left( \frac{1}{s} \right)^{2} \int_{\mathbb{R}^{d}} h^2 \| \nabla f\|^{2} \e^{- \frac{2f}{s}} \dd x.
 \end{align*}
Recognizing $\mu_s \propto \e^{-\frac{2f}{s}}$, this proves \eqref{eqn: key-estimate-vp}.

Let $R_{0,s} > 0$ be large enough such that $V_{s}(x) > 0$ for $\|x\| \ge R_{0,s}$. For $R_s > R_{0, s}$, we can define $\epsilon_s$ as
\begin{equation}
\label{eqn: epsilon_s}
\epsilon_s(R_s) := \frac{1}{\inf\{ V_s(x):\; \|x\| \geq R_s \}}.
\end{equation}
Then $\epsilon(R_s) \rightarrow 0$ as $R_s \rightarrow \infty$. Furthermore, we assume the $R_{s}$ is large enough such that
\begin{equation}
\label{eqn: half-integral}
\int_{ \| x \| \leq R_{s}}  \dd \mu_{s} \geq \frac{1}{2}.
\end{equation}
From the key inequality~\eqref{eqn: key-estimate-vp}, we obtain that
\begin{equation}
\label{eqn: estimate-vill1}
\int_{|x| \geq R_{s}} h^{2} \dd \mu_{s} \leq \epsilon(R_s) \left[ s \int_{\mathbb{R}^{d}} \| \nabla h\|^{2}  \dd \mu_{s} - (\inf_{x \in \mathbb{R}^{d}} V_s(x)) \int_{\mathbb{R}^{d}} h^{2}  \dd \mu_{s} \right].
\end{equation}
Let $B_{R_{s}}$ be the ball centered at the origin of radius $R_{s}$ in $\mathbb{R}^{d}$ and define
\[
\mu_{s, R_s} = \left[ \int_{|x| \leq R_s}  \dd \mu_{s} \right]^{-1} \mu_s \mathbf{1}_{|x| \leq R_{s}}.
\]
Using the Poincar\'{e} inequality in a bounded domain~\cite[Theorem 1, Chapter 5.8]{evans2010partial}, we get
\[
\int_{x \in \mathbb{R}^{d}} h^{2}  \dd \mu_{s, R_{s}} \leq s C(R_{s}) \int_{x \in \mathbb{R}^{d}} \| \nabla h\|^{2} \mu_{s, R_{s}} \dd \mu_{s, R_{s}} +  \left( \int_{x \in \mathbb{R}^{d}}  h \dd \mu_{s, R_{s}} \right)^{2},
\]
where $ C(R_{s})$ is a constant depending on $R_{s}$. Furthermore, using the inequality~\eqref{eqn: half-integral}, we obtain
\begin{equation}
\label{eqn: estimate<=R}
\int_{\|x\| \leq R_s} h^{2}   \dd \mu_{s} \leq sC(R_{s}) \int_{\|x\| \leq R_{s}} \| \nabla h\|^{2}  \dd \mu_{s} + 2 \left( \int_{\|x\| \leq R_s} h   \dd \mu_{s} \right)^{2}.
\end{equation}
Making use of the mean-zero property of $h$, we have
\begin{equation}
\label{eqn: estimate-variance}
\left( \int_{\|x\| \leq R_s} h  \dd \mu_{s}  \right)^{2} = \left( \int_{\|x\| > R_s} h  \dd \mu_{s}  \right)^{2} \leq \int_{\|x\| > R_s} h^{2} \dd \mu_{s}.
\end{equation}
Combining~\eqref{eqn: estimate<=R} and~\eqref{eqn: estimate-variance}, we get
\begin{align}
\int_{x \in \mathbb{R}^{d}} h^{2}  \dd \mu_{s} \leq sC(R_{s}) \int_{x \in \mathbb{R}^{d}} \| \nabla h \|^{2} \dd \mu_{s} + 3 \int_{\|x\| \geq R_{s}} h^{2} \dd \mu_{s}. \label{eqn: estimate-vill2}
\end{align}

Taking~\eqref{eqn: estimate-vill1} and~\eqref{eqn: estimate-vill2} together, we obtain 
\begin{align}
\label{eqn: estimate-vill}
\int_{\mathbb{R}^{d}} h^{2}  \dd \mu_{s} \leq s[C(R_{s}) + 3\epsilon(R_{s})] \int_{\mathbb{R}^{d}} \| \nabla h\|^{2}  \dd \mu_{s} - 3(\inf_{x \in \mathbb{R}^{d}} V_s(x)). \epsilon_s(R_s) \int_{x \in \mathbb{R}^{d}} h^{2}  \dd \mu_{s}
\end{align}
Apparently, from the definition of $\epsilon_s(x)$, we can select $R_{s} > 0$ large enough such that $1 + 3s (\inf\limits_{x \in \mathbb{R}^{d}} V_{s}(x)) \epsilon(R_s) > 0$. Then, we can rewrite \eqref{eqn: estimate-vill} as
\begin{equation}
\label{eqn: finial-villani}
\int_{\mathbb{R}^{d}} h^{2} \dd \mu_{s} \leq  \frac{s}{2} \cdot \frac{2 (C(R_{s}) + 3 \epsilon(R_{s}))}{1 + 3s (\inf\limits_{x \in \mathbb{R}^{d}} V_{s}(x)) \epsilon(R_s)} \int_{x \in \mathbb{R}^{d}} \| \nabla h\|^{2}  \dd \mu_{s}.
\end{equation}
Finally,  using $h - \int_{\mathbb{R}^{d}} h \dd \mu_{s}$ instead of $h$ in the inequality~\eqref{eqn: finial-villani}, we prove the desired Poincar\'{e} inequality by taking
\[
\lambda_{s} = \frac{1 + 3s (\inf\limits_{x \in \mathbb{R}^{d}} V_{s}(x)) \epsilon(R_s)}{2 (C(R_{s}) + 3 \epsilon(R_{s}))}.
\]

\subsection{Proof of~\Cref{lem: deriv_epsilon=0}}
\label{subsec: proof-lem-deriv-ep=0}
For convenience, we introduce a shorthand:
\[
\Pi\left( \frac{g}{s} \right) = \frac{\ee^{-\frac{2g}{s}}}{\int_{\mathbb{R}^d} \ee^{-\frac{2g}{s}} \dd x}. 
\]
Then, we can rewrite the derivative as
 \begin{align*}\nonumber
\frac{\dd \epsilon(s)}{\dd s}  &  = \frac{ \frac{2}{s^{2}} \int_{\mathbb{R}^{d}} g^{2} \e^{-\frac{2g}{s}} \dd x  \int_{\mathbb{R}^{d}}  \e^{-\frac{2g}{s}} \dd x  - \frac{2}{s^{2}} \left( \int_{\mathbb{R}^{d}} g \e^{-\frac{2g}{s}} \dd x \right)^{2}  }{\left( \int_{\mathbb{R}^{d}} \e^{-\frac{2g}{s}} \dd x\right)^{2}} \\ & = 2 \int_{\mathbb{R}^{d}} \left( \frac{g}{s} \right)^{2} \Pi\left( \frac{g}{s} \right)dx - 2 \left(\int_{\mathbb{R}^{d}}  \frac{g}{s} \cdot  \Pi\left( \frac{g}{s} \right)dx \right)^{2}.
\end{align*}
Next, we assume that $\zeta_k(x) = x^k e^{-x^{\alpha}}$, where $\alpha < 1$ is a fixed positive constant and $k = 1, 2$. The facts that $\zeta_k(0) = 0$ and $\lim\limits_{x \rightarrow + \infty} \zeta_k(x) = 0$ give
\[
0 \leq \lim_{s \rightarrow 0^+} \left( \frac{g}{s} \right)^{k} \Pi\left( \frac{g}{s} \right) \leq \lim_{s \rightarrow 0^+} \zeta_{k}\left( \frac{g}{s} \right) = 0.
\]
Then, by Fatou's lemma, we get
 \begin{align*}
 0 \leq \liminf_{s \rightarrow 0^+} \frac{\dd\epsilon(s)}{\dd s} &\leq \limsup_{s \rightarrow 0^+} \frac{\dd\epsilon(s)}{\dd s}\\
                                                            & = \limsup_{s \rightarrow 0^+} \left[2 \int_{\mathbb{R}^{d}} \left( \frac{g}{s} \right)^{2} \Pi\left( \frac{g}{s} \right)\dd x - 2 \left(\int_{\mathbb{R}^{d}}  \frac{g}{s} \cdot  \Pi\left( \frac{g}{s} \right)\dd x \right)^{2}\right] \\
                                                          & = 2\limsup_{s \rightarrow 0^+}  \int_{\mathbb{R}^{d}} \left( \frac{g}{s} \right)^{2} \Pi\left( \frac{g}{s} \right)\dd x - 2\liminf_{s \rightarrow 0^+}  \left(\int_{\mathbb{R}^{d}}  \frac{g}{s} \cdot  \Pi\left( \frac{g}{s} \right)\dd x \right)^{2} \\  
                                                  & \leq 2 \int_{\mathbb{R}^{d}} \limsup_{s  \rightarrow 0^+} \left( \frac{g}{s} \right)^{2} \Pi\left( \frac{g}{s} \right)\dd x - 2 \left(\int_{\mathbb{R}^{d}} \liminf_{s  \rightarrow 0^+} \frac{g}{s} \cdot  \Pi\left( \frac{g}{s} \right)\dd x \right)^{2} \\
                                                  & = 2 \int_{\mathbb{R}^{d}} \lim_{s  \rightarrow 0^+} \left( \frac{g}{s} \right)^{2} \Pi\left( \frac{g}{s} \right)\dd x - 2 \left(\int_{\mathbb{R}^{d}} \lim_{s  \rightarrow 0^+} \frac{g}{s} \cdot  \Pi\left( \frac{g}{s} \right)\dd x \right)^{2} \\    
                                             & = 0.
 \end{align*}
 The proof is complete.

\end{document}